\documentclass[11pt]{article}

\usepackage{color}
\definecolor{cm}{RGB}{250,0,200}

\newcommand{\Err}{\texttt{Err}}

\newcommand{\unif}{\mathrm{unif}}
\newcommand{\aalign}{\mathrm{align}}
\newcommand{\aug}{\mathrm{aug}}
\providecommand{\norm}[1]{\lVert#1\rVert}
\providecommand{\bignorm}[1]{\big\lVert#1\big\rVert}

\usepackage{geometry}  
\usepackage[english]{babel}
\usepackage[utf8]{inputenc}
\usepackage[T1,T2A]{fontenc}

\usepackage[dvipsnames]{xcolor}
\usepackage{color}
\usepackage{paralist}
\usepackage{graphicx}
\usepackage{subcaption}
\usepackage{longtable} 
\usepackage{multirow}
\usepackage{listings}
\usepackage{makecell}
\usepackage{array}
\usepackage{float}
\usepackage{dsfont}
\usepackage{rotating}
\usepackage{booktabs}
\usepackage{enumerate}
\usepackage{tikz}
\usepackage{pgf}
\usepackage{xcolor}
\usepackage{listing}
\usepackage{enumitem}

\definecolor{codegreen}{rgb}{0,0.6,0}
\definecolor{codegray}{rgb}{0.5,0.5,0.5}
\definecolor{codepurple}{rgb}{0.58,0,0.82}
\definecolor{backcolour}{rgb}{0.95,0.95,0.92}

\lstdefinestyle{mystyle}{
    backgroundcolor=\color{backcolour},   
    commentstyle=\color{codegreen},
    keywordstyle=\color{magenta},
    numberstyle=\tiny\color{codegray},
    stringstyle=\color{codepurple},
    basicstyle=\ttfamily\footnotesize,
    breakatwhitespace=false,     
    breaklines=true,             
    captionpos=b,                
    keepspaces=true,             
    numbers=left,                
    numbersep=5pt,               
    showspaces=false,            
    showstringspaces=false,
    showtabs=false,              
    tabsize=2
}
\usepackage{amsmath}
\usepackage{amssymb}
\usepackage{amsthm}
\usepackage{bm}
\usepackage{bbm}
\usepackage{mathtools}
\usepackage{mathrsfs}
\usepackage{algorithm}
\usepackage{algorithmic}

\usepackage{natbib}
\usepackage[
  colorlinks,
  citecolor=blue,
  linkcolor=red,
  anchorcolor=red,
  urlcolor=blue
]{hyperref}
\mathtoolsset{showonlyrefs}
\usepackage{authblk}
\usepackage{todonotes}
\usepackage{hyperref}

\usepackage{hyperref}
\makeatletter
\newcommand{\printfnsymbol}[1]{%
  \textsuperscript{\@fnsymbol{#1}}%
}
\makeatother

\usepackage{xcolor}
\allowdisplaybreaks
\theoremstyle{plain}
\newtheorem{theorem}{Theorem}[section]
\newtheorem{proposition}[theorem]{Proposition}
\newtheorem{lemma}[theorem]{Lemma}
\newtheorem{corollary}[theorem]{Corollary}
\theoremstyle{definition}
\newtheorem{definition}[theorem]{Definition}

\theoremstyle{remark}
\newtheorem*{remark}{Remark}

\newcommand{\veps}{\varepsilon}
\newcommand{\E}{\mathbb{E}}

\newcommand{\ind}{\mathbbm{1}}

\newcommand{\argmax}{{\rm argmax}}
\newcommand{\argmin}{{\rm argmin}}

\def \tr{\mathrm{Tr}}

\def \det{\texttt{det}}

\usepackage{xspace}

\def\##1\#{\begin{align}#1\end{align}}
\def\$#1\${\begin{align*}#1\end{align*}}

\definecolor{myblue}{rgb}{.8, .8, 1}
\definecolor{mathblue}{rgb}{0.2472, 0.24, 0.6} 
\definecolor{mathred}{rgb}{0.6, 0.24, 0.442893}
\definecolor{mathyellow}{rgb}{0.6, 0.547014, 0.24}

\newcommand{\calA}{{\mathcal{A}}}
\newcommand{\calB}{{\mathcal{B}}}

\newcommand{\calE}{{\mathcal{E}}}

\newcommand{\calL}{{\mathcal{L}}}

\newcommand{\calN}{{\mathcal{N}}}
\newcommand{\calO}{{\mathcal{O}}}

\newcommand{\calW}{{\mathcal{W}}}
\newcommand{\calX}{{\mathcal{X}}}
\newcommand{\calY}{{\mathcal{Y}}}

\newcommand{\bfm}[1]{\ensuremath{\mathbf{#1}}}

 \def\bA{\bfm A} 
 \def\bB{\bfm B} 
  
 \def\bD{\bfm D} 
  \def\EE{\mathbb{E}}
  
\def\bg{\bfm g} \def\bG{\bfm G} 
\def\bh{\bfm h} \def\bH{\bfm H} 
 \def\bI{\bfm I}

 \def\bM{\bfm M} 
  \def\NN{\mathbb{N}}
  
 \def\bP{\bfm P} \def\PP{\mathbb{P}}
  
  \def\RR{\mathbb{R}}
  
 \def\bT{\bfm T} 
 \def\bU{\bfm U} 
\def\bv{\bfm v} \def\bV{\bfm V} 
 \def\bW{\bfm W} 
 \def\bX{\bfm X} 
\def\by{\bfm y}  
 \def\bZ{\bfm Z}

\newcommand{\xx}{\text{\boldmath $x$}}
\newcommand{\ww}{\text{\boldmath $w$}}
\newcommand{\hh}{\text{\boldmath $h$}}
\newcommand{\zz}{\text{\boldmath $z$}}
\newcommand{\ff}{\text{\boldmath $f$}}
\newcommand{\uu}{\text{\boldmath $u$}}
\newcommand{\vv}{\text{\boldmath $v$}}
\newcommand{\ee}{\text{\boldmath $e$}}
\newcommand{\aaaa}{\text{\boldmath $a$}}
\newcommand{\gggg}{\text{\boldmath $g$}}

\newcommand{\btheta}{\text{\boldmath $\theta$}}
\newcommand{\bvarphi}{\text{\boldmath $\varphi$}}
\newcommand{\bbeta}{\text{\boldmath $\beta$}}
\newcommand{\bmu}{\text{\boldmath $\mu$}}
\newcommand{\bsigma}{\text{\boldmath $\sigma$}}
\newcommand{\bSigma}{\text{\boldmath $\Sigma$}}
\newcommand{\bepsilon}{\text{\boldmath $\epsilon$}}
\newcommand{\bOmega}{\text{\boldmath $\Omega$}}

\newcommand{\bzeta}{\text{\boldmath $\zeta$}}
\newcommand{\blambda}{\text{\boldmath $\lambda$}}
\newcommand{\bnu}{\text{\boldmath $\nu$}}
\newcommand{\bzero}{\mathrm{\mathbf{0}}}

\def\fullcoef{\text{\boldmath $\xi$}}

\def\siml{{\rm sim}}
\newcommand{\bone}{\mathrm{\bf 1}}

\renewcommand{\tilde}{\widetilde}
\renewcommand{\hat}{\widehat}

\title{Unraveling Projection Heads in Contrastive Learning: \\ Insights from Expansion and Shrinkage}
\author[1]{Yu Gui}
\author[1]{Cong Ma}
\author[2]{Yiqiao Zhong}
\affil[1]{Department of Statistics, University of Chicago}
\affil[2]{Department of Statistics, University of Wisconsin-Madison}
\date{\today}

\begin{document}

\maketitle

\begin{abstract}
We investigate the role of projection heads, also known as projectors, within the encoder-projector framework (e.g., SimCLR) used in contrastive learning. We aim to demystify the observed phenomenon where representations learned before projectors outperform those learned after---measured using the downstream linear classification accuracy, even when the projectors themselves are linear.

In this paper, we make two significant contributions towards this aim. Firstly, through empirical and theoretical analysis, we identify two crucial effects---expansion and shrinkage---induced by the contrastive loss on the projectors.  In essence, contrastive loss either expands or shrinks the signal direction in the representations learned by an encoder, depending on factors such as the augmentation strength, the temperature used in contrastive loss, etc. Secondly, drawing inspiration from the expansion and shrinkage phenomenon, we propose a family of linear transformations to accurately model the projector's behavior. This enables us to precisely characterize the downstream linear classification accuracy in the high-dimensional asymptotic limit. Our findings reveal that linear projectors operating in the shrinkage (or expansion) regime hinder (or improve) the downstream classification accuracy. This provides the first theoretical explanation as to why (linear) projectors impact the downstream performance of learned representations. Our theoretical findings are further corroborated by extensive experiments on both synthetic data and real image data.


\end{abstract}


\addtocontents{toc}{\protect\setcounter{tocdepth}{-1}}

\section{Introduction}
Representation learning~\citep{bengio2013representation} is a fundamental task in machine learning and statistics with the aim of extracting representations from the data that are useful for building future classifiers or predictors. 
While supervised learning is effective for this purpose (for instance, deep neural networks such as ResNet~\citep{he2016deep} have achieved remarkable performance in image classification), it is limited by the availability of massive labeled data. 

{Self-supervised} learning (SSL)~\citep{balestriero2023cookbook} has recently emerged as a novel paradigm to learn meaningful representations from huge \emph{unlabeled} datasets~\citep{Misra2019SelfSupervisedLO, chen2020simple, he2020momentum,Dwibedi2021WithAL, haochen2021provable,jing2021understanding,wang2020understanding,ji2021power}. 
Among SSL methods~\citep{chen2020simple, zbontar2021barlow, bardes2021vicreg}, contrastive learning~\citep{chen2020simple} is arguably the most popular one, which is also the focus of this paper. In essence, contrastive learning learns representations by encouraging proximity between the representations of similar inputs (also known as positive pairs), while forcing the representations of dissimilar inputs (i.e., negative pairs) to be far from each other. 

Below, we compare contrastive learning with the classical representation learning methods to help readers better understand the former one. Readers familiar with contrastive learning can jump directly to Section~\ref{subsec:puzzle}.  

\subsection{Classical unsupervised learning based on encoders and decoders}
Principal component analysis (PCA), dating back to~\cite{pearson1901principal, hotelling1933analysis}, is perhaps the oldest unsupervised representation learning method. In a nutshell, PCA aims to find a linear function $\ff_{\bW}(\xx) = \bW \xx$ of the input $\xx$ that preserves as much information about the original data as possible. Mathematically, 
 PCA can be formulated as minimizing the \emph{reconstruction loss}: denoting $\hh = \ff_{\bW}(\xx)$, we search for another linear function $\gggg_{\bT}(\hh) = \bT \hh$ such that the empirical {reconstruction loss} 
\begin{equation*}
\frac{1}{n} \sum_{i=1}^n \big\| \xx_i - \gggg_{\bT} ( \ff_{\bW} (\xx_i) ) \big\|^2
\end{equation*}
is minimized over all possible linear maps $\ff_{\bW}$ and $\gggg_{\bT}$ of fixed dimensions. Here, $\{\xx_i\}_{i \le n}$ denotes the input data. 

Using the machine learning terminology, the linear function $\ff_{\bW}(\cdot)$ is called an \textit{encoder} that maps input data to latent representations while the linear function $\gggg_{\bT}(\cdot)$ is called a \textit{decoder} that reproduces the original input as accurately as possible. 
More generally, this \emph{encoder-decoder} approach forms the core principle of many other representation learning methods~\citep{ghojogh2023elements} including autoencoders \citep{bourlard1988auto} and variational autoencoders \citep{kingma2013auto}. 

This encoder-decoder approach serves as the precursor of modern deep learning. Before 2010, 
learning features from large unlabeled data (also called pretraining), often followed by fine-tuning on a smaller labeled dataset, is known to be beneficial for downstream tasks \citep{hinton2006reducing}. 
However, it is realized that the decoder component is not essential if we are not asking for a generative model, which can be time-consuming to train~\citep{chen2020simple}. This gives rise to new attempts and ideas of learning with no or limited labeled data. Contrastive learning is one emerging approach that achieves outstanding empirical performance, producing state-of-the-art models such as CLIP \citep{radford2021learning}.

\subsection{An encoder-projector framework for contrastive learning}

Contrastive learning marks its departure from the classical representation learning methods (e.g., autoencoders~\citep{bourlard1988auto,kingma2013auto} in two aspects: (1) first, instead of the encoder-decoder framework, contrastive learning generally follows the \emph{encoder-projector} framework (cf.~Figure~\ref{fig:simclr}); (2) second, instead of minimizing the reconstruction loss, contrastive learning minimizes a \emph{contrastive loss} with the aim to pull positive pairs closer and push negative ones farther. 
Below we detail these two modifications and other essential components of contrastive learning using images as a running example for the input data. 

\paragraph{Encoder. } Similar to autoencoders, contrastive learning starts with an encoder
$\hh = \ff_{\btheta}(\xx)$ that maps an input image $\xx$ to its  feature representation $\hh$. Oftentimes, the encoder $\ff_{\btheta}$ is a deep neural network, e.g., ResNet-50~\citep{he2016deep}.

\paragraph{Projector. } Contrastive learning has a unique component $ \zz = \gggg_{\bvarphi}(\hh)$ that projects the representation $\hh$ to its embedding $\zz$, used later for calculating loss functions. This map $\gggg_{\bvarphi}$ is usually called the \textit{projection head} or \textit{projector}. One often uses a simple network, e.g., a multilayer perceptron (MLP) with one or two hidden layers for the projector.

\paragraph{Augmentation and contrastive loss. } 
As we have mentioned, contrastive learning aims to learn representations that are close for positive pairs and far for negative pairs. Here we provide one example for building positive and negative pairs from unlabeled data. 
Each sample $\xx_i$ is augmented by transformations (e.g., color distortion) to produce semantically similar positive samples $\{\xx_{i,k}\}_{1\leq k \leq K}$, where $K$ denotes the number of augmentations (a.k.a.~views). Two augmentations $(\xx_{i,k}, \xx_{j, k'})$ are called a positive pair if $i = j$, and a negative pair if $i \neq j$. Figure~\ref{fig:image_aug} gives an example of building positive and negative pairs from unlabeled image data. 

With the positive and negative pairs in place, we can introduce the contrastive loss used in contrastive learning, as opposed to the reconstruction loss in PCA and autoencoders. 
Let $\zz_{i,k} \coloneqq \gggg_{\bvarphi}(\ff_{\btheta}(\xx_{i,k}))$.
Denoting the (cosine) similarity score by $\mathrm{sim}(\zz,\zz') \coloneqq \langle \frac{\zz}{\| \zz \|}, \frac{\zz'}{\| \zz' \|} \rangle$, we train the encoder and projector by minimizing a contrastive loss, e.g., the canonical SimCLR loss \citep{chen2020simple}:

\begin{figure}[t]
\centering
\includegraphics[width=0.75\textwidth]{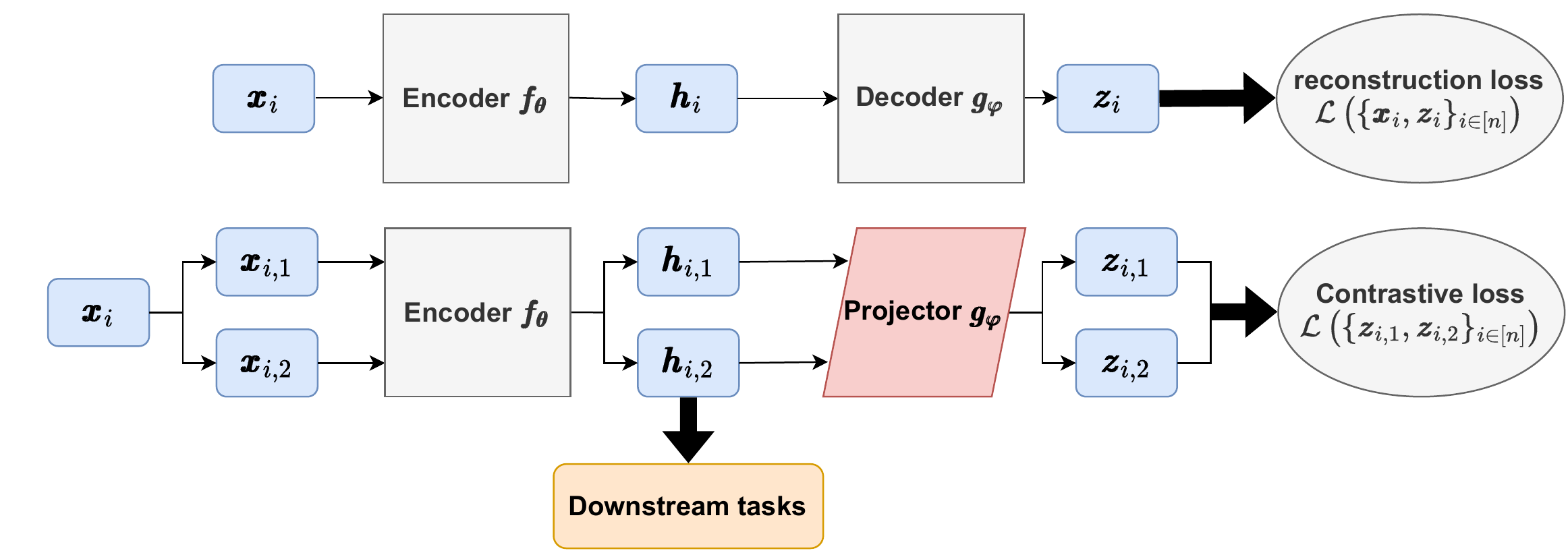}
\caption{Encoder-decoder framework vs.~encoder-projector framework.}
\label{fig:simclr}
\end{figure}

\begin{align}\label{expr:contrastiveloss}
\!\!\!\min_{\btheta, \bvarphi} \mathcal{L}(\btheta, \bvarphi) \coloneqq & \underbrace{- \frac{1}{\tau} \sum_{i} \sum_{k \neq k'}   \mathrm{sim}(\zz_{i,k}, \zz_{i,k'})}_{\text{alignment loss}}  + \underbrace{(K-1)\, \sum_{i, k}  \log  \Big( \!\!\!\!\!\!\sum_{(j,k') \neq (i,k)}e^{\mathrm{sim}(\zz_{i,k}, \zz_{j,k'}) /\tau} \Big)}_{\text{uniformity loss}},
\end{align}
where 
$\tau > 0$ is known as the \textit{temperature} parameter. Following~\citet{wang2020understanding}, we call the first term  the \emph{alignment} loss that promotes feature proximity of positive pairs, and the second term the \emph{uniformity} loss that repels negative pairs.

In fact, the contrastive loss can be viewed as pairwise cross-entropy loss (a.k.a.~logistic loss). Let $s = (i, k)$ denote the index tuple for simplicity. In this case, $s, s'$ are a positive pair if and only if $i=i'$. Then equivalently, one has $\mathcal{L}(\btheta, \bvarphi) = \sum_s \mathcal{L}_s(\btheta, \bvarphi)$ where
\begin{equation}\label{expr:contrastiveloss2}
\mathcal{L}_s(\btheta, \bvarphi) = - \sum_{s' \neq s} \bone \{i = i'\} \log \left( \frac{\exp\big( \mathrm{sim}(\zz_{s'}, \zz_{s}) / \tau\big)}{\sum_{\bar s \neq s} \exp\big( \mathrm{sim}(\zz_{\bar s}, \zz_{s}) / \tau\big)} \right) \, .
\end{equation}

\paragraph{Downstream accuracy.} After training, we freeze $\btheta$ and only keep the encoder $\ff_{\btheta}$ (i.e., throwing away the projector). Later, given a downstream task, say a classification problem with labeled data $(\tilde{y}_i, \tilde{\xx}_i)_{i\le m}$, one can simply apply logistic regression to learned features $\{ \ff_{\btheta}(\tilde{\xx}_i)\}$ and labels $\{\tilde{y}_i\}$.



\subsection{Puzzling effect of projectors on the representations}\label{subsec:puzzle}
The success of contrastive learning, or more specifically SimCLR~\citep{chen2020simple}, can be attributed (at least) to two ingredients: strong data augmentation and the use of projectors. However, the role of projectors is quite puzzling. It is observed in~\cite{chen2020simple} that representations learned before projectors outperform those learned after---measured using the downstream linear classification accuracy. This is even true when the projectors themselves are linear; see Figure~8 therein.  As a result, 
the standard practice in contrastive learning is to jointly train the encoder and the projector, and then \emph{remove} the projector completely after training~\citep{balestriero2023cookbook}.

In this paper, we aim to demystify this puzzling phenomenon about projectors in contrastive learning. In particular, we 
focus on answering the following two questions: 
\begin{enumerate}
\item[\textbf{Q1:}] \textit{What geometric structure does contrastive loss minimization induce on the projectors?}
\item[\textbf{Q2:}] \textit{Why do (linear) projectors affect the generalization properties of learned features?}
\end{enumerate}


\begin{figure}[t]
\centering
\includegraphics[width=0.75\textwidth]{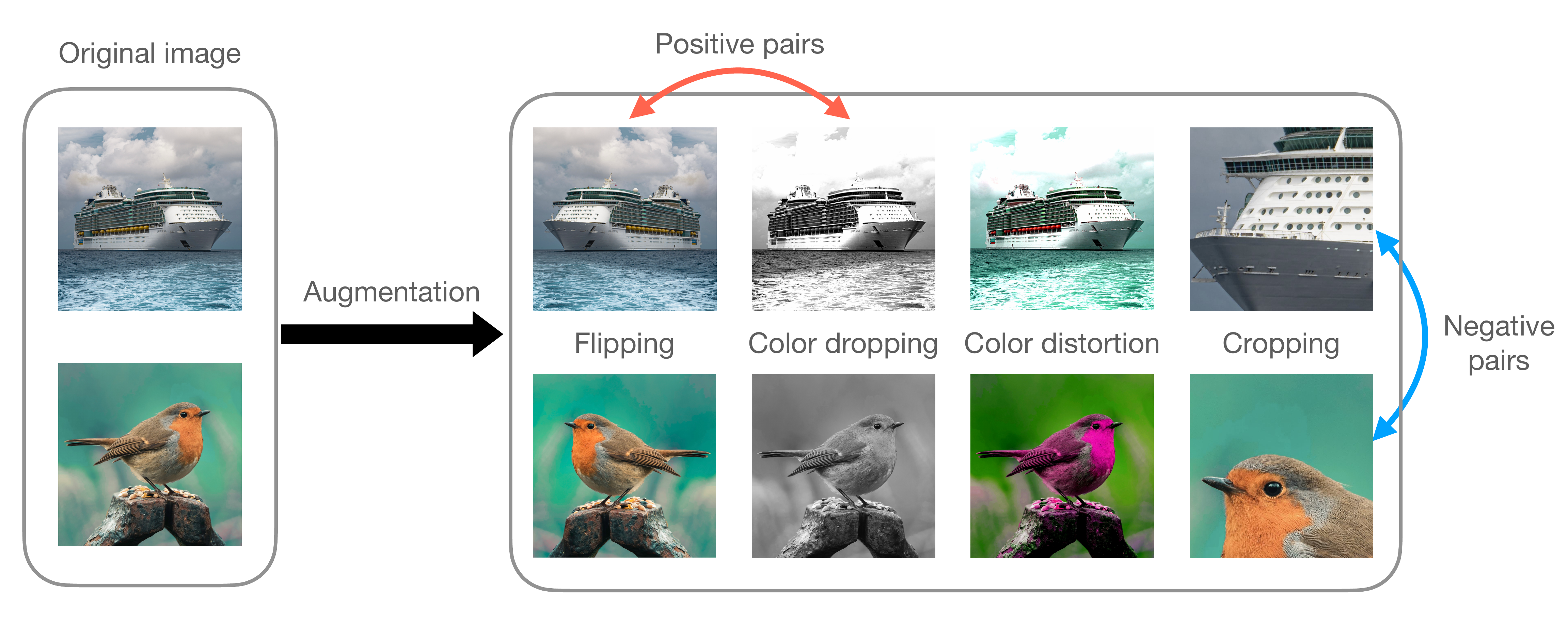}
\caption{Positive and negative pairs after image augmentation.}\label{fig:image_aug}
\end{figure}

For Q1, the contrastive loss is relatively new and thus less understood. It is unclear what properties the optimal solution to contrastive loss minimization possesses. A quantitative characterization will be helpful for understanding contrastive learning.

For Q2, heuristics in the literature are insufficient (see e.g., Section 3.2 in~\cite{balestriero2023cookbook}). One may view a projector as a buffer component: it protects features from distortion due to loss minimization, so its removal after training improves the generalization properties of features. Yet, this argument cannot reconcile with the observation that even linear projectors affect the downstream linear classification accuracy. Therefore, without a thorough investigation, the  role of projectors on generalization remains mysterious.

\subsection{Our contributions}
To delineate the effect of projectors, we do not attempt to analyze the encoder or its training dynamics, but rather assume access to a well-trained encoder, which is often achieved at the later stage of training. Through fixing such a good encoder, we make the following empirical and theoretical discoveries that are fundamental to addressing the aforementioned two questions. 
\begin{enumerate}
\item First, we identify two crucial effects---expansion and shrinkage---induced by the contrastive loss on the projectors.  In essence, contrastive loss either expands or shrinks the signal direction in the representations learned by encoders.
\item Secondly, under a simpler projector model, we precisely characterize the downstream linear classification accuracy in the high-dimensional asymptotic limit. Our findings reveal that linear projectors operating in the shrinkage (resp.~expansion) regime hinder (resp.~improve) the downstream classification accuracy. This provides the first theoretical explanation as to why (linear) projectors impact the downstream performance of learned representations.
\end{enumerate}
We also discuss connections to other empirical phenomena such as dimensional collapse, feature transferability, neural collapse, etc.

\subsection{Paper organization}
In Section~\ref{sec:empirical}, we present the empirical findings regarding the effects of the contrastive loss on the projectors, including expansion and shrinkage. In Section~\ref{sec:setup}, we introduce the feature-level Gaussian mixture model and provide empirical evidence as motivation for our modeling approach. Moving on to Section~\ref{sec:gmm_expansion_shrinkage}, we present a precise approximation of the population contrastive loss under the Gaussian mixture model. Additionally, we theoretically characterize the sharp phase transition between the expansion and shrinkage regimes, and provide the approximation bound along with the finite-sample loss. Section~\ref{sec:result} analyzes the impact of the expansion/shrinkage phenomenon of the linear projection head on downstream tasks. We calculate the precise generalization error as a function of the expansion effect. Furthermore, in Section~\ref{sec:inhomogeneous}, we extend our expansion/shrinkage results from Section~\ref{sec:gmm_expansion_shrinkage} to inhomogeneous augmentations. We identify the simultaneous expansion and shrinkage, which aligns with our empirical discoveries using the \texttt{STL-10} dataset. Simulation details and theoretical proofs are deferred to the appendix.

\subsection{Notations}
$\RR_+$ denotes $\{a \in \RR: a > 0\}$ and $\bar{\RR}$ denotes $\RR \cup \{\infty\}$.
For a vector $\uu$, we use $\norm{\uu}$ or $\norm{\uu}_2$ to denote its Euclidean norm. For vectors $\uu, \vv$ of the same length, we use $\langle \uu, \vv \rangle := \uu^\top \vv$ to denote the inner product. For a matrix $\bA$, we use $\norm{\bA}_{\mathrm F}$ to denote its Frobenius norm. The identity matrix of size $p \times p$ is denoted by $\bI_p$ or simply $\bI$.  We use $\mathcal{N}(\bmu, \bSigma)$ to denote a Gaussian distribution with mean $\bmu$ and covariance $\bSigma$. The notation $\Phi$ means the cumulative distribution function of the standard Gaussian variable.

For two real-valued sequences $(a_n)_{n \ge 1}$ and $(b_n)_{n \ge 1}$, we use the standard small-o notation: $a_n = o(b_n)$ means $\lim_{n\to \infty} a_n / b_n = 0$, and sometimes we also write $a_n \ll b_n$. For random variable $X_n$, $X_n = o_{\PP}(b_n)$ means $|X_n| / b_n$ converges in probability to $0$ as $n \to \infty$. Moreover, if $\bX_n$ is a random vector, then $\bX_n = o_{\PP}(b_n)$ means $\norm{\bX_n} / b_n$ converges in probability to $0$ as $n \to \infty$.

\section{Empirical discovery: expansion and shrinkage}\label{sec:empirical}

We begin with presenting our empirical findings on the two crucial effects---expansion and shrinkage, of contrastive loss on projectors.  

\paragraph{Experimental setup.} We provide a brief description about our experimental setup, and leave the details to the appendix. We freeze the encoder (based on  ResNet-18)\footnote{Downloaded from \url{https://github.com/sthalles/SimCLR}.} pretrained on the \texttt{STL-10} image dataset.\footnote{Source from \url{https://cs.stanford.edu/~acoates/stl10/}.} 
Then we apply standard data augmentation techniques such as random cropping and color distortion to generate positive/negative pairs.
In the end, we train a linear projector $\bW \in \RR^{512 \times 512}$ under different configurations of hyperparameters, including $5$ different temperatures. 



In what follows, we summarize our empirical findings.

\subsection{Insights from spectral decomposition}
First, for two different class labels $c_1, c_2$ (e.g., $c_1$ denotes airplane while $c_2$ denotes dog), we use $\bmu_{c_k} \coloneqq \mathrm{Ave}_{y_i=c_k} \hh_i$ $(k=1,2)$ to denote the average representation after the encoder, where $\mathrm{Ave}$ is the average operator. Correspondingly, we denote
$\bmu_{c_1,c_2} \coloneqq \bmu_{c_1} - \bmu_{c_2} \in \RR^{p}$ to be the difference between class means. It turns out that $\bmu_{c_1,c_2}$is closely connected to the top/bottom right singular subspaces of the linear projector $\bW$.

To be more precise, let $\mathcal{V}_{\mathrm{top}}, \mathcal{V}_{\mathrm{bottom}}$ be the top/bottom singular subspaces containing a few right singular vectors (SVs) of $\bW$ and $\mathcal{V}_{\mathrm{bulk}}$ be the singular subspace containing the remaining SVs. Clearly one has $\RR^p = \mathcal{V}_{\mathrm{top}} \oplus \mathcal{V}_{\mathrm{bulk}}  \oplus \mathcal{V}_{\mathrm{bottom}}$. 


\begin{center}
\textit{Our first main empirical finding is that, empirically, the following holds approximately:
}
\begin{align*}
    \bmu_{c_1, c_2} \in \mathcal{V}_{\mathrm{top}} \oplus \mathcal{V}_{\mathrm{bottom}} \, .
\end{align*}
\end{center}

In other words, the energy is concentrated on the span of extreme (right) singular vectors. Moreover, on  $\mathcal{V}_{\mathrm{top}}$ and $\mathcal{V}_{\mathrm{bottom}}$, the corresponding singular values experience sharp drops.
Indeed, in Figure~\ref{fig:image_pretrained_plot}, we train $\bW$ on the $10$-class \texttt{STL}-10 dataset and calculate the alignment score $\kappa_{j, c_1,c_2} := \langle \vv_j, \bmu_{c_1,c_2} \rangle^2 / \Vert \bmu_{c_1,c_2} \Vert^2$, where $\vv_j$ is the $j$th right SV of $\bW$. For each index $i$, we report the cumulative score $\sum_{j \le i} \kappa_{j, c_1,c_2}$, which satisfies the normalization $\sum_{j=1}^p \kappa_{j,c_1,c_2} = 1$. A wide flat cumulative score suggests orthogonality $\bmu_{c_1, c_2} \bot \; \mathcal{V}_{\mathrm{bulk}}$.
A geometric interpretation is that $\bW$ is \textit{expanding} vectors in $\mathcal{V}_{\mathrm{top}}$ and \textit{shrinking} vectors in $\mathcal{V}_{\mathrm{bottom}}$. 


\begin{figure}[t]
    \centering
    \includegraphics[width=0.3\textwidth]{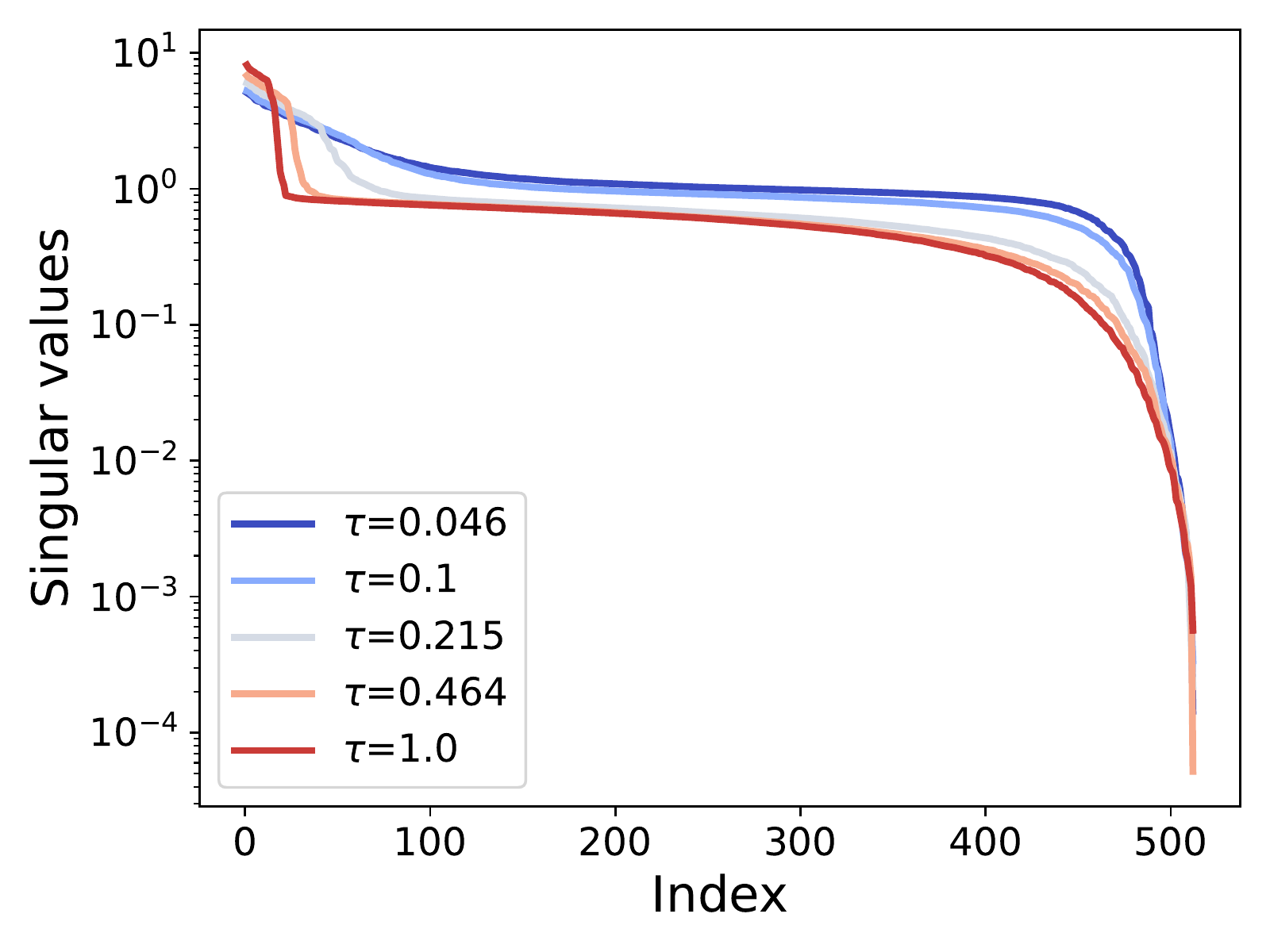} 
    \includegraphics[width=0.3\textwidth]{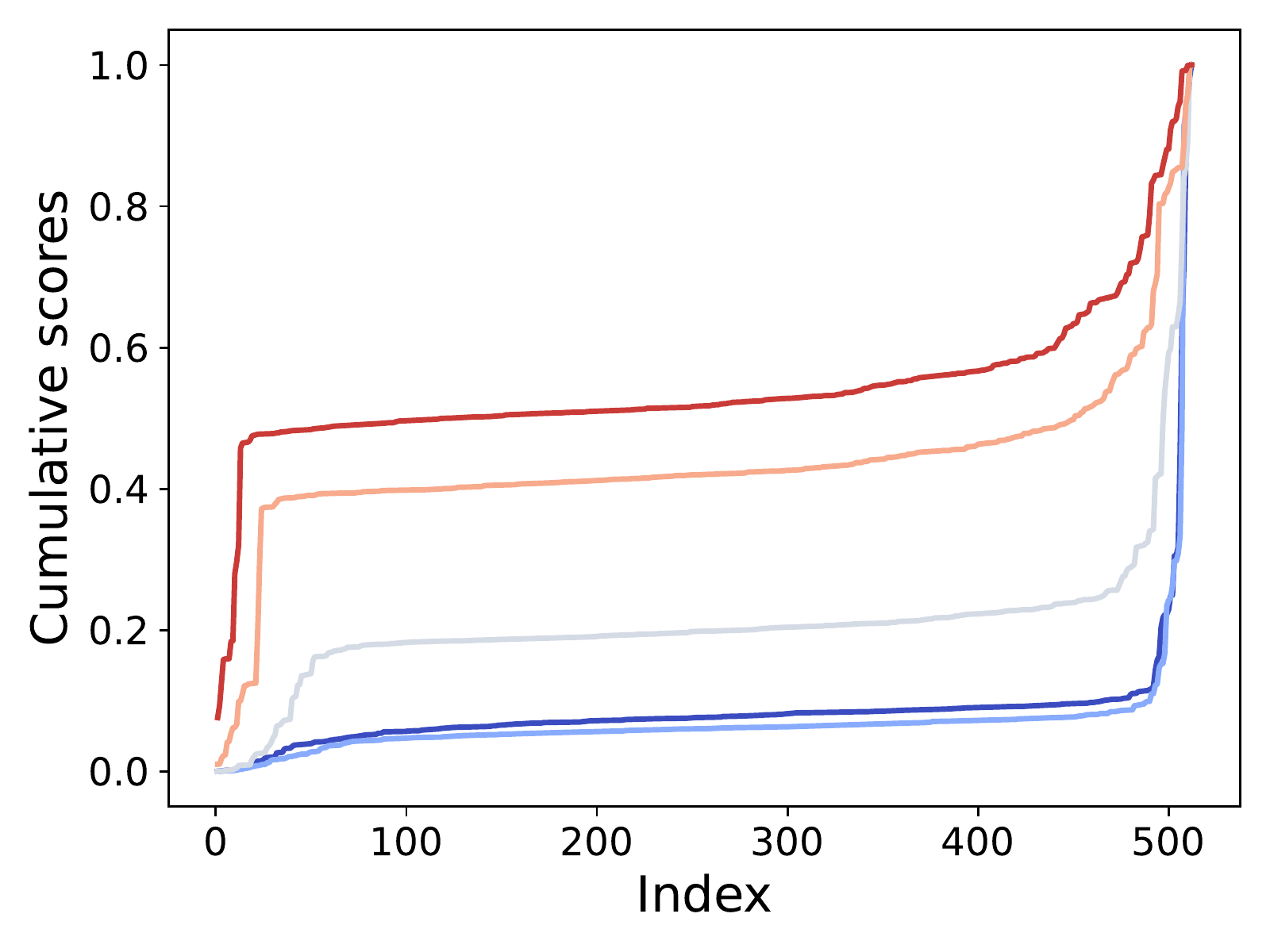} 
    \includegraphics[width=0.3\textwidth]{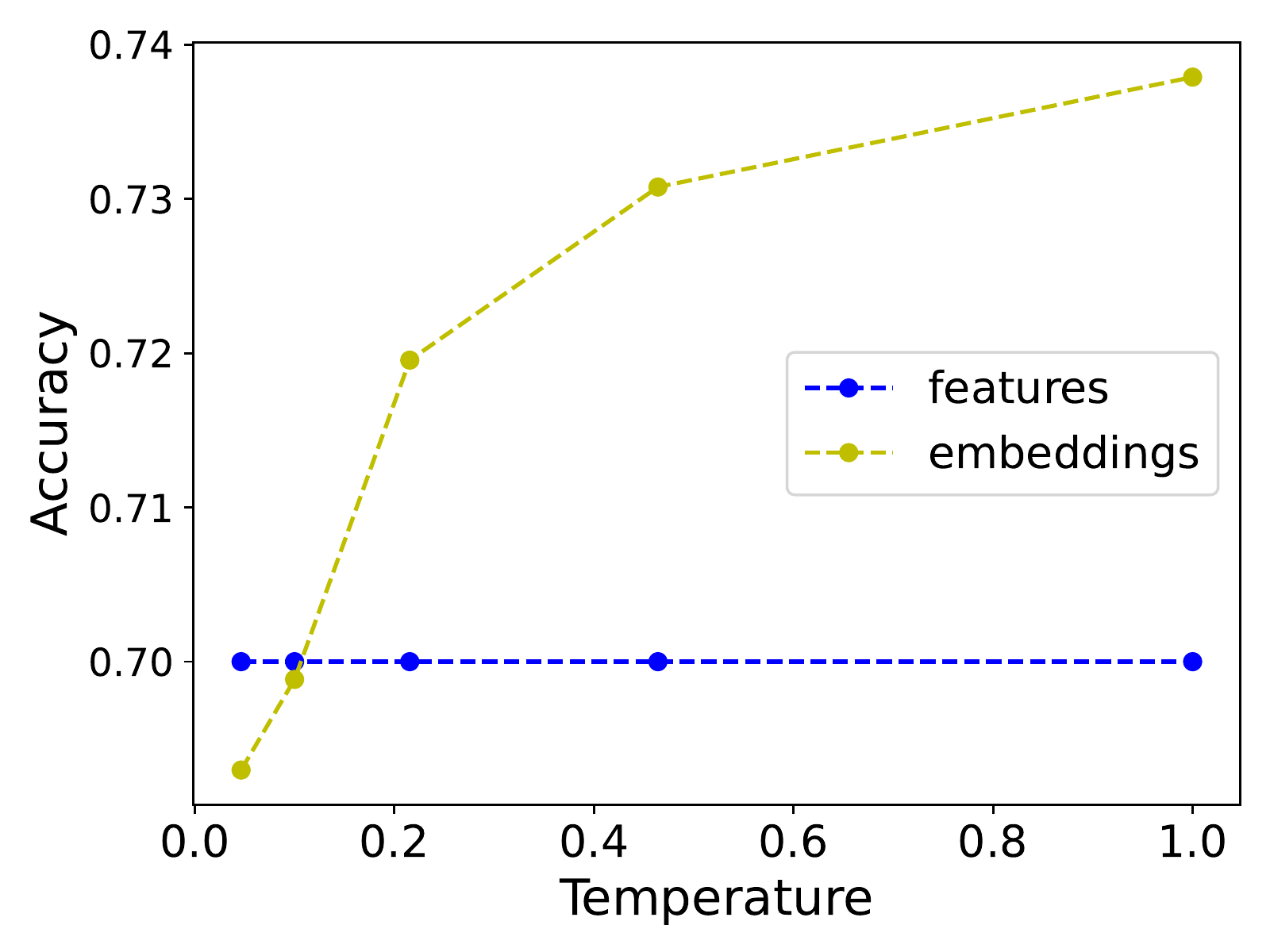}

    \caption{Results with the pretrained encoder and a one-layer linear projector $\bW \in \RR^{512 \times 512}$ under the standard SimCLR loss (ResNet-18 and 10-class \texttt{STL}-10 dataset). {\textbf{Left}}: singular values of $\bW$ with varying temperature $\tau$. {\textbf{ Middle}}: cumulative score plot 
    $\mathrm{score}_i = \sum_{j \le i} \langle \vv_j, \bmu_{c_1,c_2} \rangle^2 / \norm{ \bmu_{c_1,c_2}}^2 $ 
    where 
    $(c_1,c_2)=(7,2)$. {\textbf{ Right}}: downstream task accuracy using features $\hh$ versus using embeddings $\zz$ on the test set containing all $10$ classes.  }
    \label{fig:image_pretrained_plot}
\end{figure}

\subsection{Expansion/shrinkage affects generalization}

\begin{center}
        \textit{Our second empirical finding is that 
        the expansion/shrinkage effects are highly correlated with downstream accuracy.}
\end{center}

 As we observed earlier, when the temperature $\tau$ increases, expansion gets stronger because the alignment of the top SVs with $\mu_{7,2}$ increases. Figure~\ref{fig:image_pretrained_plot} (right) shows that a larger temperature also leads to an increase in the classification accuracy using the embeddings $(z_i)_{i\le n}$. Figure~\ref{fig:image_pretrained_plot} (right) confirms the surprise that a linear projector can change the generalization performance significantly. 

Can we theoretically justify the expansion and shrinkage effects of contrastive loss and their impact on the generalization power of the learned representations?
%
%
\subsection{A simple simulation and heuristic explanations}
Here, we present a simple simulation using synthetic input data that recreates the salient characteristics of our empirical findings. 
We hope to provide some heuristic explanations for the observed phenomena, while leaving the formal proof to later sections. 

\paragraph{A visual illustration.} The main characteristics of the empirical structure of projectors are displayed in a simple clean model. Consider a 2-component Gaussian mixture model $\frac{1}{2} \mathcal{N}(-\bmu, \bI_p) +  \frac{1}{2} \mathcal{N}(\bmu, \bI_p)$. We generate data according to this model, and then generate augmented data by adding independent perturbation (Gaussian noise). The data pass through a linear layer $\bW \in \mathbb{R}^{p \times p}$ treated as the projector. We obtain $\bW^*$ by minimizing the SimCLR loss over $\bW$.

We project and visualize our simulated data in Figure~\ref{fig:gmm-illustration} using $\bW^*$. Under two different hyperparameter settings, the original data (shown in the left plot) are either (i) extended along the separation direction $\bmu$ in the middle plot, or (ii) compressed along $\bmu$ in the right plot. 

\begin{figure*}[t]
    \centering
    \includegraphics[width=0.9\textwidth]{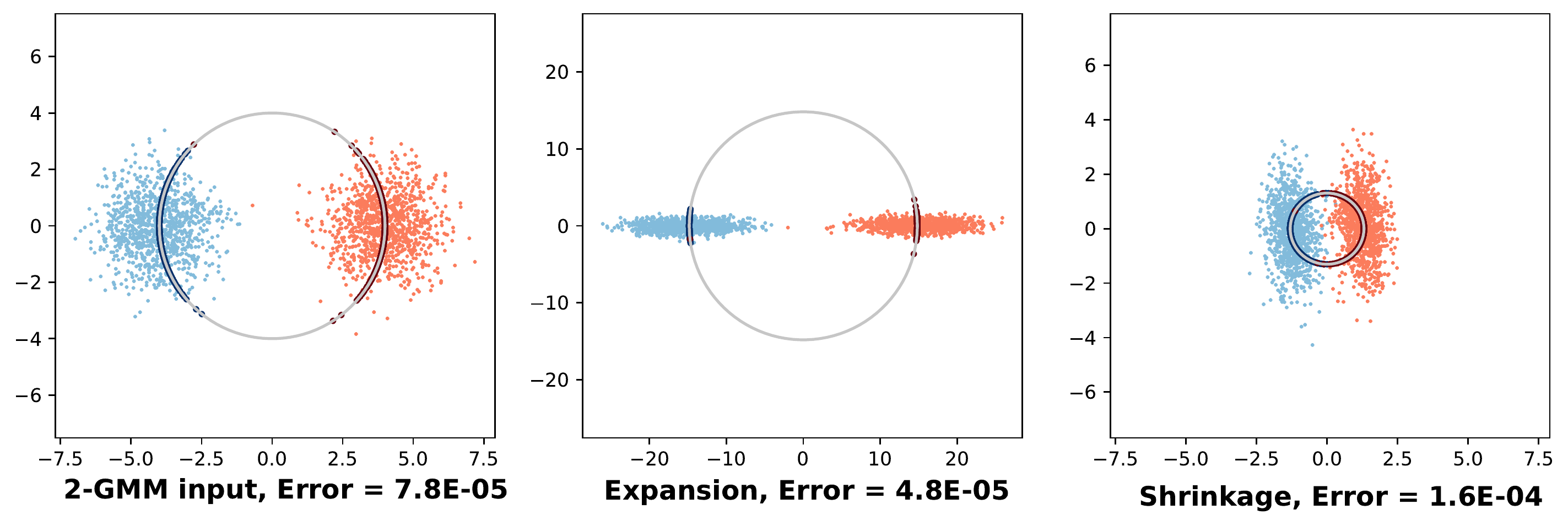}
    \caption{Contrastive loss drives expansion or shrinkage. \textbf{Left:} we generate clean 2-component GMM data. \textbf{Middle and right:} we use SimCLR loss for training $100$-dim data, and then visualize embeddings on a 2D plane under large/small augmentation and large/small temperature (middle/right plot). Also, we show the projection of embeddings onto a circle with varying radii, which is $\zz_i/\norm{\zz_i}$ used for calculating cosine similarities. \textbf{Bottom text:} expansion decreases test error, while shrinkage reduces signals and increases test error.}
    \label{fig:gmm-illustration}
\end{figure*}

The simulation result matches earlier empirical findings. (i) In the expansion regime, normalized features (projected to a circle) are more aligned, signaling a large top singular value; and in the shrinkage regime, normalized features are more uniform, signaling a smaller singular value. (ii) The test error decreases in the expansion regime that has a higher temperature, and it increases in the shrinkage regime that has a lower temperature.

\paragraph{Expansion and shrinkage promote alignment and uniformity.} In \citet{wang2020understanding}, the contrastive loss is decomposed into two components, which are called the \textit{alignment} loss and the \textit{uniformity} loss---which correspond to the two terms in the RHS of \eqref{expr:contrastiveloss}. Feature embeddings in the hypersphere are driven by the two opposite forces induced by the two loss components. 
Our expansion and shrinkage perspective offers a consistent explanation. If the projector stretches the features along the signal direction thus increasing variance in that direction, then after normalization to the hypersphere the features become more aligned. Conversely, if the projector compressed the features along the signal direction, after normalization features will become more uniform.
See Section~\ref{sec:gmm_expansion_shrinkage} for a formal analysis.

\paragraph{Projector as reparametrization changes inductive bias.} Why does the projector, even in the linear case, change the generalization performance on downstream tasks? Reparametrization techniques such as skip connections \citep{he2016deep} and batch normalization \citep{ioffe2015batch} are commonly used in deep learning, yet how they impact generalization is rarely elucidated. Here, we give an explanation from the lens of inductive bias. When the downstream task involves linearly separable data, gradient descent on the commonly used logistic loss produces a sequence of iterates $\bbeta^{(1)}, \bbeta^{(2)}, \ldots$ that converge in direction to the max-margin solution \citep{Soudry2018TheIB}.
\begin{equation*}
\begin{split}
\begin{array}{rcl}
\max_{\bbeta} & & \min_{i \le n} y_i \langle \zz_i, \bbeta \rangle \\
\text{subject to} & & \norm{\bbeta}_2 \le 1.
\end{array}
\end{split}
\end{equation*}
The $\ell_2$ norm in the constraint plays the role of implicit regularization that is induced by the gradient descent. Applying a linear transformation to features, however, will effectively lead to a different norm. Thus, the generalization properties on downstream tasks are affected by even a linear projector. See Section~\ref{sec:result} for the formal analysis.


\section{Feature-level modeling via GMM}\label{sec:setup}





To decouple the encoder and the projector, we assume access to a well-trained encoder and model the output using a well-separated Gaussian mixture model (GMM), where each Gaussian component represents a class. 
This assumption is also supported by at least two empirical observations: (i) the cluster structure in the feature space, and (ii) convergence speed during training. We defer detailed justifications about this assumption to Section~\ref{sec:modelmotiv}, before which we detail the model setup using GMMs.





\subsection{Model setup}\label{sec:GMM-setup}
Mathematically, we assume that the outputs of the encoder (i.e., features) are generated from a 2-component GMM (or 2-GMM), that is 
$$
\hh_{0,i} \stackrel{\mathrm{i.i.d.}}{\sim} \frac{1}{2}\, \mathcal{N} ( -\bmu, \bI_p ) + \frac{1}{2}\, \mathcal{N} ( \bmu, \bI_p ),
$$ 
where $\bmu \in \mathbb{R}^{p}$ denotes the mean difference between two classes. Conditional on $\hh_{0,i}$, we construct two augmentations (views) via 
$$
\hh_i, \hh_i^+ \mid \hh_{0,i} \stackrel{\mathrm{i.i.d.}}{\sim} \mathcal{N}(\hh_{0,i}, \sigma_{\aug}^2 \bI_p),
$$
where $\sigma_{\aug}^2>0$ represents the augmentation strength.
As a result, $(\hh_i, \hh_i^+)$ is viewed as a positive pair, and $(\hh_i, \hh_j)$, $(\hh_i, \hh_j^+)$ $(i \neq j)$ are considered negative pairs. 
For the projector, we focus on \emph{linear} projectors, i.e., $\zz = \gggg_{\bW}(\hh) = \bW \hh$, where $\bW \in \mathbb{R}^{p \times p}$.
We present a modified SimCLR loss that is more amenable to theoretical analysis.
Given the temperature $\tau>0$ and projector outputs $(\zz_i)_{i \le n}$ where $\zz_i = \gggg_{\bW}(\hh_i)$, we define 
    \begin{align}
        &\mathcal{L}_n(\bW) \coloneqq -\frac{1}{n\tau} \sum_{i=1}^n \E_{\zz_i^+, \zz_i \mid \hh_{0,i}} \big[ \mathrm{sim}^*(\zz_i, \zz_i^+) \big] + \log \left(\frac{1}{n} \sum_{i=1}^n \E_{\zz_i^-, \zz_i \mid \hh_{0,i}} \big[ e^{ \mathrm{sim}^*(\zz_i^-, \zz_i)/\tau } \big] \right), \label{def:Ln} \\
        &\text{where}~~\mathrm{sim}^*(\zz_1, \zz_2) \coloneqq \frac{-\| \zz_1 - \zz_2 \|^2/2}{\{ \E[\|\zz_1\|^2] \}^{1/2} \cdot \{ \E[\|\zz_2\|^2] \}^{1/2}} \; . \notag
    \end{align}
Here, $\E_{\zz_i^+, \zz_i \mid \hh_{0,i}}$ denotes expectation over random augmentations \textit{conditioning} on $\hh_{0,i}$ (i.e., using every possible positive pair), $\zz_i^-$ denotes an independent copy of $\zz_i$, and $\E_{\zz_i^-,\zz_i \mid \hh_{0,i}}$ denotes expectation over the negative pair $(\zz_i, \zz_i^-)$ conditioning on $\hh_{0,i}$ (i.e., using every possible negative pair).

Our modified SimCLR loss is different from the original SimCLR in the following ways: (i) We consider full batch and training for infinite time (so averages over augmentations are replaced by expectations); (ii) similar to \citet{wang2020understanding}, we interchange $\log$ with the summation; and (iii) we use a variant of similarity score $\mathrm{sim}^*$ where we replace instance-based normalization with the population level normalization. 

\begin{center}
\textit{We emphasize that this modified loss is only used for theoretical analysis. In all experiments, we use the original SimCLR contrastive loss. }
\end{center}

\subsection{Why fixing a well-trained encoder? }\label{sec:modelmotiv}
Arguably, our theoretical study departs from common practice as we assume a well-trained encoder and fix it to be a simple Gaussian mixture model. However, we would like to argue that this assumption sheds light on the practice.

\paragraph{Convergence of projectors.} Training an encoder and projector jointly creates complex dynamics that are beyond the scope of this paper. Yet, we observe in experiments that the dynamics are significantly simplified in the later stage of training. Let $( \btheta^{(t)}, \bW^{(t)})$ be the parameters at epoch $t$, and $\tilde \bW^{(t)}$ be the optimal projector parameters $\tilde \bW^{(t)} \coloneqq \argmin_{\bW}  \mathcal{L}(\btheta^{(t)}, \bW)$ while freezing encoder parameters $\btheta^{(t)}$ at epoch $t$. It can be seen from Figure~\ref{fig:train_freeze} in the appendix that when $t \geq 50$, 
\begin{equation*}
\norm{ \tilde \bW^{(t)}  - \bW^{(t)} } \ll \norm{ \bW^{(t)} },
\end{equation*}
which suggests that the projector is close to (conditionally) optimal. This allows us to decouple the dynamics and encoders and projectors into two separate optimization problems:
\begin{enumerate}
\item For fixed encoder $\btheta$, solve $\min_{\bW}  \mathcal{L}(\btheta, \bW)$. 
\item Solve $\min_{\btheta}  \mathcal{L}(\btheta, \tilde \bW(\btheta))$ where $\tilde \bW(\btheta)$ is the minimizer of the first part.
\end{enumerate}
Our focus is the first optimization problem, assuming that the encoder $\btheta$ is sufficiently good, e.g., $\btheta = \btheta^{(t)}$ for large $t$. We remark that the feature-learning process and training dynamics of encoder networks are complicated and elusive for analysis, so it is not uncommon to decouple multiple layers or components by freezing some of them \cite{han2021neural, bietti2022learning}.


\begin{figure}[t]
    \centering
    \begin{subfigure}{0.30\textwidth}
        \includegraphics[width=0.9\textwidth]{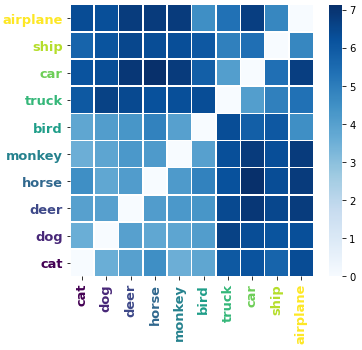}
        \label{fig:cluster1}
    \end{subfigure}
    \hspace{-2em}
    \begin{subfigure}{0.68\textwidth}
        \includegraphics[width=0.9\textwidth]{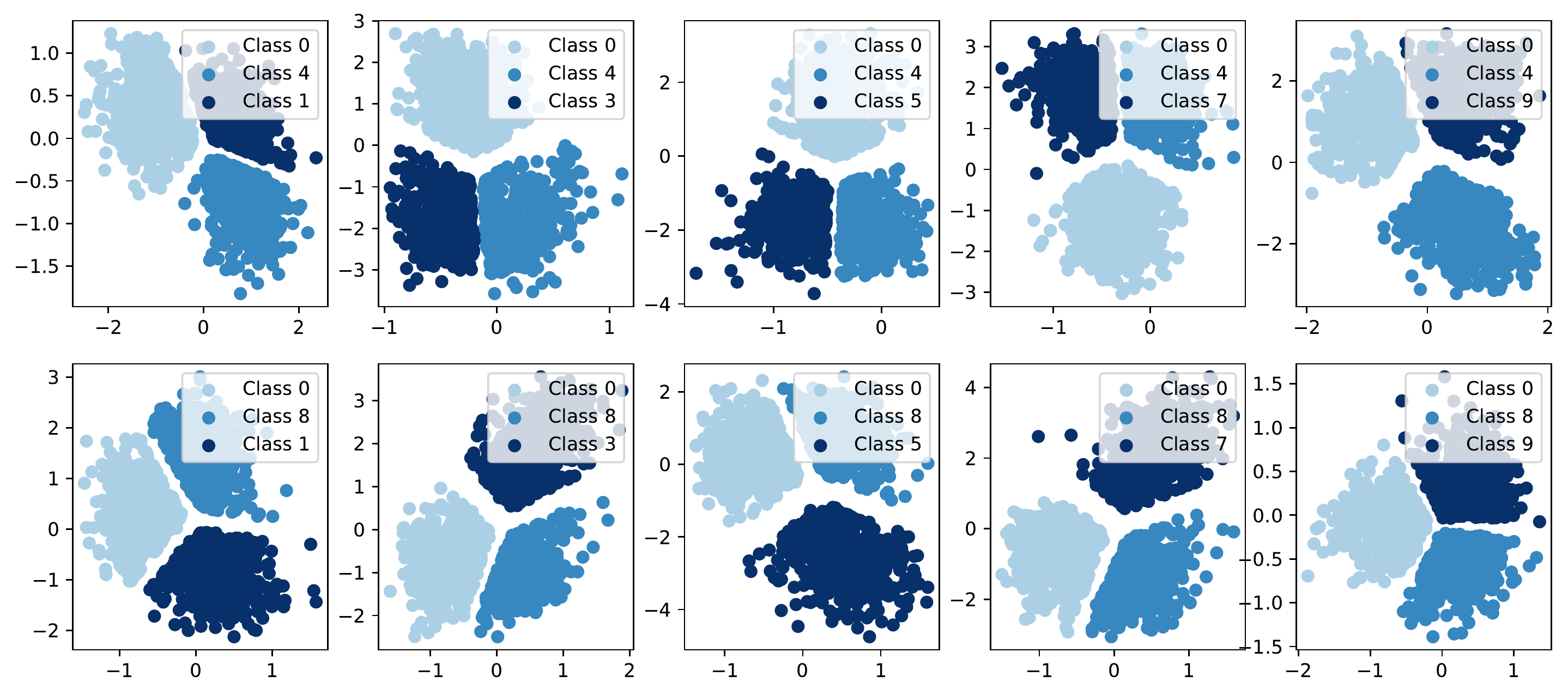}
        \label{fig:class_2d}
    \end{subfigure}
    \caption{Pretrained \texttt{STL}-10 features from the standard SimCLR model exhibit a strong cluster structure. \textbf{Left}: positive pairwise margins of linear classifier imply pairwise linear separability. \textbf{Right}: 2D visualization of 3-class subsets of features after $30$ epochs. The cluster structure emerges quickly.}
    \label{fig:cluster}
\end{figure}

\paragraph{Cluster structure in feature space.}

When the encoders are trained well, we observe well-separated clusters formed by pretrained features in contrastive learning, and these clusters contain class information \citep{bohm2022unsupervised}. 
We examine the cluster structure of pretrained features both visually and quantitatively. In Figure~\ref{fig:cluster} (left plot), for every pair of classes, we perform max-margin linear classification on a subset of pretrained features. We find that every pair achieves zero training error with a large margin, which indicates that all classes are linearly separable. In Figure~\ref{fig:cluster} (right plot), we visualize pretrained features by following the same projection technique \citet{muller2019does}. 
For well-clustered pretrained features, the GMM is a natural model, which is the starting point of our analysis.


\section{\mbox{Contrastive loss drives expansion/shrinkage}}\label{sec:gmm_expansion_shrinkage}


Under the model setup in Section~\ref{sec:GMM-setup}, we hope to characterize the minimizer of $\mathcal{L}_n(\bW)$. However, the loss $\mathcal{L}_n$ is very complex, and quantitative results are hard to obtain. It turns out that, under reasonable simplification, the minimization admits an explicit solution. We then quantify the approximation error in the aforementioned simplification.

\subsection{Quantitative results under simplification}
We start our analysis with the infinite-sample case, i.e., $n=\infty$. Let $\mathcal{L}(\bW) \coloneqq \E_{\zz} \mathcal{L}_n(\bW)$ be the population loss. Denote by $\bW = \sum_{j=1}^p \sigma_j \uu_j \vv_j^\top$ the SVD of $\bW \in \mathbb{R}^{p \times p}$, where $\sigma_1\ge\ldots \sigma_p\ge 0$, $[\uu_1,\ldots, \uu_p]$ and $[\vv_1,\ldots, \vv_p]$ are orthogonal matrices. Define 
$$
\alpha \coloneqq \alpha(\bW) = (1+\sigma_\aug^2)\| \bW \|_{\mathrm{F}}^2 + \| \bW \bmu \|^2.
$$
A straightforward calculation yields the following decomposition of the population loss.
\begin{proposition}\label{prop:loss}
We have $\calL(\bW) = \calL_{\aalign}(\bW) + \calL_{\unif}(\bW)$ with
\begin{align}
\calL_{\aalign}(\bW) &= \frac{\sigma_{\aug}^2}{\tau \alpha} \| \bW \|_{\mathrm{F}}^2, \notag \\
\calL_{\unif}(\bW) &= - \log 2 -\frac{1}{2} \sum_{j=1}^p \log \left(1 + \frac{2(1+\sigma_\aug^2) \sigma_j^2}{\tau \alpha} \right) + \log \Big( 1 + \exp\Big( - \sum_{j=1}^p \frac{2\sigma_j^2 \langle \bmu, \vv_j \rangle^2}{2(1+\sigma_\aug^2)\sigma_j^2 + \tau \alpha } \Big) \Big) \notag \\
&=: - \log 2 + \calL_{\unif}^{(1)}(\bW) + \calL_{\unif}^{(2)}(\bW).
\end{align}
\end{proposition} 

\noindent See Appendix~\ref{sec:loss_calc} for the proof.
\medskip

To understand the two loss components, let us make some observations. 
\begin{enumerate}
\item The main effect of $\calL_{\aalign}(\bW)$ is to align positive pairs. If we only minimize $\calL_{\aalign}(\bW)$, then we need to maximize\footnote{Note that $\alpha$ also depends on $\bW$.} $\| \bW \bmu \|$ for any fixed value of $\| \bW \|$. In other words, $\bW$ needs to stretch $\bmu$ as much as possible. This is intuitive since positive pairs after such stretching and normalization will be closer to each other; see the middle plot of Figure~\ref{fig:gmm-illustration}. 
\item The main effect of $\calL_{\unif}(\bW)$ is to expel negative pairs. If we only minimize one critical part $\calL_{\unif}^{(1)}(\bW)$, then given fixed $(\sigma_j)_{j \le p}$ we need to minimize $\| \bW \bmu \|$ (i.e., $\bW$ compresses $\bmu$ as much as possible), leading to $|\langle \bmu , \vv_p \rangle|=1$. Embeddings after shrinkage and normalization will be more diversely distributed; see the right plot of Figure~\ref{fig:gmm-illustration}. 
\end{enumerate}

We can gain further insights by considering a \emph{first-order approximation} of the loss: intuitively, if the signal strength $\| \bmu \|^2$ is large, then we may expect $\alpha \gg \| \bW \|_{\mathrm{F}}^2 \ge \max_j \sigma_j^2$, so we can try a first-order expansion by treating $\alpha$ as a diverging quantity.
This motivates us to define an approximate loss:
\begin{align}\label{eq:approx-loss}
\tilde \calL(\bW) &\coloneqq \frac{\sigma_{\aug}^2}{\tau \alpha} \| \bW \|_{\mathrm{F}}^2 - \log 2 - \sum_{j=1}^p \frac{(1+\sigma_\aug^2) \sigma_j^2}{\tau \alpha}+  \log \Big( 1 + \exp\Big( - \sum_{j=1}^p \frac{2\sigma_j^2 \langle \bmu, \vv_j \rangle^2}{ \tau \alpha } \Big) \Big) \nonumber \\
& = - \frac{t}{\tau} + \log\Big( \frac{1}{2} + \frac{1}{2}\exp\big( - \frac{2}{\tau} + \frac{2(1+\sigma_\aug^2)}{\tau}t \big) \Big), 
\end{align}
where $t \coloneqq t(\bW)= \frac{\|\bW\|_{\mathrm{F}}^2}{\alpha(\bW)}$. 
As an approximation, it always holds that $\mathcal{L}(\bW) \ge \widetilde{L}(\bW)$. Now that $\widetilde{L}(\bW)$ is actually a univariate function, we can state precise quantitative results.

\paragraph{Theoretical prediction.} Now we are ready to theoretically demonstrate the two key effects---expansion and shrinkage. 
Let $F(t; \sigma_\aug^2, \tau) \coloneqq \frac{\mathsf{d}}{\mathsf{d}t} \widetilde \calL(\bW)$. 

\begin{definition}
A three-parameter configuration $(\sigma_\aug^2, \tau, \norm{\bmu}^2)$ is said to be in the 
\begin{itemize}[topsep=-3pt]
\setlength\itemsep{0.01em}
\item expansion regime if $F\big(\frac{1}{1 +  \sigma_\aug^2 + \norm{\bmu}^2}; \sigma_\aug^2, \tau\big) > 0$,
\item shrinkage regime if $F\big(\frac{1}{1 +  \sigma_\aug^2 + \norm{\bmu}^2}; \sigma_\aug^2, \tau\big) < 0$. 
\end{itemize}
\end{definition}

The expansion and shrinkage regimes characterize the distinctive behavior of the projector $\bW$ at minimization, as stated in the next theorem.

\begin{theorem}[expansion/shrinkage phase transition]\label{thm:phase}
The properties of the minimizer $\bW^*$ of $\tilde \calL(\bW)$ depends on the configuration of $(\sigma_\aug^2, \tau, \norm{\bmu}^2)$. Specifically, with the notation $t^* \coloneqq t(\bW^*)$,
\begin{itemize}[topsep=-3pt]
\setlength\itemsep{0.01em}
    \item in the expansion regime, $\widetilde \calL(\bW)$ is minimized at $t^* = 1/(1+ \sigma_\aug^2+ \norm{\bmu}^2)$, which happens if and only if $\sigma_2=\cdots=\sigma_p=0$ and $\langle \vv_1, \bmu \rangle^2 = \| \bmu \|^2$; 
    \item in the shrinkage regime, $\widetilde \calL(\bW)$ is minimized at 
    \[
    t^* = \frac{1}{1+\sigma_\aug^2} \left(1 - \frac{\tau}{2} \log(1+2\sigma_\aug^2) \right) \in \left[\frac{1}{1+ \sigma_\aug^2+\norm{\bmu}^2},\frac{1}{1+\sigma_\aug^2}\right]
    \]
    that solves $F(t) = 0$. Moreover,  if $\sigma_\aug^2 \to 0$, then $t^* \to 1$ and $\norm{\bW \bmu} / \norm{\bW}_{\mathrm{F}} \to 0$. 
\end{itemize}
\end{theorem}

\noindent See Appendix~\ref{sec:proof-phase-transition} for the proof. 
\medskip

Several remarks are in order. First, 
when the configuration $(\sigma_\aug^2, \tau, \norm{\bmu}^2)$ is in the expansion regime, optimizing the contrastive loss leads to a linear projector $\bW^*$ that maximally expands along the signal direction $\bmu$.
Second, in contrast, in the shrinkage regime, the signal direction is compressed in $\bW^*$, since for fixed value $\norm{\bW}_{\mathrm{F}}$, a vanishing $\norm{\bW \bmu}$ implies $\max_j| \sigma_j \langle \vv_j, \bmu \rangle| \to 0$---any singular vector positively correlated with $\bmu$ must have a vanishing singular value.



\begin{figure*}[t]
    \centering
    \includegraphics[height=0.3\textwidth]{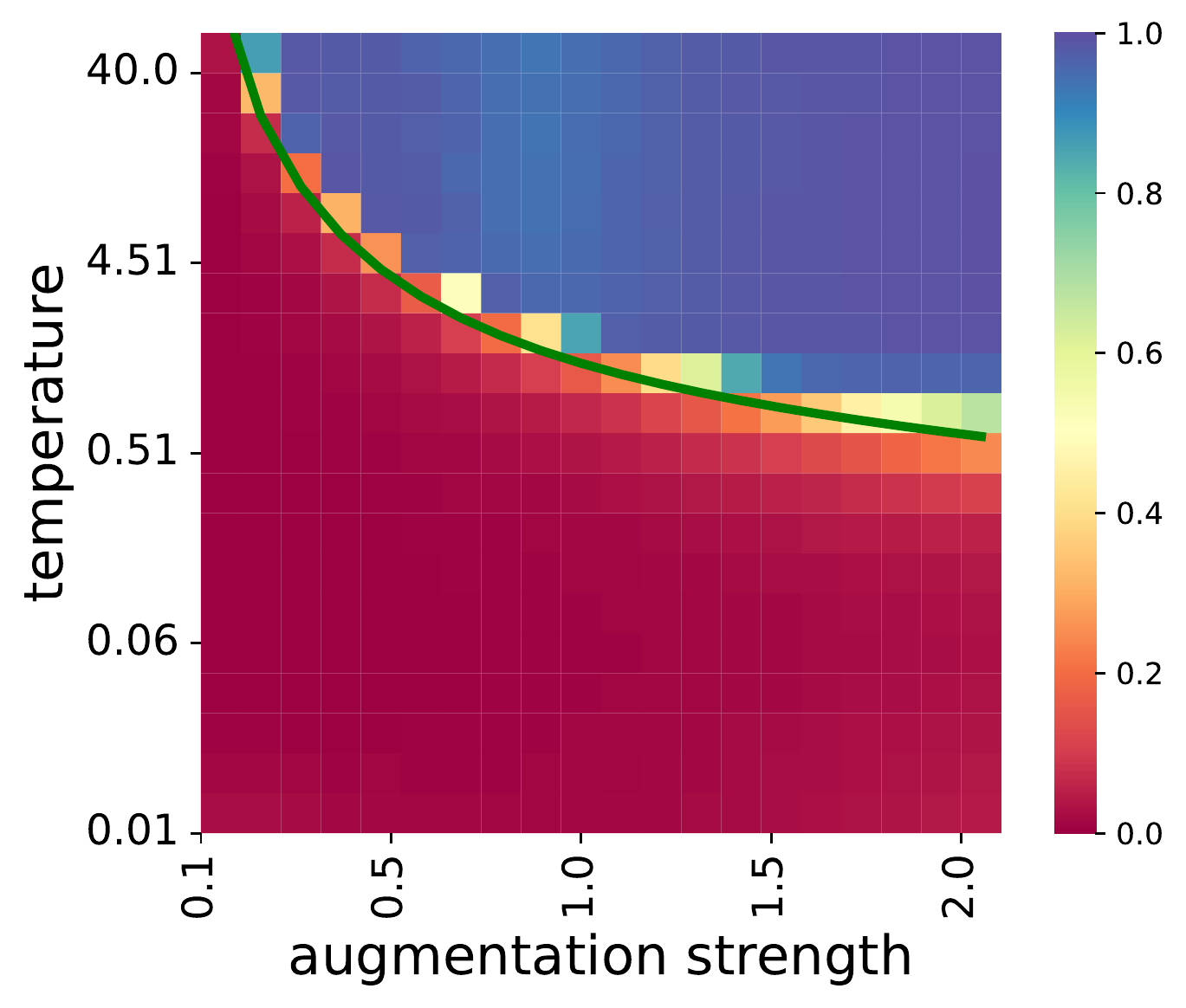}\qquad 
    \includegraphics[height=0.3\textwidth]{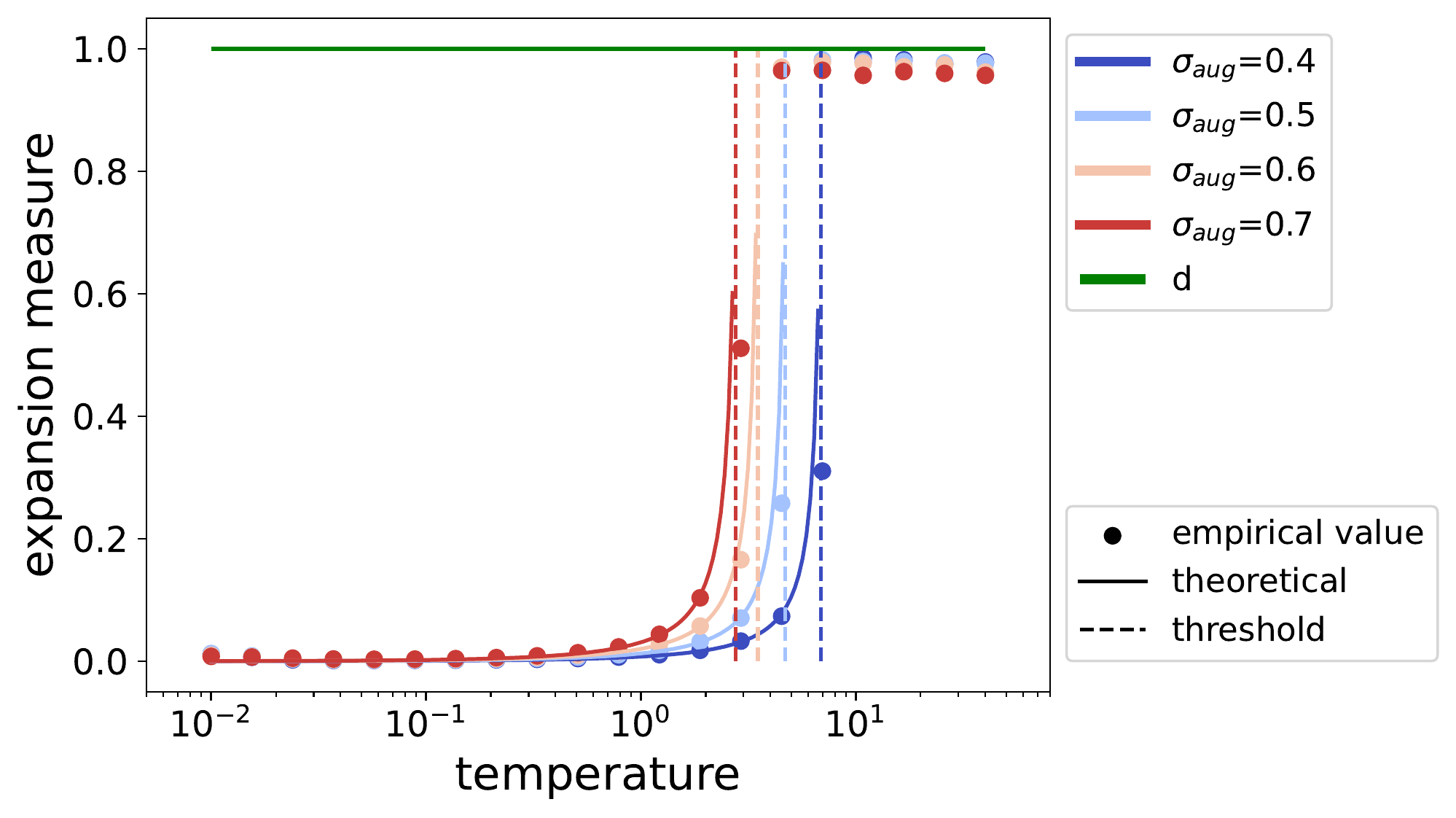}
    \caption{Expansion/shrinkage phase transition for the configuration $n=2000$, $p=50$, $\norm{\bmu}=5$. {\bf Left}: heatmap of expansion measure $T^* \coloneqq T(\bW^*)$; {\bf green curve}: theoretical phase transition.
    {\bf Right}: $T^*$ vs.~$\tau$ with different fixed $\sigma_\aug$.}
    \label{fig:gmm_plot}
\end{figure*}


\paragraph{Numerical experiments. }
Under the 2-GMM model stated in Section~\ref{sec:setup},
in Figure~\ref{fig:gmm_plot} (left), we plot the heatmap of the expansion measure $T^* \coloneqq T(\bW^*)$ with $T(\bW) \coloneqq  \| \bW \bmu\|^2 / (\| \bW \|_{\mathrm{F}}^2 \| \bmu \|^2) \in [0,1]$ with different pairs of temperature and augmentation strength. The green curve is the transition curve
\[
\tau^* = 2\Vert \bmu \Vert^2\left\{(1+\sigma_{\aug}^2+\Vert \bmu \Vert^2)\log(1+2\sigma_{\aug}^2)\right\}^{-1}
\]
which separates the expansion and shrinkage regimes. 
In Figure~\ref{fig:gmm_plot} (right), with a set of $\sigma_{\aug}$, we plot the expansion measure against varying temperatures which shows that our theoretical prediction is also precise in the regime when $0 < T^* < 1$. 
When $\tau < \tau^*$ (shrinkage), $T^*$ is close to zero, indicating a significant compression along $\bmu$; when $\tau > \tau^*$ (expansion), $T^* \approx 1$, corresponding to maximal expansion along $\bmu$. 

Despite simplification and approximation, our theoretical prediction does align well with empirical phase transition with the original SimCLR loss.

\subsection{Extensions to the general case}
Recall that $\tilde \calL$ is an approximation to the original loss $\calL$. Therefore it is of interest to quantify 
the approximation error in the optimal solutions. We also briefly discuss the technical difficult of extending these analysis to the finite-sample regime.

\paragraph{Approximation errors.} Compared with $\tilde \calL$, the original loss $\calL$ cannot be expressed as a simple univariate function. Recall that in the approximation, we consider large signal strength $\| \bmu \|$. Denote $\bW_{{\rm simclr}}$ as the solution to the original loss $\calL$ and $t_{{\rm simclr}} \coloneqq t(\bW_{{\rm simclr}}) = \| \bW_{{\rm simclr}} \| / ((1+\sigma_\aug^2)\| \bW_{{\rm simclr}}\|^2 + \|\bW_{{\rm simclr}} \bmu \|^2)$.
Approximation bounds are shown in the next proposition for the expansion and shrinkage regimes, respectively.
\begin{proposition}\label{prop:approximation}
The difference between $t^*$ and $t_{{\rm simclr}}$ has the following bounds:
\begin{itemize}
\item in the expansion regime,
\[
0 \leq t_{{\rm simclr}} - t^* \leq \left(\frac{2(1+\sigma_\aug^2)}{1+e^{2/\tau}} - 1\right)^{-1}\frac{4(1+\sigma_\aug^2)}{\tau e^{1/\tau}} \frac{1}{\|\bmu\|^2} +  C(\tau, \sigma_\aug^2) \|\bmu\|^{-4} ;
\]
\item in the shrinkage regime, 
\[
\| t_{{\rm simclr}} - t^* \|^2 \leq \frac{e^{2\sigma_\aug^2/\tau}}{(1+\sigma_\aug^2)^2} \frac{1}{p-1} + C(\tau)\cdot \left(\frac{\sigma_\aug^4}{\|\bmu\|^2}\right).
\]
\end{itemize}
where $C(\tau, \sigma_\aug^2), C(\tau)>0$ are some constants.
\end{proposition}

\noindent See Appendix~\ref{sec:append_appro} for the proof. 
\medskip

\noindent We can infer the following properties of the approximation error bounds.
\begin{itemize}
    \item For the expansion regime, the quality of approximation improves almost linearly in $\| \bmu \|^2$, and the dominating bound vanishes as $\tau \rightarrow +\infty$; 
    \item For the shrinkage regime, a sufficiently large dimension $p$ implies a smaller approximation error, and when $\sigma_\aug^2$ is adequately small, the error bound becomes tighter as $\sigma^2_\aug$ increases.
\end{itemize}


\paragraph{Finite-sample regime.} Now we comment on the difficulty when dealing with finite samples, as 
opposed to the infinite-sample regime we have been focusing on thus far. 
In the finite-sample regime, the loss function becomes
\begin{align}
    &\mathcal{L}_n(\bW) \coloneqq -\frac{1}{\tau} \EE_n  \left\{ \EE_{\hh^+,\hh \mid \hh} \big[ \mathrm{sim}^*(\bW\hh, \bW\hh^+) \big| \hh_0\big]\right\} + \log \left(\EE_n \left\{\EE_{\hh^-,\hh \mid \hh_0} \big[ e^{ \mathrm{sim}^*(\bW\hh, \bW\hh^-)/\tau }\big| \hh_0 \big]\right\} \right),
\end{align}
where $\EE_n[a]$ denotes the expectation over the empirical distribution of $\{\hh_{0,i}\}_{i \leq n}$. To write $\calL_n(\bW)$ as a more explicit function in $\bW$, denote $\bM = (\bI+(\tau \alpha)^{-1}(1+\sigma^2_\aug) \bW^\top \bW)^{-1}$. Additionally, define 
\begin{align*}
S_{\hh_0} &\coloneqq \exp\left(-\frac{1}{2(1+2\sigma_\aug^2)}(\hh_0 - \bmu)^\top (\bI-\bM) (\hh_0 - \bmu)\right)\\ 
& \qquad \qquad +  \exp\left(-\frac{1}{2(1+2\sigma_\aug^2)}(\hh_0 + \bmu)^\top (\bI-\bM) (\hh_0 + \bmu)\right),\\
\tilde S_{\hh_0} &\coloneqq \exp\left(-\frac{1}{2\tau \alpha}\|\bW(\hh_0-\bmu)\|^2\right) + \exp\left(-\frac{1}{2\tau \alpha}\|\bW(\hh_0+\bmu)\|^2\right),
\end{align*}
then we have the following proposition.
\begin{proposition}\label{prop:finite}
The finite sample loss $\calL_n(\bW)$ can be written as
\begin{align*}
\calL_n(\bW) = - \log 2 + \frac{\sigma_\aug^2}{\tau \alpha} \| \bW \|_{\mathrm{F}}^2 + \frac{1}{2}\log {\rm det}[\bM] + \log\left(\E_n S_{\hh_0}\right).
\end{align*}
Further, if $\alpha \gg \|\bW\|_{\mathrm{F}}^2$, $\calL_n(\bW)$ has the following approximation
\begin{align*}
\tilde\calL_n(\bW) = - \log 2 - \frac{1}{2\tau \alpha} \|\bW\|_{\mathrm{F}}^2 + \log \left(\EE_n \tilde S_{\hh_0}\right).
\end{align*}
\end{proposition}
From the proposition, the difficulty in analyzing the finite-sample loss lies in the nonlinearity in $\log {\rm det}[\bM]$ as well as the terms $S_{\hh_0}$. If we assume that $\alpha \gg \|\bW\|_{\mathrm{F}}^2$, the approximation is based on the fact that $\bM \approx \bI - (\tau \alpha)^{-1}(1+\sigma_\aug^2)\bW^\top\bW$.
\begin{remark}
As $\hh_0 \sim \calN(y \cdot \bmu, \bI)$, if we replace $\EE_n$ by $\EE_{\hh_0}$ in $\tilde \calL_n(\bW)$, then the loss can be approximated by
\[
-\log 2 - \frac{1}{\tau \alpha}\|\bW\|_{\mathrm{F}}^2 + \log\left(1 + \exp\left(-\frac{2}{\tau \alpha} \|\bW \bmu\|^2\right)\right),
\]
which is exactly the loss function $\tilde \calL(\bW)$ defined in \eqref{eq:approx-loss}.
\end{remark}



\section{Effect of projectors on generalization}\label{sec:result}
We move on to investigate the effect of projectors on the generalization  of learned representations. 

\subsection{One-parameter projector model}

To understand the generalization puzzle about projectors, an important empirical observation is that pretrained features typically form \textit{linearly separable} clusters; see Figure~\ref{fig:cluster}. 
Adopting our \mbox{2-GMM} view, we assume that cluster labels $\{y_i\}_{i \le n}$ are revealed to us on downstream tasks, and our data $\{\hh_{0,i}, y_i\}_{i \le n}$ follow the model
\begin{align*}
   \hh_{0,i} \mid y_i \stackrel{\mathrm{i.i.d.}}{\sim} \calN(y_i\cdot \bmu, \bI_p), \quad \text{where} \quad y_i  \stackrel{\mathrm{i.i.d.}}{\sim} \mathrm{Unif}(\{-1,1\}).
\end{align*}
Assuming a linear projector as before, we denote $\zz_i = \bW \hh_{0,i}$ and $\{\zz_i, y_i\}_{i \leq n}$ are the inputs of the classification problem.

In reality, the minimizer from the standard SimCLR loss~\eqref{expr:contrastiveloss} is not perfectly characterized by the ideal expansion/shrinkage; see the singular value plot in Figure~\ref{fig:image_pretrained_plot}. Nevertheless, to gain insights, we consider a simpler form of linear transforms:
\begin{equation*}
\mathcal{W} = \big\{ \bW_\eta = \bI_p + \eta\rho^{-1}  \bmu \bmu^\top: \eta > -1 \big\}, ~ \text{where}~\rho \coloneqq \| \bmu\|^2.
\end{equation*} 
Since only the right singular vectors of $\bW$ is of interest in terms of expansion/shrinkage, the symmetric projection head is without loss of generality.
If we constrain ourselves to this simple one-parameter space $\mathcal{W}$, then expansion/shrinkage regimes are solely determined by $\eta$. Specifically, $\eta = 0$ means effectively no projector, $\eta>0$ corresponds to the expansion regime, while $\eta<0$ corresponds to the shrinkage regime. 
\begin{equation*}
\begin{split}
\left\{ 
\begin{array}{ll}
\text{If}~\eta \in (-1,0), & \text{then shrinkage regime}; \\  \text{If}~\eta \in (0,\infty), & \text{then expansion regime}; \\
  \text{If}~\eta  = 0, & \text{effectively no projector}.
\end{array}
\right.
\end{split}
\end{equation*}

\subsection{Confirming folklore: invariance of test errors in low-dimension regime}

In the low-dimensional regime (i.e., when $p/n \to 0$), the data $(\hh_{0,i},y_i)_{i \le n}$ are not linearly separable with probability approaching one. We focus on $\ell_2$-regularized logistic regression with an intercept in a fixed dimension $p$. 
Define
\begin{align}\label{eq:loss_log_l2}
    \ell_{n}(\gamma,\bbeta; \lambda_n)
    = \EE_n \left\{\log \left[1 + e^{-y\left(\gamma+\zz^\top \bbeta\right)}\right]\right\} + \lambda_n \Vert \bbeta \Vert^2,
\end{align}
where $\E_n$ denotes the sample average over $n$ samples $i=1,\ldots,n$. 
With coefficients $\gamma$, $\bbeta$ and a given $\bW = \bW_\eta$, define the test error as $\Err\,(\gamma, \bbeta; \eta)  \coloneqq  \mathbb{P}(\gamma + \tilde{y} \langle \tilde{\zz} , \bbeta \rangle < 0)$,
where $(\tilde{\zz}, \tilde{y})$ is a new independent sample. 
We further define the expected error $\overline{\Err}(\eta) = \EE[\Err(\hat \gamma, \hat \bbeta; \eta)]$, where $\hat \gamma, \hat \bbeta$ is the logistic regression estimator, and the expectation is taken w.r.t.~the training samples. Intuitively, the linear transformation will not change the generalization error.
The following result confirms this intuition by proving that the effects of expansion or shrinkage are negligible, unless we add an unreasonably large regularizer; see the second case below.


\begin{theorem}\label{thm:low-d}
Consider minimizing the regularized logistic loss function~\eqref{eq:loss_log_l2}.
\begin{enumerate}[topsep=-3pt]
 \setlength\itemsep{0.01em}
    \item If $0 \leq \lambda_n \ll \sqrt{n}$, then the test error obeys 
    \begin{align}
        \overline \Err\,(\eta) 
        & = \Phi(-\Vert \bmu\Vert) + o(1).
    \end{align}
    Here, the dominant term (i.e., the first term) in the test error remains the same for varying $\eta$.
    \item If $\lambda_n \asymp \sqrt{n}$, then $\overline\Err\,(\eta)$ is decreasing in $\eta$.
\end{enumerate}
\end{theorem}

\noindent See Appendix~\ref{sec:proof_low_d} for the proof. 


\subsection{High-dimensional regime: inductive bias matters}

In search of the explanation, we then turn to the high-dimensional regime. 
In the high-dimensional regime, two distinct phenomena arise: first, $\{\hh_{0,i}\}_{i \leq n}$ are linearly separable with high probability. Second, solving logistic regression on linearly separable data via gradient descent results in a max-margin classifier~\citep{Rosset2003, Soudry2018TheIB}---a form of inductive bias~\citep{neyshabur2014search}. Recall that 
$\zz_i = \bW \hh_{0, i}$, and 
the max-margin classifier is given by
\begin{align}\label{opt:maxmargin}
\max_{\bbeta}   \min_{i \le n} y_i \langle \zz_i, \bbeta \rangle \qquad 
\text{subject to} \;  \| \bbeta \| \le 1. 
\end{align}
To understand the different inductive biases brought by different linear projectors, let us 
rewrite the max-margin problem above using the original input data:
denoting $\tilde \bbeta = \bW \bbeta$, the optimization problem \eqref{opt:maxmargin} can be reformulated into
\begin{equation}\label{opt:max-margin2}
\begin{split}
\begin{array}{rcl}
\displaystyle\max_{\tilde \bbeta} & & \displaystyle\min_{i \le n} y_i \langle \hh_{0,i}, \tilde \bbeta \rangle \\
\text{subject to} & & \norm{\tilde \bbeta}_{(\bW \bW^\top)^{1/2}} \le 1,
\end{array}
\end{split}
\end{equation}
where  $\norm{\uu}_{\bA} := \uu^\top \bA^{-1} \uu$ is a norm associated with the positive definite matrix $\bA$.
Clearly the inductive bias by $\bW$ manifests itself via the induced norm $\norm{\cdot}_{(\bW \bW^\top)^{1/2}}$. 
Therefore, it is reasonable to expect that reparametrization due to expansion or shrinkage affects generalization in high dimensions.
The question then boils down to provably demonstrate the effect of $\bW$ on the downstream classification accuracy
$\Err(\bbeta;\eta) \coloneqq  \mathbb{P}(\tilde{y} \langle \tilde{\zz} , \bbeta \rangle < 0)$ with new independent sample $(\tilde z, \tilde y)$.

\subsection{Precise asymptotic characterization in high dimension}
Before presenting the main result, let us define some useful quantities. 
The solution is denoted by $\hat \bbeta$ (unique in the separable case). 
Second, define $\bP_{\bmu} \coloneqq \rho^{-1} \bmu \bmu^\top$, $\bP_{\bmu}^\perp \coloneqq \bI_p - \rho^{-1}\bmu \bmu^\top$, the ratio $\hat u \coloneqq \| \bP_{\bmu}^\perp \bW_\eta\hat \bbeta \| / \| \bP_{\bmu} \bW_\eta\hat \bbeta \|$.
We have the following theoretical guarantees. 
\begin{theorem}\label{thm:CGMT}
Suppose that $n / p \to \delta > 0$, $\bW_\eta\in\calW$, and $\|\bmu\|$ is a constant. There exists $\delta^*(\rho) > 0$ such that the following holds.
\begin{enumerate}[topsep=-3pt]
\setlength\itemsep{0.1em}
\item (non-separability) If $\delta > \delta^*(\rho)$, then with probability approaching one, 
$(\hh_i, y_i)_{i \le n}$ is not linearly separable and $\hat \bbeta = \mathbf{0}$.
\item (separability) If $\delta < \delta^*(\rho)$, then with probability approaching one,
there exists a unique solution $\hat \bbeta$ to \eqref{opt:maxmargin} with the margin
\[
\hat \kappa = \min_{i\le n} y_i \langle \zz_i, \hat \bbeta \rangle \xrightarrow{p} \kappa^*(\rho, \eta) > 0.
\]
\item (monotonicity of error) If $\delta < \delta^*(\rho)$, $\hat u \xrightarrow{p} u^*(\rho, \eta)$ and $u^*(\rho,\eta)$ is monotonically decreasing in $\eta$. 
Moreover, the test error obeys
\begin{equation*}
\Err\,(\hat \bbeta; \eta) \xrightarrow{p} \Phi\left(- \rho^{1/2} \sqrt{\tfrac{1}{1+[u^*(\rho, \eta)]^2}}\right),
\end{equation*}
where $\Phi$ denotes Gaussian CDF.
Thus, the asymptotic test error is decreasing in $\eta$. 
\end{enumerate}
\end{theorem}

\noindent See Appendix~\ref{sec:proof_high_d} for the proof, and the precise definitions of $\kappa^*$ and $u^*$. 
\medskip

While separability thresholds (claims 1 and 2) are known in similar models \citep{deng2022model}, the third claim (the most interesting and important one) shows that in the high-dimensional regime, even a linear projector can change the test accuracy of a ``linear'' max-margin classifier, and the test accuracy increases with the expansion strength $\eta$. This partially explains the puzzling effect of projectors on downstream performance. We remark in passing that this result is based on a recent technique known as the convex Gaussian minimax theorem~\citep{gordon1988milman, pmlr-v40-Thrampoulidis15}. 


\paragraph{Numerical evidence.}

In Figure~\ref{fig:err}, 
we use the $2$-GMM with separation parameter $\rho \in \{1,2,3,5\}$ to generate a dataset $(\hh_{0,i})_{i\le n}$ of size $n=1000$ and dimension $p \in \{200,600,2000\}$. We apply $\bW = \bI_p + \eta \rho^{-1} \bmu \bmu^\top \in \mathcal{W}$ to obtain embeddings $(\zz_i)_{i\le n}$. We treat the mixture membership as labels $(y_i)_{i\le n}$ and compute the max-margin classifier on $(\zz_i, y_i)_{i\le n}$. For varying $\rho,p,\eta$, we report the test error of the max-margin classifier and confirm the monotone decreasing property.

\begin{figure}[t]
    \centering
    \includegraphics[width=0.9\textwidth]{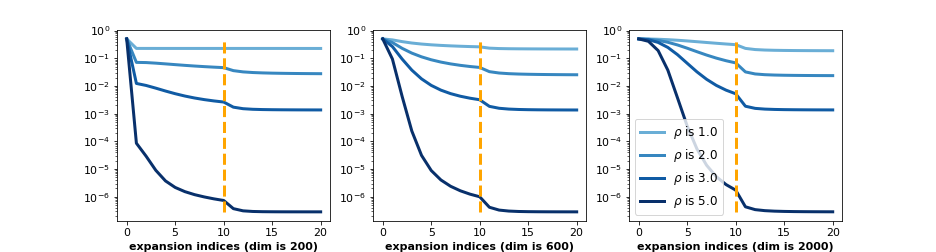}
    \caption{Test error is monotone in $\eta$ under the one-parameter projector model. We apply a linear transform $\bW \in \mathcal{W}$ on train data from $2$-GMM, and calculate the test errors of the max-margin classifiers. Note $\eta=0$ (vertical dashed) is equivalent to no transform.}
    \label{fig:err}
\end{figure}


\section{Extensions to inhomogeneous feature augmentation}\label{sec:inhomogeneous}

So far our analysis applies to scenarios where \textit{either} expansion \textit{or} shrinkage appears. 
However, in practice, it is not uncommon to encounter situations where both expansion and shrinkage appear. 
This does not render our previous analysis vacuous. As we will demonstrate in this section, 
we can extend the previous analysis to the case with inhomogeneous augmentations, under which both 
expansion and shrinkage can appear. Our treatment throughout this section parallels that in Section~\ref{sec:gmm_expansion_shrinkage}. 

\subsection{Feature augmentation with spiked covariance}

Suppose that features are generated from the same $2$-GMM but augmentations are inhomogeneous: 
\begin{align*}
    \hh_{0,i} \stackrel{\mathrm{i.i.d.}}{\sim} \frac{1}{2} \mathcal{N} (-\bmu, \bI_p) + \frac{1}{2} \mathcal{N} (\bmu, \bI_p), \quad \text{and} \quad
\hh_i,\hh_i^+ | \,\hh_{0,i} \stackrel{\mathrm{i.i.d.}}{\sim} \mathcal{N}(\hh_{0,i}, \sigma_{\mathrm{aug}}^2 \bA),
\end{align*}
where $\bA \succeq \bI_p$ is covariance matrix. Throughout this section, we assume $p > 2$. This inhomogeneous model is supported by empirical evidence. In the appendix, we show that image-level augmentation (random cropping, color distortion, etc.) does produce inhomogeneous features, which lead to more realistic phenomena. 

Now following the same setup as before: we consider linear projectors $\zz = \bW \hh$, the modified SimCLR loss \eqref{def:Ln}, the population counterpart $\calL(\bW) = \E_{\zz} \calL_n(\bW)$ and its approximation $\tilde \calL(\bW)$. To ease the notation, let $\tilde\bW = (\bI + \sigma_\aug^2)^{1/2} \bW$, $\tilde\bmu = (\bI + \sigma_\aug^2)^{-1/2} \bmu$ and $\sum_{j=1}^p \tilde \sigma_j^2 \tilde \uu_j^\top \tilde \vv_j$ be the SVD of $\tilde\bW$.
The following proposition gives the explicit formulas of $\calL(\bW)$ and $\tilde \calL(\bW)$.

\begin{proposition}\label{prop:loss-aug}
Define
\[
\tilde \alpha := \tilde \alpha(\bW) = \EE\| \bW \hh \|^2 = \|\bW\|_{\mathrm{F}}^{2}+\|\bW\bmu\|^{2}+\sigma_\aug^{2}\tr\left(\bW^\top \bA \bW\right).
\]
The loss $\calL(\bW)$ takes the form $\calL(\bW) = \calL_{\aalign}(\bW) + \calL_{\unif}(\bW)$,
where
\begin{align*}
\calL_{\aalign}(\bW) &= -\frac{1}{\tau \tilde\alpha}\Vert \bW\Vert_{\mathrm{F}}^{2}\\
\calL_{\unif}(\bW) &= -\frac{1}{2} \sum_{j=1}^p \log \left(1 + \frac{2 \tilde \sigma_j^2}{\tau \tilde \alpha} \right) + \log \left( 1 + \exp\Big( - \sum_{j=1}^p \frac{2\tilde \sigma_j^2 \langle \tilde \bmu, \tilde \vv_j \rangle^2}{2\tilde \sigma_j^2 + \tau \tilde \alpha } \Big) \right) - \log 2.
\end{align*}
\end{proposition}

\noindent Similar to before, one can consider a first-order approximation $\tilde\calL(\bW)$:
\begin{align}
\tilde \calL(\bW) &= -\frac{1}{\tau\tilde \alpha}\Vert \bW\Vert_{\mathrm{F}}^{2}+\log\left(1+\exp\left(-\frac{2\|\bW \bmu\|^{2}}{\tau\tilde \alpha}\right)\right)-\log 2. \label{def:appro-loss-inhomo}
\end{align}
The loss $\tilde \calL(\bW)$ depends on $\bW$ through $\| \bW \|$, $\| \bW \bmu \|$ and $\| \bW \vv_\aug\|$. 
We further consider a simple one-spike covariance model 
\begin{equation}\label{eq:onespike}
\bA = \bI_p + \rho_{\aug} \vv_{\aug}\vv_{\aug}^\top, \qquad \text{where}~\norm{\vv_{\aug}}=1\, .
\end{equation}
Here $\rho_\aug \ge 0$ quantifies the strength of the spike in augmentations. In particular, setting $\rho_\aug = 0$ recovers the homogeneous case.
It is natural to expect that $\mathrm{span}\{ \bmu, \vv_\aug\}$ plays a critical role in its minimization. For that purpose, let us define an orthogonal basis in $\mathrm{span}\{ \bmu, \vv_\aug\}$.
Let $\bar{\bmu} = \bmu/\|\bmu\|$, $r = \langle \bar{\bmu}, \vv_\aug \rangle$, and $\bmu_{\perp}$ be a vector such that $\| \bmu_{\perp} \|=1$, $\langle \bmu_{\perp}, \bmu \rangle = 0$ and $r \langle \bmu_{\perp}, \vv_\aug \rangle \geq 0$. 

\paragraph{Theoretical prediction.} On the surface, the loss $\tilde \calL(\bW)$ here is not a univariate function. It turns out, as our proof reveals, that minimization of $\tilde \calL(\bW)$ is equivalent to a univariate minimization problem that depends only on $T = \| \bW \bar{\bmu} \|^2 / \| \bW \|_{\mathrm{F}}^2$.
Define the threshold
\[
\tau_1^* = \frac{2 (1-r^2) \| \bmu \|^2}{ \log(1+2\sigma_\aug^2)(1+\sigma_\aug^2+(1-r^2) \| \bmu \|^2)} \,.
\]
We have the following result that characterizes the phase transition of the minimizer $\bW^*$.
\begin{theorem}(phase transition under spiked augmentation)
Consider the one-spike inhomogeneous model \eqref{eq:onespike} and the approximate loss \eqref{def:appro-loss-inhomo}. Recall the definition $T := T(\bW) = \| \bW \bar{\bmu} \|^2 / \| \bW \|_{\mathrm{F}}^2$. Let $\bW^*$ be a minimizer of $\tilde \calL(\bW)$ and $T^* = T(\bW^*)$. Then, $T^*$ is given by the minimization problem
\begin{align}
\begin{split}\label{opt:T}
\min_{T \in [0,1]} \quad &-\frac{1}{\tau\left[(1+\sigma_\aug^{2})+\|\bmu\|^{2}T + \rho_\aug \sigma_\aug^2 [ (r \sqrt{T} - \sqrt{1-r^2}\, \sqrt{1-T})_+ ]^2 \right]}\\
&+\log\left(1+\exp\left(-\frac{2\|\bmu\|^{2}T}{\tau\left[(1+\sigma_\aug^{2})+\|\bmu\|^{2}T + \rho_\aug \sigma_\aug^2 [ (r \sqrt{T} - \sqrt{1-r^2}\, \sqrt{1-T})_+ ] ^2\right]}\right)\right).
\end{split}
\end{align}
where $x_+$ denotes $\max\{x,0\}$. Under $\rho_\aug > 0$ and $0<|r|<1$, we have
\begin{itemize}[topsep=-3pt]
    \item if $\tau > \tau_1^*$, then the minimizer is attained at $T^* \in (1-r^2, 1)$, which is associated with a rank-one projector $\bW^*$ given by
    \[
    (\bW^*)^\top \bW^* = (\sqrt{T^*} \bar{\bmu} - \sqrt{1-T^*}\bmu_{\perp})(\sqrt{T^*} \bar{\bmu} - \sqrt{1-T^*}\bmu_{\perp})^\top;
    \]
    \item if $\tau \le \tau_1^*$, the minimizer is attained at $T^* = \frac{\tau(1+\sigma_\aug^2) }{2\| \bmu \|^{2}} \log(1+2\sigma_\aug^2) \left[1-\frac{\tau}{2}\log(1+2\sigma_\aug^2)\right]^{-1}$ and $\bW^* \vv_\aug =0$.
\end{itemize}
\label{thm:inhomo}
\end{theorem}

\noindent See Appendix~\ref{sec:proof-inhomo} for the proof.

\medskip
If $\rho_\aug = 0$ in the optimization problem of Theorem~\ref{thm:inhomo}, then the analysis reduces to the homogeneous case. However, with \textit{any small} spike strength $\rho_\aug>0$ and assuming nondegeneracy ($0<|r|<1$), the phase transition becomes qualitatively very different.
\begin{itemize}
\item The phase threshold is smaller as $\tau_1^* \le \tau^*$, which we recall that 
\[
\tau^* = 2\Vert \bmu \Vert^2\left\{(1+\sigma_{\aug}^2+\Vert \bmu \Vert^2)\log(1+2\sigma_{\aug}^2)\right\}^{-1}.
\]
The difference depends on the cosine angle between the signal $\bar \bmu$ and spike direction $\vv_\aug$, \textit{irrespective of the spike strength} $\rho_\aug$.
\item There is no perfect expansion along the signal direction $\bar \bmu$, as $T^* < 1$ is \textit{always true}.
\end{itemize}

\begin{figure}[t]
    \centering
    \begin{subfigure}[t]{0.72\textwidth}
    \includegraphics[width=\textwidth]{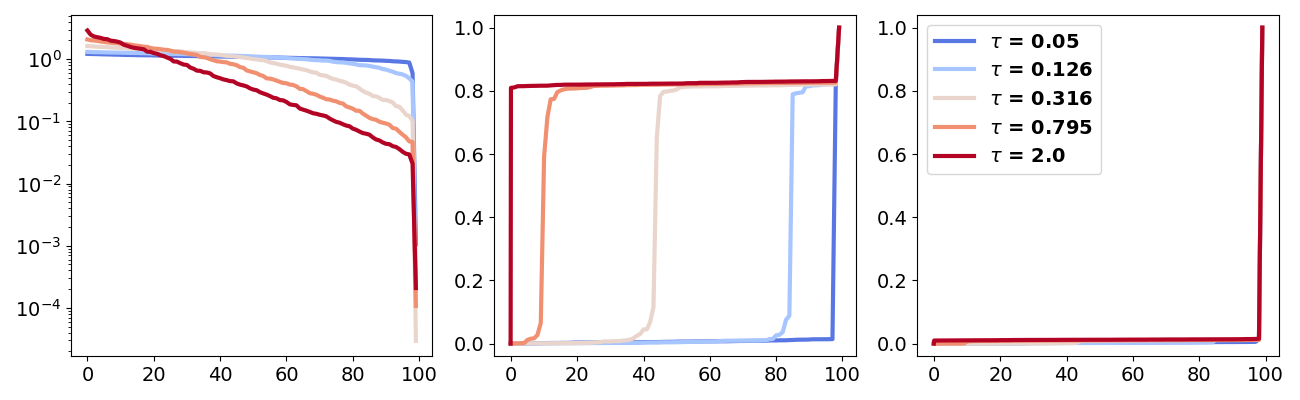}
    \caption{}
    \label{fig:inhomo}
    \end{subfigure}
    \hfill
    \begin{subfigure}[t]{0.27\textwidth}
    \includegraphics[width=\textwidth]{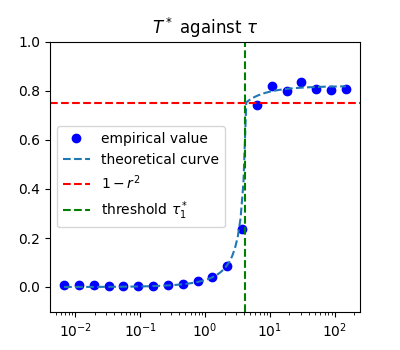}
    \caption{}
    \label{fig:plot_inhomo}
    \end{subfigure}
    \caption{GMM with inhomogeneous augmentation. (a) \textbf{Left}: $p=100$ singular values of $\bW$ in descending order with varying $\tau$ indicated by different colors. (a) \textbf{Middle}: Cumulative score for $\bmu$, i.e., $\mathrm{score}_i(\bmu) = \sum_{j \le i} \langle \bmu/\norm{\bmu}, \vv_j \rangle^2$. (a) \textbf{Right}: Cumulative score for $\vv_\aug$, i.e., $\mathrm{score}_i(\vv_\aug) = \sum_{j \le i} \langle \vv_\aug, \vv_j \rangle^2$. (b) Phase transition showing $T^*$ against $\tau$: theoretical prediction (dashed curve) versus empirical values (circles).}
\end{figure}

\paragraph{Degenerate cases.} It is beneficial to consider two degenerate examples: $r=0$ and $r=1$. When $r=0$, the signal direction and spike direction are orthogonal; and when $r=1$, they are perfectly aligned. 

Analyzing the optimization problem in Theorem~\ref{thm:inhomo} for the degenerate examples yields the following characterization.
\begin{enumerate}
\item When $r=0$ (namely $\bar \bmu \perp \vv_\aug$), we have $\tau_1^* = \tau^*$, and $\bW^* \vv_\aug = 0$, namely pure shrinkage along $\vv_\aug$. Moreover, similar to the homogeneous case, we have shrinkage along $\bar \bmu$ if $\tau \le \tau^*$ and expansion along $\bar \bmu$ if $\tau > \tau^*$.
\item When $|r|=1$ (namely $\bar \bmu \parallel \vv_\aug$), we have expansion along $\bar \bmu$ if 
\begin{equation*}
\tau \cdot \log\Big( \frac{2(1+\sigma_\aug^2)\| \bmu\|^2}{\| \bmu\|^2 + \rho_\aug \sigma_\aug^2} - 1 \Big) > \frac{2 \| \bmu\|^2}{1+\sigma^2+\| \bmu\|^2+\rho_\aug \sigma_\aug^2} \; .
\end{equation*}
\end{enumerate}
In the first example, expansion/shrinkage operates independently along the two directions $\bar \bmu$ and $\vv_\aug$. In the second example, the presence of a parallel spike produces a threshold that is nonlinear in $\rho_\aug$ for expansion/shrinkage along $\bar \bmu$.


\paragraph{Numerical experiments.} Figure~\ref{fig:inhomo} shows the singular vectors of $\bW^*$ after training with the SimCLR loss and also the cumulative scores with $\bar{\bmu}$ and $\vv_\aug$, respectively. We vary the temperature $\tau$, and fix $\rho_\aug=5$, $\sigma_\aug=0.5$, and $r = \langle \bar{\bmu}, \vv_\aug \rangle =0.5$. When $\tau \leq 2$, there is consistent shrinkage in the direction of $\vv_{\aug}$. In addition, for example, when $\tau=2$ or $0.795$, expansion and shrinkage in the direction of $\bmu$ coexist and $\bmu$ is only spanned by either the top singular vector or the bottom singular vector of $\bW^*$, which is consistent with Figure~\ref{fig:image_pretrained_plot}.

\subsection{Simultaneous expansion/shrinkage}

\begin{figure}[t]
    \centering
    \includegraphics[width=0.7\textwidth]{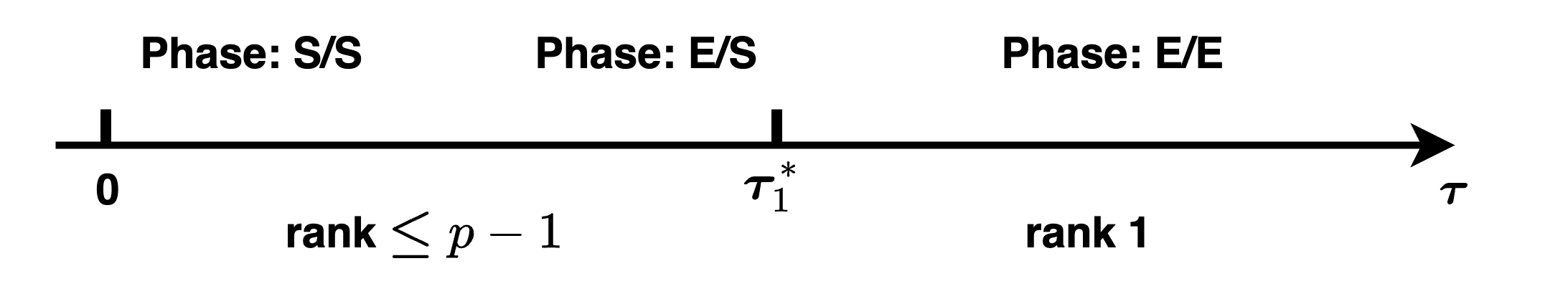}
    \caption{An illustration showing the phase changes in the one-spike inhomogeneous model as we increase $\tau$. In the figure `E' is a shorthand for expansion and `S' for shrinkage; see Corollary~\ref{cor:change}.}
    \label{fig:change}
\end{figure}

The signal $\bmu$ and spike $\vv_\aug$ exhibit different levels of expansion/shrinkage at the same temperature parameter, which sometimes leads to simultaneous expansion/shrinkage. 

Consider the one-spike inhomogeneous model \eqref{eq:onespike} and the approximate loss \eqref{def:appro-loss-inhomo}. Assume nondegeneracy $\rho_\aug > 0$ and $0<|r|<1$. Recall the SVD of $\bW^*$ is $\sum_{j \le p} \sigma_j \uu_j \vv_j^\top$. Without loss of generality we assume $\|\bW^*\|_\mathrm{F} = 1$.

\begin{corollary}[Phase change under varying $\tau$]\label{cor:change}
As we increase the temperature parameter $\tau$, treating other parameters $\|\bmu\|, \sigma^2_\aug, \rho_\aug, r$ as constants, we experience the following different phases.
\begin{enumerate}
\item When $\tau \ll 1$: both shrinkage. We have $\bW^* \vv_\aug = 0$, and also $\| \bW^* \bar \bmu \| \ll  1$, the latter of which implies $\max_{1\leq j \leq p} |\sigma_j \langle \vv_j, \bar \bmu \rangle| \ll 1$.
\item When $\omega(1) \le \tau \le \tau_1^*$: simultaneous expansion and shrinkage. We have $\bW^* \vv_\aug = 0$, and also $\| \bW^* \bar \bmu \| \ge c$ for certain constant $c>0$. There are two jumps in the cumulative score: for certain dimension-free constant $\epsilon \in (0,1)$, we have $\sum_{j \le \epsilon^{-1}} \langle \vv_j, \bar \bmu \rangle^2 \ge \epsilon$ and $\langle \vv_p, \bar \bmu \rangle^2 \ge \epsilon$.
\item When $O(1) \ge \tau > \tau_1^*$: both expansion. $\bW^*$ is a rank-one matrix, and its right singular vector $\sqrt{T^*} \bar{\bmu} - \sqrt{1-T^*}\bmu_{\perp}$ has positive cosine angles with both $\bar \bmu$ and $\vv_\aug$.
\item When $\tau \gg 1$: expansion increasingly aligns with $\bar \bmu$. Note that $\sqrt{T^*}$ is very close to $1$ but always strictly smaller than $1$.
\end{enumerate}
\end{corollary}

\noindent See Appendix~\ref{sec:proof-change} for the proof. 
\medskip

Figure~\ref{fig:change} gives an illustration of the phase change in the above corollary. It offers a theoretical explanation for the singular value/vector plots for \texttt{STL-10} data (Figure~\ref{fig:image_pretrained_plot}) and simulated GMM data (Figure~\ref{fig:inhomo}).

    
    



\section{Discussion and related work}\label{sec:related}

\subsection{Connections to emerging empirical phenomena}

The puzzles about projectors in contrastive learning echo several known phenomena in deep learning. 

\paragraph{Dimensional collapse.} It is often observed that the trained features and embeddings do not span the entire ambient space. To be more precise, the singular values of the feature matrix $\bH = [\hh_1, \ldots, \hh_n]^\top$ and the trained embedding matrix $\bZ = [\zz_1,\ldots,\zz_n]^\top$ (with $\hh_i = \ff_{\btheta}(\xx_{i})$ and $\zz_i = \gggg_{\bvarphi}(\hh_i)$) contain one or more (approximate) zeros. This phenomenon is known as \textit{dimensional collapse}, which has been repeatedly reported in the literature~\citep{chen2020simple, jing2021understanding, balestriero2022contrastive}.\footnote{In fact, dimensional collapse is a more salient issue for non-contrastive approaches in SSL~\citep{hua2021feature, Tian2021UnderstandingSL} due to the lack of negative pairs.} When dimensional collapse occurs in the feature space, we obtain less informative representations. This is generally undesirable according to these papers and requires careful handling due to its adverse effects on generalization. Figure~\ref{fig:image_pretrained_plot} (left) confirms dimensional collapse by showing that the linear projector does not have full rank. 

To address dimensional collapse in contrastive learning (and more so in non-contrastive SSL), a line of work proposes to refine loss functions and design structured projectors \citep{balestriero2023cookbook}, but a systematic treatment is still lacking.

\paragraph{Transferability of intermediate-layer features.} In supervised learning, it is well observed that trained deep neural networks contain interpretable features that become progressively complex when moving up layers \citep{zeiler2014visualizing}. Therefore, it is natural to use intermediate-layer features pretrained on large datasets for related tasks \citep{yosinski2014transferable}. At the very top layers, features are believed to be very specific to a classification task, and thus they need to be finetuned on downstream tasks.

Intuitively, projectors bear similarity to those top layers that require finetuning on downstream tasks. For both supervised learning and contrastive learning, minimizing a specific loss seems to reduce the information in features and thus their generality.

\paragraph{Neural collapse.} In supervised learning, the features in the penultimate layer tend to form a symmetric structure, if we train neural networks many epochs well past the terminal phase where the train error achieves zero. This phenomenon is known as \textit{neural collapse} \citep{papyan2020prevalence}. In short, as training evolves, the penultimate features gradually collapse to their respective class means, which form an equiangular simplex. The highly symmetric and compact cluster structure is observed on the training dataset and rarely on the test dataset (unless the test error is also zero).

Contrastive learning produces weaker cluster structure but the learned features are more general for downstream tasks \citep{wang2020understanding}. The cluster structure resulting from contrastive learning is desirable for generalization, whereas the strong cluster structure resulting from cross-entropy minimization is partly due to the optimization artifact. In fact, a fine-grained intermediate-layer neural collapse suggests that top layers (including the penultimate layer) do not improve and sometimes even harm the generalization properties of features \citep{galanti2022implicit}. 

\subsection{Related work}

\paragraph{Analysis of contrastive learning.} Contrastive learning has received tremendous attention 
in the past few years, and a large body of work has been done around this topic. 
We refer interested readers to the recent overview~\citep{balestriero2023cookbook} for the historical account 
and for the recent updates. Though empirically successful, contrastive learning also brings
various intriguing phenomena including dimensional collapse \citep{jing2021understanding, hua2021feature} and behavior of projectors \citep{chen2020simple,chen2020big,cosentino2022toward}. These motivate quite a few recent theoretical attempts to explain the success of contrastive learning. Most notable and related is the paper by \citet{wang2020understanding}, where they discover that contrastive loss promotes both alignment and uniformity of the learned representations. This viewpoint is instrumental in our analysis and understanding, as can be seen from e.g., Proposition~\ref{prop:loss}. Our work goes beyond alignment/uniformity by precisely characterizing the effect (namely, expansion and shrinkage) of contrastive loss on the projection head. 
\cite{jing2021understanding} also studies the role of projectors by arguing that they prevent dimensional collapse in the representation space. However, no theoretical study of the generalization property is provided. 
In addition, several recent papers~\citep{haochen2021provable, wen2021toward, ji2021power, lee2021predicting, wen2022mechanism, von2021self, saunshi2022understanding} theoretically study contrastive learning without focusing on the role of projectors.

\paragraph{Implicit bias and interpolating models.} 
 It is recently discovered 
that in over-parametrized models (i.e., interpolating models), gradient descent (GD) type algorithms have implicit regularization effects on the model parameters. Relevant to our paper are the results of \citet{Soudry2018TheIB} and \citet{gunasekar2018characterizing}, where it is proved that for linearly separable data, GD iterates converge in direction to the \textit{max-margin classifier}. It is also important to characterize the generalization error of the solutions with implicit bias. A line of relevant papers include \citet{belkin2019reconciling, bartlett2020benign, hastie2022surprises, liang2020just, bartlett2020benign, hastie2022surprises, montanari2019generalization, montanari2022interpolation, deng2022model, mei2022generalization, liang2022precise, montanari2021tractability}. Closely related to our Section~\ref{sec:result} is the recent paper by \citet{deng2022model}, but their main purpose is explaining the double-descent phenomenon rather than studying the effects for expansion/shrinkage.

\bibliographystyle{apalike}
\bibliography{ref}


\newpage
\appendix
\tableofcontents

\addtocontents{toc}{\protect\setcounter{tocdepth}{3}}

\section{Experiments: details and extensions}\label{sec:append_exp}

\paragraph{Reproducibility.} Our code and data are included in the supplemental materials.

\subsection{Experiment setup and details}\label{sec:append_setup}

\paragraph{Fixed encoder network.} In Figure~\ref{fig:image_pretrained_plot}, we freeze the encoder network (ResNet-18) trained and saved in \url{https://github.com/sthalles/SimCLR} to focus on the behavior of projector under SimCLR loss. The pretrained architecture is trained with the default temperature $\tau=0.07$ and the following composite augmentation
\begin{lstlisting}[language=Python, caption=Image augmentation]
class TransformsSimCLR:
    """
    A stochastic data augmentation module that transforms any given data example randomly
    resulting in two correlated views of the same example,
    denoted x_i and x_j, which we consider as a positive pair.
    """

    def __init__(self, size, aug_str=1):
        color_jitter = torchvision.transforms.ColorJitter(
            0.8 * aug_str, 0.8 * aug_str, 0.8 * aug_str, 0.2 * aug_str
        )
        self.train_transform = torchvision.transforms.Compose(
            [
                torchvision.transforms.RandomResizedCrop(size=size),
                torchvision.transforms.RandomHorizontalFlip(), 
                torchvision.transforms.RandomApply([color_jitter], p=0.8),
                torchvision.transforms.RandomGrayscale(p=0.2),
                torchvision.transforms.ToTensor(),
            ]
        )

        self.test_transform = torchvision.transforms.Compose(
            [
                torchvision.transforms.Resize(size=size),
                torchvision.transforms.ToTensor(),
            ]
        )

    def __call__(self, x):
        return self.train_transform(x), self.train_transform(x)
\end{lstlisting}

\paragraph{Training details.}
For simplicity, we extract and then center pretrained features of 10-class \texttt{STL}-10 images (which conforms to the zero-mean assumption in 2-GMM). We train the projector $\bW \in \RR^{512 \times 512}$ for $50$ epochs with the batch size of $64$.

In all experiments reported in the paper, we use one-layer linear projector $\bW \in \RR^{p \times p}$ without the bias term, and this matrix is initialized by a random orthogonal matrix (orthogonal initialization avoids potential optimization artifacts \citep{xiao2018dynamical, arora2018convergence}. We train the linear projector using the standard SimCLR loss function \url{https://github.com/sthalles/SimCLR}. For downstream accuracy, we use ``linear\underline{\hspace{0.5em}}model.LogisticRegression'' from \texttt{sklearn} with very small $\ell_2$ regularization (choosing $C=1000$), with the aim to approximate the max-margin classifier when data are linearly separable \citep{Rosset2003}.

To assess the behavior of $\bW$ and the transition between expansion and shrinkage regimes, $10$ values of $\log(\tau)$ are chosen from equi-spaced grids in $[\log(0.01), \log(10)]$ and two values of the augmentation strength ($0.1$ and $1.0$) are chosen to represent small and moderate augmentation respectively.

\paragraph{Results for other pairs and deeper encoder.} Since the projector $\bW$ is trained with the 10-class \texttt{STL}-10 dataset, then the plots for singular values of $\bW$ and downstream accuracy are the same as is shown in Figure~\ref{fig:image_pretrained_plot}. In Figure~\ref{fig:image_pretrained_pair}, the cumulative sums of alignment scores are shown for different pairs which show similar patterns: as $\tau$ increases, the expansion effect is gradually gained.
\begin{figure}[t]
    \centering
    \includegraphics[width=0.95\textwidth]{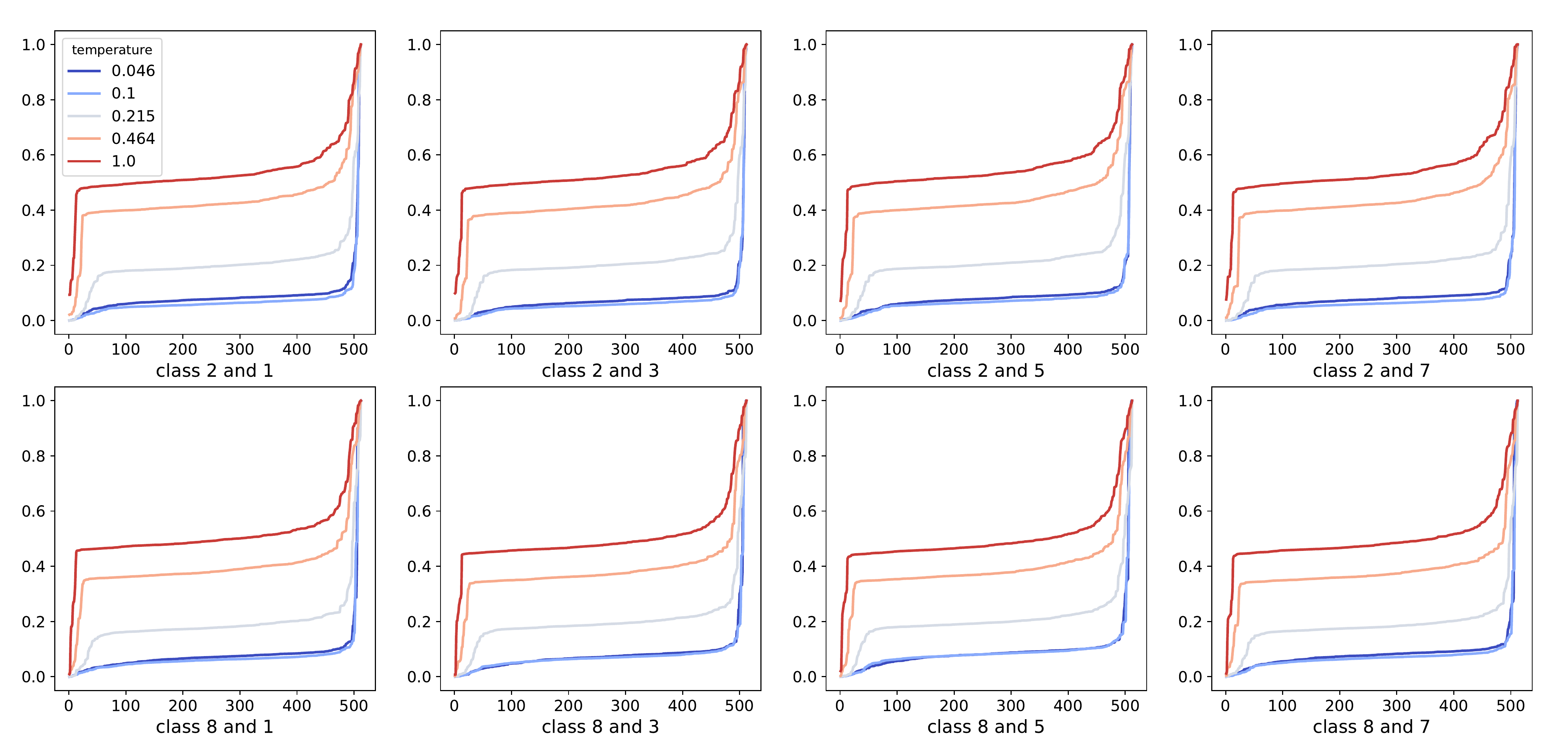}
    \caption{Cumulative sums of alignment scores $\kappa_{j,c_1,c_2} = \langle \vv_j, \bmu_{c_1,c_2} \rangle^2 / \Vert \bmu_{c_1,c_2} \Vert^2$ for different pairs with the pretrained encoder and a one-layer linear projector $\bW \in \RR^{512 \times 512}$ under the standard SimCLR loss (ResNet-18 and 10-class \texttt{STL}-10 dataset).}
    \label{fig:image_pretrained_pair}
\end{figure}

In addition, we also experimented with ResNet-50 as the encoder instead of ResNet-18. Here we present results using ResNet-50 in  Figure~\ref{fig:image_pretrained_plot_resnet50}, which is similar to Figure~\ref{fig:image_pretrained_plot}. We also presented results for different pairs in \ref{fig:image_pretrained_pair_resnet50}. 
\begin{figure}[t]
    \centering
    \begin{subfigure}{0.95\textwidth}
        \includegraphics[width=0.95\textwidth]{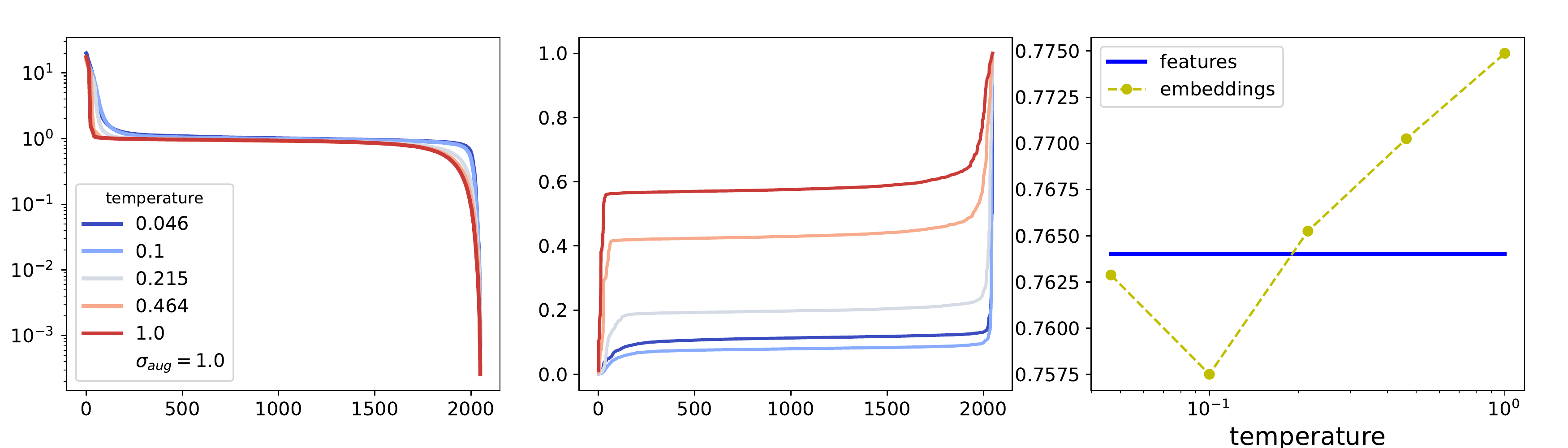} 
        \caption{}
        \label{fig:image_pretrained_plot_resnet50}
    \end{subfigure}
    \begin{subfigure}{0.95\textwidth}
        \includegraphics[width=0.95\textwidth]{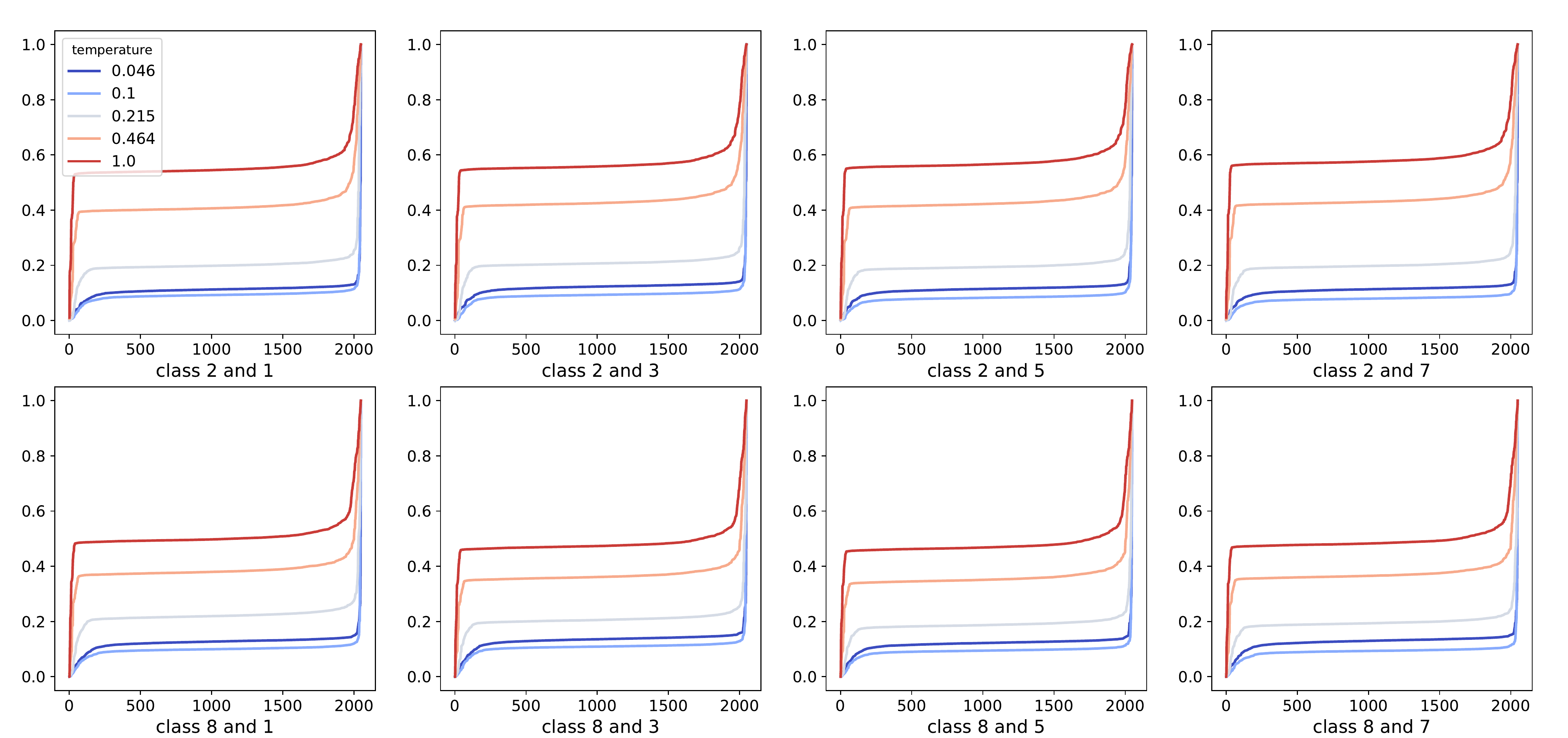}
        \caption{}
        \label{fig:image_pretrained_pair_resnet50}
    \end{subfigure}
    \caption{Results with the pretrained encoder and a one-layer linear projector $\bW \in \RR^{2048 \times 2048}$ under the standard SimCLR loss (ResNet-50 and 10-class \texttt{STL}-10 dataset). {\bf (a)~~Left}: singular values of $\bW$ with varying temperature. {\bf (a)~~Middle}: cumulative sums of alignment scores $\kappa_{j,c_1,c_2} = \langle \vv_j, \bmu_{c_1,c_2} \rangle^2 / \Vert \bmu_{c_1,c_2} \Vert^2$ for the pair $(c_1,c_2)=(7,2)$. {\bf (a)~~Right}: comparison between downstream task accuracy with features and embeddings for 10-class evaluation. {\bf (b)} Cumulative sums of alignment scores $\kappa_{j,c_1,c_2} = \langle \vv_j, \bmu_{c_1,c_2} \rangle^2 / \Vert \bmu_{c_1,c_2} \Vert^2$ for different pairs.}
\end{figure}

We generate $n=2000$ data $(\xx_i)_{i\le n}$ from 2-GMM with $p=100$ and $\bmu = 4\ee_1$ where $(\ee_k)_{k \le p}$ forms the canonical basis. The left plot shows the first two coordinates of these data points, which is equivalent to projecting data onto $\mathrm{span}(\ee_1, \ee_2)$.

We add random perturbation $\mathcal{N}(0,\sigma_\aug^2 \bI_p)$ to form augmented data, and then we train a linear projector $\bW$ and the standard SimCLR loss on $(\xx_i)_{i\le n}$. After training, we calculate embeddings $\zz_i = \bW \xx_i$, project the embeddings onto $\mathrm{span}(\bW\ee_1, \bW\ee_2)$ and visualize these 2D projections. We plot the embeddings under an archetypal expansion regime (middle plot, $\sigma_\aug = 0.1, \tau=0.2$) and an archetypal shrinkage regime (right plot, $\sigma_\aug = 1, \tau=20$).

To aid visualization, we add circles in each of the three plots. Note that the normalized embedding $\zz_i / \norm{\zz_i}$ is used for calculating the cosine similarities in the SimCLR loss. We can interpret the plots using the alignment vs.~uniformity perspective \citep{wang2020understanding}: in the expansion regime. the alignment loss is the dominant term and forces concentrations of normalized embeddings, whereas in the shrinkage regime, the uniformity loss is the dominant term and encourages normalized embeddings to be evenly spread.

\subsection{Additional experiments}\label{sec:append_add_exp}

\paragraph{Beyond fixed encoders: training SimCLR from scratch.} To explore the expansion/shrinkage phenomenon without freezing the encoder component, we train the entire architecture (e.g.,  ResNet-50 encoder and a one-layer linear projector) on \texttt{STL}-10 train dataset with $400$ epochs and the batch size of $256$. We choose $10$ values of $\log(\tau)$ from equi-spaced grids in $[\log(0.01), \log(10)]$ and we choose the augmentation as the default value $1.0$. 
\begin{figure}[ht]
    \centering
    \includegraphics[width=0.9\textwidth]{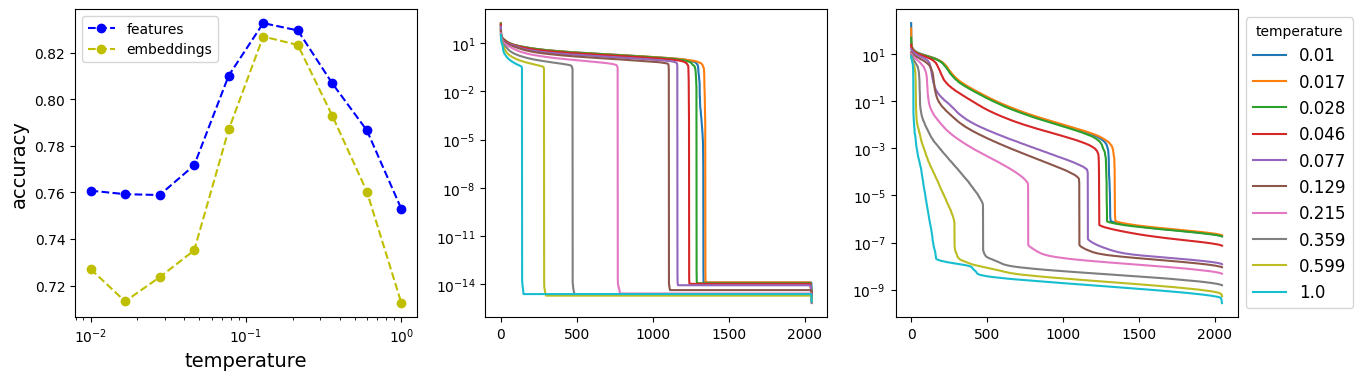}
    \caption{Results for SimCLR models trained from scratch. {\bf Left}: comparison between classification accuracy with features (before projection) and embeddings (after projection). {\bf Middle}: singular values of feature matrix $\bH_0 = [\hh_{1,0},\ldots,\hh_{n,0}]$ with varying temperatures. {\bf Right}: singular values of embedding matrix $\bZ = [\zz_{1}, \ldots, \zz_{n}]$ with varying temperatures.}
    \label{fig:resnet50}
\end{figure}

\begin{figure}[ht]
    \centering
    \includegraphics[width=0.9\textwidth]{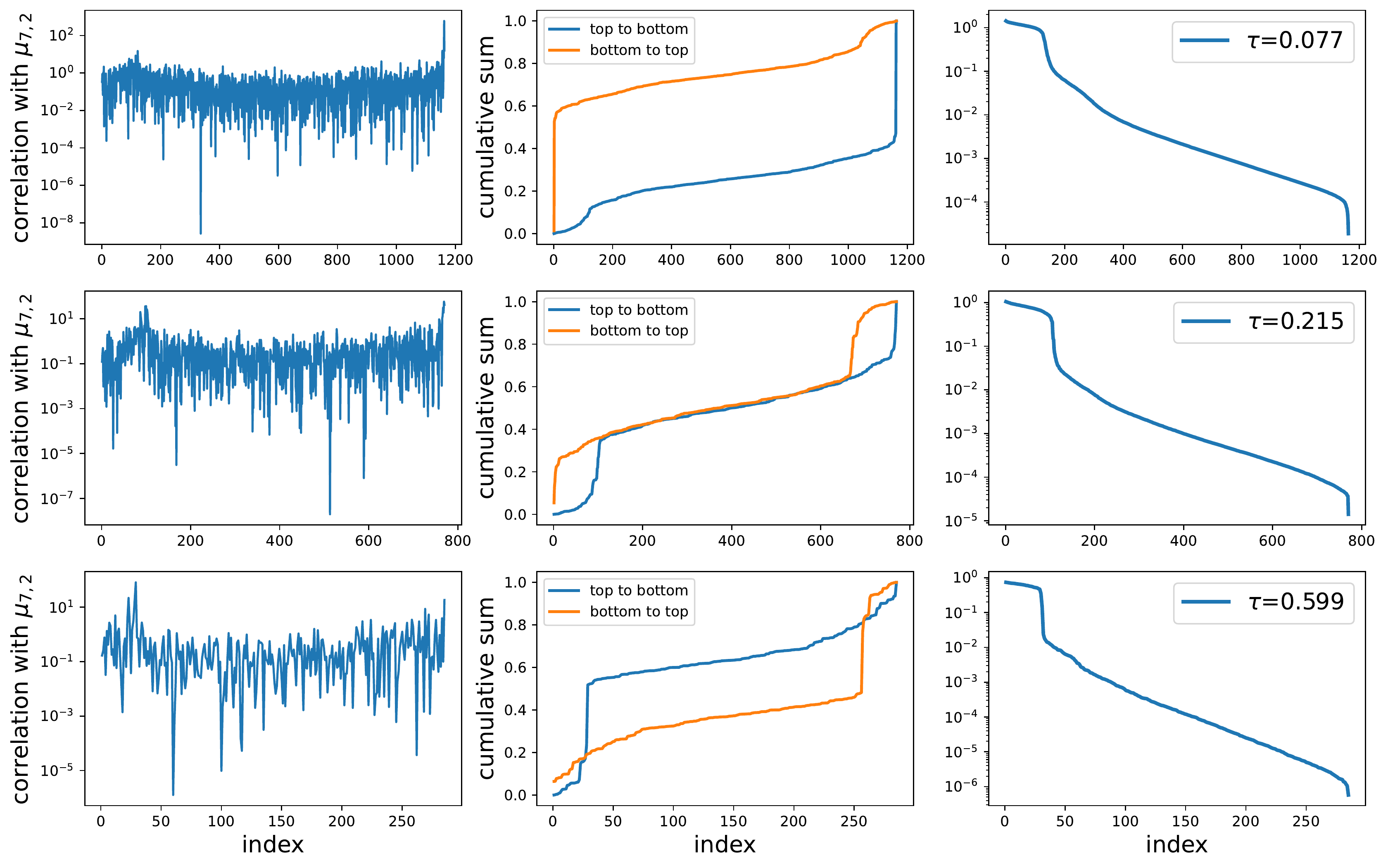}
    \caption{Plots for train-from-scratch experiment with projector $\bW^{(r)} = \bW\bV_r(\bH_0) \in \RR^{p \times r}$ and $\bmu^{(r)} = \bV^\top_r(\bH_0) \bmu \in \RR^r$ which are both projected onto feature matrix $\bH_0$'s top-$r$ right singular subspace with $r=\argmax\{k: s_k(\bH_0) \geq 10^{-3}\}$. {\bf Left}: unsquared alignment scores $\langle \vv^{(r)}_j, \bmu^{(r)}_{c_1,c_2} \rangle / \Vert \bmu^{(r)}_{c_1,c_2} \Vert$ for the pair $(c_1,c_2)=(7,2)$. {\bf Middle}: cumulative sum of alignment scores. {\bf Right}: singular values of $\bW^{(r)}$.}
    \label{fig:resnet_res_scores}
\end{figure}

As shown in Figure~\ref{fig:resnet50}, visibly, the dimensional collapse phenomenon is evident in both the feature space and the projector. Our GMM theory does not apply directly to this scenario since in Section~\ref{sec:gmm_expansion_shrinkage} we assume that features are generated from a full-dimensional mixture model. Still, our expansion/shrinkage analysis provides partial explanations as summarized below.

\begin{enumerate}
    \item When we train the encoder and projector simultaneously, their roles and effects are not distinctly separated. Indeed, there are many ways to express $\gggg_{\bvarphi} \circ \ff_{\btheta}$ as function compositions. Thus, the dimensional  collapse in the feature space can be interpreted as a shrinkage effect induced by the last few layers in the encoder. Understanding how dimensional collapse emerges progressively across layers is an interesting research direction.
    \item Our analysis still provides useful information about downstream accuracy when both the encoder and the projector are trained. For example, when we vary the temperature parameter, the severity of the collapse is correlated with the downstream accuracy; see Figure~\ref{fig:resnet50}.
    \item Our theory matches the empirical results if we restrict the linear transform $\bW$ on the subspace that the feature/embedding vectors span. Figure~\ref{fig:resnet_res_scores} shows that the singular values/vectors of the restricted linear transform. Note that we  recover similar cumulative score plots as in the fixed encoder scenario.
\end{enumerate}

Below we provide more detailed explanations for point 2 and 3. 

First, from Figure~\ref{fig:resnet50} (left), we can see that the downstream accuracy using features is higher than that using embeddings, which validates the practice of using only the features before the projector for classification.
 When $\tau \leq 0.129$, the difference between two curves is decreasing and both achieve the highest value at $\tau=0.129$. However, when $\tau$ further increases, which disagrees with our previous findings, both accuracy start to decrease. This phenomenon can be explained by the plots of singular values. As we can see from Figure~\ref{fig:resnet50} (middle), the features already have dimensional collapse with the one-layer linear projector even when $\tau$ is small, but the collapse is relatively moderate when $\tau \leq 0.129$, which refers to the bunch of curves starting to drop after the index of $1000$. When $\tau > 0.129$, the collapse becomes much more severe and the effective rank decreases fast below $200$ when $\tau$ goes to $1.0$. The singular values of embeddings change accordingly.

The trend in the downstream task accuracy together with the changes in singular values of $\bW$ convey the message that
\begin{itemize}
    \item when $\tau$ is moderate, the increase in $\tau$, which enhances the expansion of signal (will be shown in the following figures), will improve the downstream task accuracy with embeddings, making it as good as the accuracy with features even when features are undergoing the dimensional collapse;
    \item when $\tau$ is large, it poses negative effects in downstream task accuracy in that the features are already low-rank as is shown in \citet{Tian2022UnderstandingDC}, which may lead to the information loss in the data and may further do harm to the training of projector. As a result, the accuracy with either embeddings or features decreases. Also, the benefits from the expansion of signal are surpassed and accuracy with embeddings can be worse than that with features. In contrast, with pretrained model and full-rank features, the accuracy with embeddings can be better than that with features with the benefits from expansion as is shown in Figure~\ref{fig:image_pretrained_plot}.
\end{itemize}

In Figure~\ref{fig:resnet_res_scores}, we first write the SVD of the feature matrix $\bH_0$ as $\bH_0 = \bU(\bH_0) \bD(\bH_0) \bV^\top(\bH_0)$, where $\bD(\bH_0)$ has singular values of $\bH_0$ as diagonal elements: $\{s_k(\bH_0):~1 \leq k \leq p\}$. We consider projecting both $\bW$ and $\bmu$ onto the features' top-$r$ right singular subspace, that is we define $\bW^{(r)} = \bW\bV_r(\bH_0)$ and $\bmu^{(r)} = \bV_r^\top(\bH_0) \bmu$ where $\bV_r(\bH_0) \in \RR^{p \times r}$ is the submatrix consisting of the first $r$ columns of $\bV(\bH_0)$. We choose $r$ by $r=\argmax\{k: s_k(\bH_0) \geq 10^{-3}\}$. Then, we can instead calculate the SVD of $\bW^{(r)} = \bU^{(r)} \bD^{(r)} \bV^{(r) \top}$. Denote $\vv^{(r)}_j$ as the $j$th right singular vector of $\bW^{(r)}$. 

In the left column, we plot the unsquared alignment score  $\langle \vv^{(r)}_j, \bmu^{(r)}_{c_1,c_2} \rangle / \Vert \bmu^{(r)}_{c_1,c_2} \Vert$, where $(c_1,c_2)=(7,2)$, and the cumulative scores are plotted in the middle column. With the truncated scores, we can see that when $\tau=0.077$, the bottom singular vectors align much better with the $\bmu^{(r)}_{c_1,c_2}$ than top ones (shrinkage regime). As $\tau$ increases, the expansion and shrinkage effects are comparable to each other with $\tau=0.219$, but when $\tau$ further increases to $0.599$, top singular vectors align better with $\bmu^{(r)}_{c_1,c_2}$ and the expansion effect is dominating. However, expansion effect directly enhances downstream accuracy with full-rank features as is shown in the pretrained experiments and the benefits can be hidden by the dimensional collapse of features as is shown in Figure~\ref{fig:resnet50}.

\paragraph{Convergence of the encoder}
Recall the decomposition
\[
\calL(\btheta^{(t)},\bW^{(t)}) = \calL(\btheta^{(t)},\tilde \bW^{(t)}) + \left( \calL(\btheta^{(t)},\bW^{(t)}) - \calL(\btheta^{(t)},\tilde \bW^{(t)}) \right),
\]
where $( \btheta^{(t)}, \bW^{(t)})$ are the parameters at epoch $t$, and $\tilde \bW^{(t)} \coloneqq \argmin_{\bW}  \mathcal{L}(\btheta^{(t)}, \bW)$, i.e. the optimal projector with frozen encoder parameters $\btheta^{(t)}$ at epoch $t$. We observe that when $t \geq T_0$ for $T_0 \approx 50$, we observe $\norm{ \tilde \bW^{(t)}  - \bW^{(t)} } \ll \norm{ \bW^{(t)} }$.
\begin{figure*}[t]
    \centering
    \begin{subfigure}{0.49\textwidth}
        \includegraphics[width=0.7\textwidth]{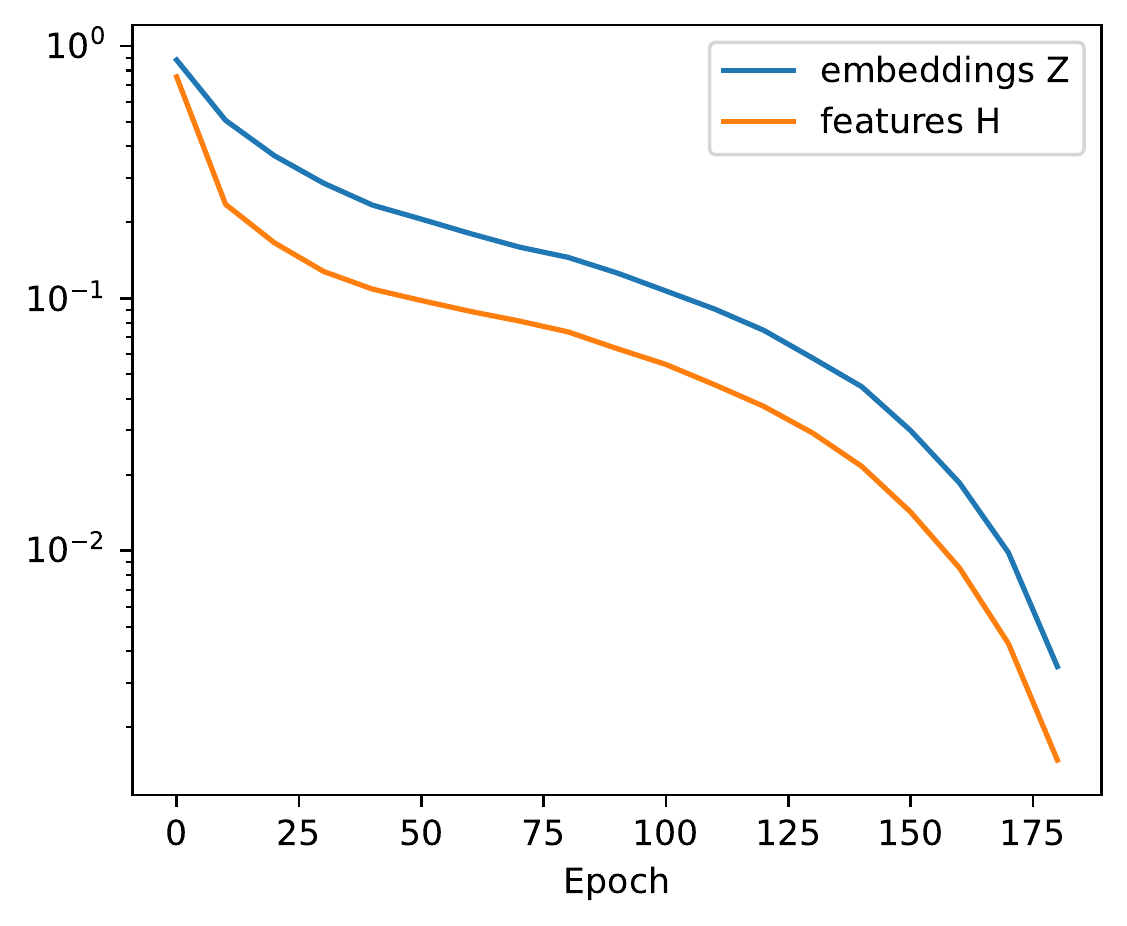}
    \end{subfigure}
    \hspace{-2em}
    \begin{subfigure}{0.49\textwidth}
        \includegraphics[width=0.7\textwidth]{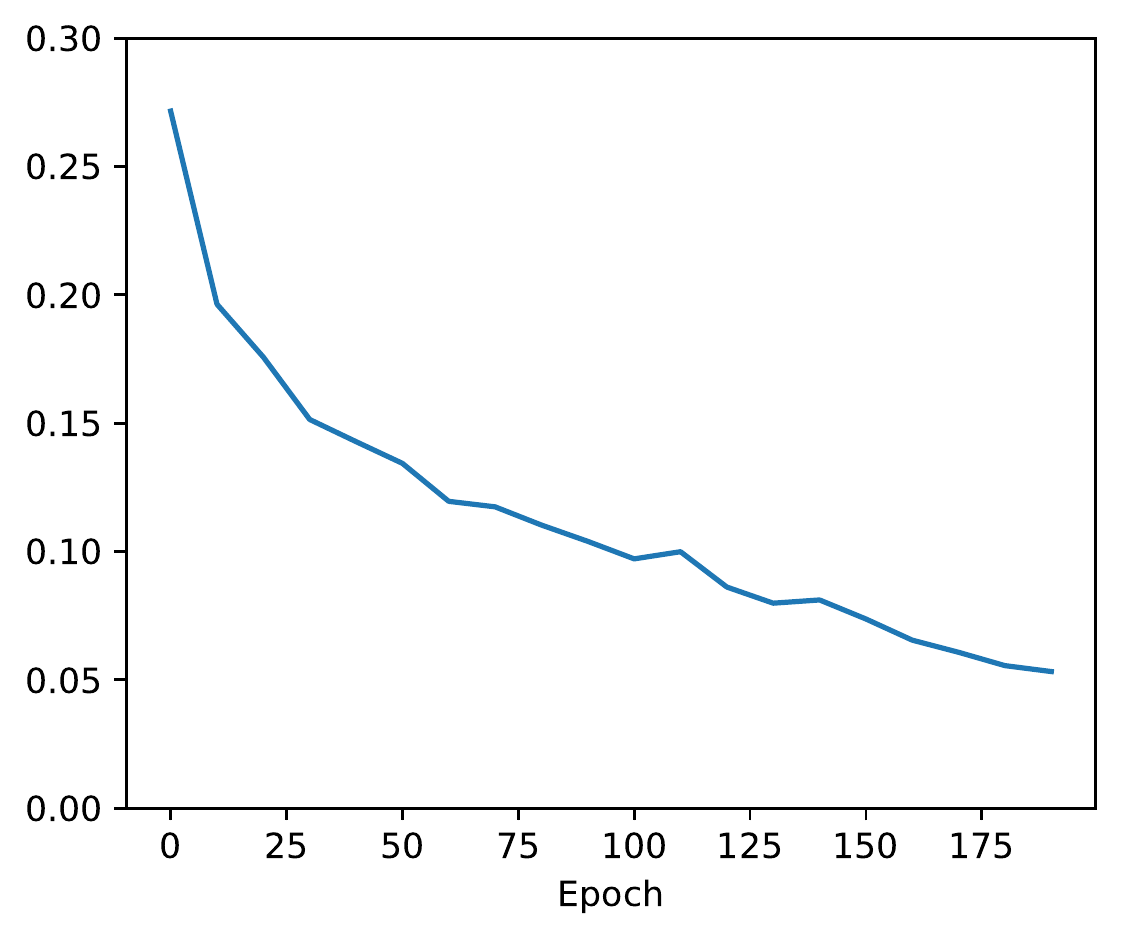}
    \end{subfigure}
    \caption{\textbf{Left} the relative change of features and embeddings: $\Vert \btheta^{(t+1)} - \btheta^{(t)} \Vert/\Vert \btheta^{(t)} \Vert$ and $\Vert \bW^{(t+1)} - \bW^{(t)} \Vert/\Vert \bW^{(t)} \Vert$; \textbf{Right} the standardized distance between $\bW^{(t)}$ and the minimizer $\tilde{\bW}^{(t)}$: $\norm{ \tilde \bW^{(t)}  - \bW^{(t)} } /\norm{ \bW^{(t)} }$.}
    \label{fig:train_freeze}
\end{figure*}

\paragraph{Feature-level vs.~image-level augmentation}
Image-level augmentation results in feature-level perturbation in that with the composite augmentation and the fixed encoder $\ff_{\btheta_n^*}$, for each image $\xx_i$, we have the augmented image $\xx_i^+$ and the mapped features $\hh_i = \ff_{\btheta_n^*}(\xx_i)$, $\hh_i^+ = \ff_{\btheta_n^*}(\xx_i^+)$. Then, the image-level augmentation is associated with the feature-level perturbation $\ee_i = \hh_i^+ - \hh_i$. This perturbation is correlated with $\hh_i$ and can be complicated. To make a connection with our 2-GMM theory, we also consider the homogeneous feature-level augmentation and its effect on the expansion/shrinkage phenomenon.

We freeze the encoder network (ResNet-50) as before, then the features are also fixed in this case. For each $\hh_{0,i}$, we add Gaussian perturbation $\bepsilon_i \sim \calN(0,\sigma_{\aug}^2 \bI_p)$ independently to obtain $\hh_i^+ = \hh_{0,i} + \bepsilon_i$. Then, with the same training process, we train the projector $\bW \in \RR^{2048 \times 2048}$ without the bias term under the standard SimCLR loss for 50 epochs with the batch size of 64. We choose $10$ values of $\log(\tau)$ from equi-spaced grids in $[\log(0.01), \log(10)]$ and we choose the augmentation as the default value $1.0$. \begin{figure}[t]
    \centering
    \includegraphics[width=0.95\textwidth]{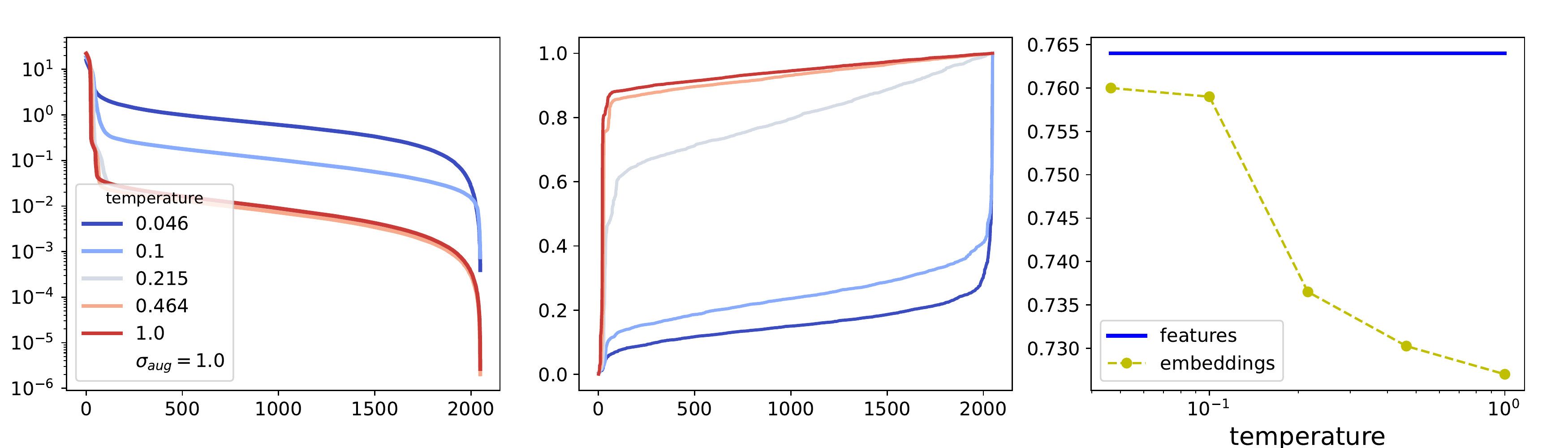} 
    \caption{Results with the pretrained encoder and a one-layer linear projector $\bW \in \RR^{2048 \times 2048}$ under the standard SimCLR loss (ResNet-50, 10-class \texttt{STL}-10 dataset and feature-level perturbation with $\sigma_{\aug}=1.0$). {\bf Left}: singular values of $\bW$ with varying temperature. {\bf Middle}: cumulative sums of alignment scores $\kappa_{j,c_1,c_2} = \langle \vv_j, \bmu_{c_1,c_2} \rangle^2 / \Vert \bmu_{c_1,c_2} \Vert^2$ for the pair $(c_1,c_2)=(7,2)$. {\bf Right}: comparison between downstream task accuracy with features and embeddings for 10-class evaluation.}
    \label{fig:feature_pretrained_plot}
\end{figure}

From Figure~\ref{fig:feature_pretrained_plot} ( middle), different from the results for image-level augmentation, we can see the dominating expansion effect when $\tau > 0.215$ where the bottom singular vectors of $\bW$ is merely uncorrelated with $\bmu_{c_1,c_2}$ while the top singular values contribute over $90\%$ of the correlation.

\section{Proofs for Section~\ref{sec:gmm_expansion_shrinkage}}

\subsection{Proof of Proposition~\ref{prop:loss}}\label{sec:loss_calc}
Recall that the population loss $\calL(\bW) = \mathbb{E} [\calL_{n}(\bW)]$ obeys
\begin{align}
    \calL(\bW)\coloneqq\calL_{\aalign}(\bW)+\calL_{\unif}(\bW),
\end{align}
where we denote 
\begin{align*}
\calL_{\aalign}(\bW) & \coloneqq \frac{1}{\tau}\mathbb{E}_{\hh,\hh^{+}}\left[\frac{1}{2}\frac{\|\bW \hh-\bW \hh^{+}\|_{2}^{2}}{\left(\mathbb{E}[\|\bW \hh\|_{2}^{2}] \cdot \mathbb{E}[\|\bW \hh^{+}\|_{2}^{2}]\right)^{1/2}}\right];\\
\calL_{\unif}(\bW) & \coloneqq\log\left(\mathbb{E}_{\hh^{-},\hh}\left[\exp\left(-\frac{1}{2\tau}\frac{\|\bW \hh-\bW \hh^{-}\|_{2}^{2}}{\left(\mathbb{E}[\|\bW \hh\|_{2}^{2}] \cdot  \mathbb{E}[\|\bW \hh^{-}\|_{2}^{2}]\right)^{1/2}}\right)\right]\right).
\end{align*}

In what follows, we compute $\calL_{\aalign}(\bW)$ and $\calL_{\unif}(\bW)$.

\paragraph{Computing $\calL_{\aalign}(\bW)$.} 
For the first term, we have
\[
\mathbb{E}[\|\bW \hh\|_{2}^{2}]=\mathbb{E}[\|\bW \hh^{+}\|_{2}^{2}]=(1+\sigma^2_{\aug})\|\bW\|_{\mathrm{F}}^{2}+\|\bW \bmu\|_{2}^{2}=\alpha,
\]
where the last relation is the definition of $\alpha$.
Therefore, $\calL_{\aalign}$ can be simply written as 
\[
\calL_{\aalign}(\bW)=\frac{1}{2\tau\alpha}\mathbb{E}_{\hh,\hh^{+}}\left[\|\bW \hh-\bW \hh^{+}\|_{2}^{2}\right].
\]
In addition, since $\hh=\hh_{0}+\bepsilon$, and $\hh^{+}=\hh_{0}+\bepsilon^{+}$ with $\bepsilon, \bepsilon^{+} \overset{\text{i.i.d.}}{\sim} \mathcal{N}(0, \sigma_{\aug} \bI)$,
we obtain 
\[
\hh-\hh^{+}=\bepsilon-\bepsilon^{+}\sim\mathcal{N}(0,2\sigma_{\aug}^{2}\bI),
\]
which yields 
\[
\mathbb{E}_{\hh,\hh^{+}}\left[\|\bW \hh-\bW \hh^{+}\|_{2}^{2}\right]=2\sigma_{\aug}^{2}\|\bW\|_{\mathrm{F}}^{2}.
\]
This further leads to the simplified expression for $\calL_{\aalign}(\bW)$: 
\begin{align}
    \calL_{\aalign}(\bW) = \frac{\sigma^2_{\aug}\Vert \bW \Vert_{\mathrm{F}}^2}{\tau\alpha}.
\end{align}

\paragraph{Computing $\calL_{\unif}(\bW)$.} 
For the second term $\calL_{\unif}(\bW)$, using
the definition of $\alpha$, we see that 
\begin{align}\label{eq:uniform-decomposition-with-log}
\calL_{\unif}(\bW)=\log\left(\mathbb{E}_{\hh^{-},\hh}\left[\exp\left(-\frac{\|\bW \hh-\bW \hh^{-}\|_{2}^{2}}{2\tau\alpha}\right)\right]\right).
\end{align}
Recall that $\hh$ and $\hh^{-}$ are two independent draws
from the Gaussian mixture model, i.e., 
\begin{align*}
\hh & =y\bmu+\sqrt{1+\sigma^2_{\aug}}\gggg,~~~\hh^{-} =y^{-}\bmu+\sqrt{1+\sigma^2_{\aug}}\gggg^{-},
\end{align*}
where $y$ and $y^{-}$ are independent Rademacher random variables,
$\gggg$ and $\gggg^{-}$ are two independent random vectors from $\calN(0,\bI)$. This further yields
\begin{align*}
\hh-\hh^{-} & \overset{\mathsf{d}}{=}(y-y^{-})\bmu+\sqrt{2(1+\sigma^2_{\aug})}\gggg\\
 & \overset{\mathsf{d}}{=}\frac{1}{2}\delta_{0}+\frac{1}{4}\delta_{2\bmu}+\frac{1}{4}\delta_{-2\bmu}+\sqrt{2(1+\sigma^2_{\aug})}\gggg.
\end{align*}
As a result, we obtain 
\begin{align}\label{eq:uniform-decomposition}
 & \mathbb{E}_{\hh^{-},\hh}\left[\exp\left(-\frac{\|\bW \hh-\bW \hh^{-}\|_{2}^{2}}{2\tau\alpha}\right)\right] \nonumber \\
 & \quad=\mathbb{E}_{\gggg}\left[\frac{1}{2}\exp\left(-\frac{(1+\sigma^2_{\aug})\|\bW \gggg\|_{2}^{2}}{\tau\alpha}\right)+\frac{1}{2}\exp\left(-\frac{\|\sqrt{2}\bW \bmu+\sqrt{1+\sigma^2_{\aug}}\bW \gggg\|_{2}^{2}}{\tau\alpha}\right)\right].
\end{align}
Let $\sum_{j=1}^{p}\sigma_{j}\uu_{j}\vv_{j}^{\top}$ be the singular value decomposition of $\bW$.
Then one has 
\[
\bW \gggg=\sum_{j=1}^{p}\sigma_{j}\vv_{j}^{\top}\gggg\ \uu_{j},\qquad\|\bW \gggg\|_{2}^{2}=\sum_{j=1}^{p}\sigma_{j}^{2}\left(\vv_{j}^{\top}\gggg\right)^{2},
\]
and 
\[
\|\sqrt{2}\bW \bmu+\sqrt{1+\sigma^2_{\aug}}\bW \gggg\|_{2}^{2}=\sum_{j=1}^{p}\sigma_{j}^{2}\left(\vv_{j}^{\top}\left(\sqrt{2}\bmu+\sqrt{1+\sigma^2_{\aug}}\gggg\right)\right)^{2}.
\]
Taking expectations, we arrive at 
\begin{align*}
\mathbb{E}_{\gggg}\left[\exp\left(-\frac{(1+\sigma^2_{\aug})\|\bW \gggg\|_{2}^{2}}{\tau\alpha}\right)\right] & =\mathbb{E}_{\gggg}\left[\exp\left(-\sum_{j=1}^{d}\frac{(1+\sigma^2_{\aug})\sigma_{j}^{2}}{\tau\alpha}\left(\vv_{j}^{\top}\gggg\right)^{2}\right)\right]\\
 & =\prod_{j=1}^{p}\mathbb{E}_{\gggg}\left[\exp\left(-\frac{(1+\sigma^2_{\aug})\sigma_{j}^{2}}{\tau\alpha}\left(\vv_{j}^{\top}\gggg\right)^{2}\right)\right],
\end{align*}
where the last equality uses the independence among $\{\vv_{j}^{\top}\gggg\}_{1\leq j \leq p}$. 
Since $(\vv_{j}^{\top}\gggg)^2$ is a $\chi^2_1$ random variable, we can 
use the moment generating function of $\chi^2_1$
to obtain
\[
\mathbb{E}_{\gggg}\left[\exp\left(-\frac{(1+\sigma^2_{\aug})\sigma_{j}^{2}}{\tau\alpha}\left(\vv_{j}^{\top}\gggg\right)^{2}\right)\right]=\left(1+\frac{2(1+\sigma^2_{\aug})\sigma_{j}^{2}}{\tau\alpha}\right)^{-1/2}.
\]
Combining the previous two relations yields
\begin{equation}\label{eq:uniform-T-1}
\mathbb{E}_{\gggg}\left[\exp\left(-\frac{(1+\sigma^2_{\aug})\|\bW \gggg\|_{2}^{2}}{\tau\alpha}\right)\right]=\prod_{j=1}^{p}\left(1+\frac{2(1+\sigma^2_{\aug})\sigma_{j}^{2}}{\tau\alpha}\right)^{-1/2}.
\end{equation}
Using a similar decomposition, we have 
\begin{align}\label{eq:unif_exp}
&\mathbb{E}_{\gggg}\left[\exp\left(-\frac{\|\sqrt{2}\bW \bmu+\sqrt{1+\sigma^2_{\aug}}\bW \gggg\|_{2}^{2}}{\tau\alpha}\right)\right] \nonumber \\
&\quad  = \prod_{j=1}^{p} \mathbb{E}_{\gggg}\left[\exp\left(-\frac{\sigma_{j}^{2}}{\tau\alpha}\left(\sqrt{2}\vv_{j}^{\top}\bmu+\sqrt{1+\sigma^2_{\aug}}\vv_{j}^{\top}\gggg\right)^{2}\right)\right].
\end{align}
Elementary calculations tell us that for any $t>0$ and any $a \in \mathbb{R}$, one has 
$
\mathbb{E}_{u\sim N(0,1)}\left[\exp\left(-t(a+u)^{2}\right)\right]=\frac{1}{\sqrt{1+2t}}\exp\left(-\frac{a^{2}t}{2t+1}\right).
$
This allows us to simplify each term in~\eqref{eq:unif_exp} as
\begin{align*}
&\mathbb{E}_{u\sim N(0,1)}\left[\exp\left(-\frac{\sigma_{j}^{2}}{\tau\alpha}\left(\sqrt{2}\vv_{j}^{\top}\bmu+\sqrt{1+\sigma^2_{\aug}}u\right)^{2}\right)\right] \\
&\quad  =\frac{1}{\sqrt{1+\frac{2(1+\sigma^2_{\aug})\sigma_{j}^{2}}{\tau\alpha}}}\exp\left(-\frac{2\sigma_{j}^{2}}{2(1+\sigma^2_{\aug})\sigma_{j}^{2}+\tau\alpha}(\vv_{j}^{\top}\bmu)^{2}\right).
\end{align*}
The previous two displays taken together lead to  
\begin{align}\label{eq:uniform-T-2}
&\mathbb{E}_{\gggg}\left[\exp\left(-\frac{\|\sqrt{2}\bW \bmu+\sqrt{1+\sigma^2_{\aug}}\bW \gggg\|_{2}^{2}}{\tau\alpha}\right)\right] \nonumber \\
& =  \prod_{j=1}^{p}\left(1+\frac{2(1+\sigma^2_{\aug})\sigma_{j}^{2}}{\tau\alpha}\right)^{-1/2}\exp\left(-\sum_{j=1}^{p}\frac{2\sigma_{j}^{2}}{2(1+\sigma^2_{\aug})\sigma_{j}^{2}+\tau\alpha}(\vv_{j}^{\top}\bmu)^{2}\right).
\end{align}
Substitute~\eqref{eq:uniform-T-1} and~\eqref{eq:uniform-T-2} into the identities~\eqref{eq:uniform-decomposition} and~\eqref{eq:uniform-decomposition-with-log} to see that 
\begin{align}
&\mathbb{E}_{\hh^{-},\hh}\left[\exp\left(-\frac{\|\bW \hh-\bW \hh^{-}\|_{2}^{2}}{2\tau\alpha}\right)\right]\nonumber\\
&=\frac{1}{2}\prod_{j=1}^{p}\left(1+\frac{2(1+\sigma^2_{\aug})\sigma_{j}^{2}}{\tau\alpha}\right)^{-1/2}\left(1+\exp\left(-\sum_{j=1}^{p}\frac{2\sigma_{j}^{2}}{2(1+\sigma^2_{\aug})\sigma_{j}^{2}+\tau\alpha}(\vv_{j}^{\top}\bmu)^{2}\right)\right),
\end{align}
and  
\begin{align*}
\calL_{\unif}(\bW) &=-\frac{1}{2}\sum_{j=1}^{p} \log\left(1+\frac{2(1+\sigma^2_{\aug})\sigma_{j}^{2}}{\tau\alpha}\right) \\
&\quad +\log\left(1+\exp\left(-\sum_{j}\frac{2\sigma_{j}^{2}}{2(1+\sigma^2_{\aug})\sigma_{j}^{2}+\tau\alpha}(\vv_{j}^{\top}\bmu)^{2}\right)\right)-\log2.
\end{align*}
These complete the proof of Proposition~\ref{prop:loss}.

\subsection{Justification of approximate loss}
Here we present justification underlying the approximation~\eqref{eq:approx-loss}.
In the regime where $\tau\alpha=\|\bW\|_{\mathrm{F}}^{2}+\|\bW \bmu\|_{2}^{2}\gg\sigma_{j}^{2}$
for each $j$, we can use the approximations $\log (1 + x) \approx x$ and $2(1+\sigma^2_{\aug})\sigma_{j}^{2}+\tau\alpha \approx \tau\alpha$ to obtain the approximate loss
\begin{align*}
\tilde{\calL}_{\unif}(\bW) & =-\sum_{j=1}^{p}\frac{(1+\sigma^2_{\aug})\sigma_{j}^{2}}{\tau\alpha}+\log\left(1+\exp\left(-\sum_{j}\frac{2\sigma_{j}^{2}}{\tau\alpha}(\bv_{j}^{\top}\bmu)^{2}\right)\right)-\log2\\
 & =-\frac{(1+\sigma^2_{\aug})\|\bW\|_{\mathrm{F}}^{2}}{\tau\alpha}+\log\left(1+\exp\left(-\frac{2\|\bW \bmu\|_{2}^{2}}{\tau\alpha}\right)\right)-\log2.
\end{align*}
This combined with the alignment term results in
\begin{align}
    \tilde{\calL}(\bW) &= \calL_{\aalign}(\bW) + \tilde{\calL}_{\unif}(\bW)\nonumber\\
    & = -\frac{\|\bW\|_{\mathrm{F}}^{2}}{\tau\alpha}+\log\left(1+\exp\left(-\frac{2\|\bW \bmu\|_{2}^{2}}{\tau\alpha}\right)\right)-\log2.
\end{align}
This loss function turns out to be a simple univariate function. 
To see this, we define $t = \Vert \bW \Vert_{\mathrm{F}}^2/\alpha$. As a result, we have $\Vert \bW \bmu\Vert_2^2 = \alpha - (1+\sigma^2_{\aug})\alpha t$, and hence
\begin{align}
    \tilde{\calL}(\bW) &= -\frac{t}{\tau} + \log\left(1+\exp\left(-\frac{2 \alpha \left[1-( 1+\sigma^2_{\aug})t\right]}{\tau\alpha}\right)\right)-\log2\nonumber\\
    & = -\frac{t}{\tau} + \log\left(1+\exp\left(\frac{2(1+\sigma^2_{\aug})t}{\tau} - \frac{2}{\tau}\right)\right) - \log 2.
\end{align}

\subsection{Proof of Theorem~\ref{thm:phase}}\label{sec:proof-phase-transition}

Recall the approximate loss function
\begin{align}
    \tilde{\calL}(\bW) = \ell(t) = -\frac{t}{\tau} + \log\left(\frac{1}{2} + \frac{1}{2} \exp\left(-\frac{2}{\tau} + \frac{2(1+\sigma_{\aug}^2)}{\tau}t\right)\right),
\end{align}
where $t = \|\bW\|_{\mathrm{F}}^2/\alpha \in [1/(1+\sigma_{\aug}^2+\Vert \bmu \Vert^2), 1/(1+\sigma_{\aug}^2)]$ with $\alpha=(1+\sigma^2_{\aug})\|\bW\|_{\mathrm{F}}^{2}+\|\bW \bmu\|_{2}^{2}$. 

Since $F(t; \sigma_\aug^2, \tau) \coloneqq \frac{d}{dt} \widetilde \calL(\bW)$, we have
\begin{align}
    F(t; \sigma_\aug^2, \tau) = \frac{1}{\tau} \left\{1 + 2\sigma_{\aug}^2 - \frac{2(1+\sigma_{\aug}^2)}{1+\exp\left(-\frac{2}{\tau} + \frac{2(1+\sigma_{\aug}^2)}{\tau}t\right)}\right\},
\end{align}
which is a strictly increasing function in $t$. In addition, $F\left((1+\sigma_{\aug}^2)^{-1}; \sigma_\aug^2, \tau\right) = \sigma_{\aug}^2/\tau$ > 0. To determine the minimizer of $\ell(t)$, it suffices to check the sign of $F\left((1+\sigma_{\aug}^2+\Vert \bmu \Vert)^{-1}; \sigma_\aug^2, \tau\right)$. 

\paragraph{Case 1: expansion regime. } 
If $F\left((1+\sigma_{\aug}^2+\Vert \bmu \Vert)^{-1}; \sigma_\aug^2, \tau\right) > 0$ (i.e., in the {\it expansion regime}), we have $F\left(t; \sigma_\aug^2, \tau\right) > 0$ for any $t \in [1/(1+\sigma_{\aug}^2+\Vert \bmu \Vert^2), 1/(1+\sigma_{\aug}^2)]$. Thus, $\tilde{\calL}(W)$ is strictly increasing in $t$ and is minimized when $t = 1/(1+\sigma_{\aug}^2+\Vert \bmu \Vert^2)$, which implies that
\begin{align}
    \frac{\|\bW\|_{\mathrm{F}}^2}{(1+\sigma^2_{\aug})\|\bW\|_{\mathrm{F}}^{2}+\|\bW \bmu\|_{2}^{2}} = \frac{1}{1+\sigma_{\aug}^2+\Vert \bmu \Vert^2} ~~~\Longrightarrow~~~ \frac{\|\bW \bmu\|_{2}^{2}}{\|\bW\|_{\mathrm{F}}^2} = \Vert \bmu \Vert^2.\label{eq:expansion_cond}
\end{align}
If we write the SVD of $\bW$ as $\bW = \sum_{j=1}^p \sigma_j \uu_j \vv_j^\top $, then \eqref{eq:expansion_cond} can be written as 
\begin{align}
     \sum_{j=1}^p \langle \vv_j, \bmu\rangle^2 = \sum_{j=1}^p \frac{\sigma^2_j}{\sum_{k=1}^p \sigma^2_k} \langle \vv_j, \bmu\rangle^2 \leq \frac{\sigma^2_1}{\sum_{k=1}^p \sigma^2_k} \sum_{j=1}^p \langle \vv_j, \bmu\rangle^2 \leq \sum_{j=1}^p \langle \vv_j, \bmu\rangle^2,
\end{align}
where equality holds if and only if $\sigma_j = 0$ for all $j \geq 2$ and $\sigma_1>0$. Then,
\begin{align}
    \Vert \bmu \Vert^2 = \sum_{j=1}^p \langle \vv_j, \bmu\rangle^2 = \langle \vv_1, \bmu\rangle^2.
\end{align}

\paragraph{Case 2: shrinkage regime. } 
If $F\left((1+\sigma_{\aug}^2+\Vert \bmu \Vert)^{-1}; \sigma_\aug^2, \tau\right) < 0$ ({\it shrinkage regime}), then $\tilde{\calL}(W)$ is minimized at some $t^*$ in the interior satisfying
\begin{align}
    F\left(t^*; \sigma_\aug^2, \tau\right) = \frac{1}{\tau} \left\{1 + 2\sigma_{\aug}^2 - \frac{2(1+\sigma_{\aug}^2)}{1+\exp\left(-\frac{2}{\tau} + \frac{2(1+\sigma_{\aug}^2)}{\tau}t^*\right)}\right\} = 0.
\end{align}
Solving this provides us with 
\begin{align}
    t^* = \frac{1}{1+\sigma_{\aug}^2}\left\{1 - \frac{\tau}{2}\log\left(1+2\sigma_{\aug}^2\right)\right\}.
\end{align}
When $\sigma_{\aug} \rightarrow 0$, we have $t^* \rightarrow 1$, which implies that 
\begin{align}
    \frac{\|\bW\|_{\mathrm{F}}^2}{\|\bW\|_{\mathrm{F}}^{2}+\|\bW \bmu\|_{2}^{2}} \rightarrow 1~~~\Longrightarrow~~~ \frac{\|\bW \bmu\|_{2}^{2}}{\|\bW\|_{\mathrm{F}}^2} \rightarrow 0.
\end{align}
Note that for any fixed value of $\norm{\bW}_{\mathrm{F}}$, $\norm{\bW \bmu}^2 \to 0$ is equivalent to
\begin{equation}
    \sum_{j=1}^p \sigma^2_j \langle \vv_j, \bmu\rangle^2 \to 0
\end{equation}
from which we have $\max_{j \leq p} |\sigma_j \langle \vv_j, \bmu\rangle| \to 0$.

\subsection{Proof of Proposition~\ref{prop:approximation}}\label{sec:append_appro}

\subsubsection{A sandwich formula}
Define $\calL_\ell(\bW) \coloneqq \tilde \calL(\bW) $, and 
$\calL_u(\bW) \coloneqq \tilde \calL(\bW) + \Delta(\bW)$, where we denote $r = 1+\sigma_\aug^2$ and 
\begin{align*}
\Delta(\bW) \coloneqq \sum_{j \le p} \frac{r^2\sigma_j^4}{\tau^2 \alpha^2} + \exp \left( - \frac{\norm{\bW \bmu}^2}{\tau \alpha} \right) \cdot \left[ \exp\Big( \sum_{j \le p} \frac{4r\sigma_j^4 \langle \bmu, \vv_j \rangle^2}{\tau^2 \alpha^2} \Big) - 1 \right]. 
\end{align*}
We have the following sandwich-type result. 

\begin{lemma}\label{lem:Lupbnd} We have
\begin{equation*}
\calL_\ell(\bW) \le \calL(\bW) \le \calL_u(\bW), \qquad \text{for all} ~ \bW \in \RR^{p \times p}.
\end{equation*}
\end{lemma}
\begin{proof}
The lower holds by the construction of $\tilde \calL(\bW)$. Hence we focus on the upper bound.

Since $\log (1+x) - x \ge -x^2/2$, we compare the difference between the first term in $\calL_\unif$ and that in $\tilde \calL_\unif$ to obtain
\begin{align*}
-\frac{1}{2} \sum_{j=1}^p\log \left( 1 + \frac{2r \sigma_j^2}{\tau \alpha} \right) + \sum_{j=1}^p \frac{r \sigma_j^2}{\tau \alpha} \le \sum_{j=1}^p\frac{r^2 \sigma_j^4}{\tau^2 \alpha^2} \; .
\end{align*}
Write $\beta_j = \langle \bmu, \vv_j \rangle^2$ for simplicity. We compare the difference between the second terms.
\begin{align*}
&~~\log \left( 1 + \exp \Big( - \sum_{j=1}^p \frac{2 \sigma_j^2 \beta_j}{2r \sigma_j^2 + \tau \alpha} \Big) \right) - \log \left( 1 + \exp \Big( - \sum_{j=1}^p \frac{2 \sigma_j^2 \beta_j}{\tau \alpha} \Big) \right)  \\
&= \log \left[1 + \frac{\exp\big( - \sum_{j\le p} 2\sigma_j^2 \beta_j/(2r \sigma_j^2 + \tau \alpha) \big) - \exp\big( - \sum_{j\le p} 2\sigma_j^2 \beta_j/(\tau \alpha) \big)}{1+ \exp\big( - \sum_{j\le p} 2\sigma_j^2 \beta_j/(\tau \alpha) \big)} \right] \\
&\stackrel{(i)}{\le} \exp\Big( - \sum_{j\le p} \frac{2\sigma_j^2 \beta_j}{2r \sigma_j^2 + \tau \alpha} \Big) - \exp\Big( - \sum_{j\le p} \frac{2\sigma_j^2 \beta_j}{\tau \alpha} \Big) \\
&= \exp \Big( - \frac{2\norm{\bW \bmu}^2}{\tau \alpha} \Big) \cdot \left[\exp\Big(- \sum_{j\le p} \frac{2\sigma_j^2 \beta_j}{2r \sigma_j^2 + \tau \alpha}  + \sum_{j\le p} \frac{2\sigma_j^2 \beta_j}{\tau \alpha}  \Big) - 1 \right] \\
& \stackrel{(ii)}{\le}  \exp \Big( - \frac{2\norm{\bW \bmu}^2}{\tau \alpha} \Big) \cdot \left[ \exp \Big( \sum_{j \le p} \frac{4r\sigma_j^4 \beta_j}{\tau^2 \alpha^2} \Big) - 1 \right]
\end{align*} 
where in (i) we used $\log (1+x) \le x$ and $\exp\big( - \sum_{j\le p} 2\sigma_j^2 \beta_j/(\tau \alpha) \big) \geq 0$, and  in (ii) we used brute force computation. This completes the proof.
\end{proof}

We denote the minimizers of $\calL_\ell$, $\calL$, $\calL_u$ by $\bW_\ell^*$, $\bW^*$, $\bW_u^*$, respectively. We also denote the related \textit{expansion measure}
\begin{equation*}
t_\ell^* = t(\bW_\ell^*), \quad t^* = t(\bW^*), \quad t_u^* = t(\bW_u^*), 
\end{equation*}
where we recall 
\begin{equation*}
t = t(\bW) = \frac{\norm{\bW}_{\mathrm{F}}^2}{r \norm{\bW}_{\mathrm{F}}^2 + \norm{\bW \bmu}^2}.
\end{equation*}

In what remains, we are mainly interested in showing that $t^*$ and $t_\ell^*$ are close. This would suggest that our approximation has small effects on the expansion/shrinkage phenomenon, thus justifying our approximation.

%
%
%
%
%

\subsubsection{Approximation bound in expansion regime} 
Recall the notation $r = 1 + \sigma_\aug^2$, and define $\rho = \norm{\bmu}$. We observe from Lemma~\ref{lem:Lupbnd} that 
\begin{align*}
\calL_u(\bW_\ell^*) - \calL_\ell(\bW_\ell^*) &\ge \calL(\bW_\ell^*) - \calL_\ell(\bW_\ell^*) \ge \calL(\bW^*) - \calL_\ell(\bW_\ell^*) \\
&\ge \calL_\ell(\bW^*) - \calL_\ell(\bW_\ell^*) \\
&\ge \min_{t \in [(r+\rho^2)^{-1}, r^{-1}]} F(t) \cdot \big( t^* - t_\ell^* \big).
\end{align*}
Here the last inequality uses the definition of $F(\cdot)$ as well as the fact that $\calL_\ell$ is convex in $t$.

Note that in the expansion regime, $t_\ell^*$ achieves the smallest value, namely $t_\ell^* = 1/(r+\rho^2)$. 
Therefore we must have $0 \le  t^* - t_\ell^*$. 
In addition, we know that in the expansion regime, $F(t) > 0$ for all $t > t_\ell^*$. As a result, we obtain
\begin{equation}\label{ineq:tdiff}
0 \le  t^* - t_\ell^* \le \left[  \min_{t \in [(r+\rho^2)^{-1}, r^{-1}]} F(t) \right]^{-1} \left(\calL_u(\bW_\ell^*) - \calL_\ell(\bW_\ell^*) \right) \; .
\end{equation}
Since $F(t)$ is increasing in $t$, we have
\begin{align*}
\min_{t \in [(r+\rho^2)^{-1}, r^{-1}]} F(t) = F(1/(r+\rho^2)) = -\frac{1}{\tau} + \frac{2r}{\tau} \cdot \left[ 1 + \exp \Big( \frac{2}{\tau} - \frac{2r}{\tau(r+\rho^2)} \Big) \right]^{-1} \; .
\end{align*}
Denoting the following quantity
\begin{align*}
\tilde{\theta}_{\rho,r,\tau} = \frac{1}{1+e^{2/\tau}}  + \frac{2r e^{2/\tau}}{(1+e^{2/\tau})^2 \rho^2 \tau},
\end{align*}
then as $\rho \rightarrow \infty$,
\begin{align*}
\bigg| \left[ 1 + \exp \Big( \frac{2}{\tau} - \frac{2r}{\tau(r+\rho^2)} \Big) \right]^{-1} - \tilde{\theta}_{\rho,r,\tau} \bigg|  = O(\rho^{-4}).
\end{align*}
Therefore, 
\begin{align}
\min_{t \in [(r+\rho^2)^{-1}, r^{-1}]} F(t) &= -\frac{1}{\tau} + \frac{2r}{(1+e^{2/\tau})\tau} + \frac{4r^2 e^{2/\tau}}{(1+e^{2/\tau})^2 \rho^2 \tau^2} + O(\rho^{-4})\notag \\
&= \frac{1}{\tau} \Big( \frac{2r}{1+e^{2/\tau}} - 1 \Big) + O(\rho^{-2}). \label{eq:Fasymp}
\end{align}
We view $\tau$ and $\sigma_\aug^2$ as fixed parameters and let $\rho \to \infty$. Recall that the expansion regime occurs if $F((r+\rho^2)^{-1}) > 0$ and the shrinkage regime occurs if the reverse inequality is true. Under the asymptotics $\rho \to \infty$, the phase transition boundary simplifies to $2 r > 1 + e^{2/\tau}$ for expansion regime and $2 r < 1 + e^{2/\tau}$ for the shrinkage regime. Thus, the leading term \ref{eq:Fasymp} is always positive in this asymptotics.

Since $\bW = \bW_\ell^*$ if and only if $\langle \bmu, \vv_1 \rangle^2 > 0$, $\langle \bmu, \vv_j \rangle^2 = 0$ for $j > 1$ and $\sigma_2 = \ldots = \sigma_p = 0$, we can show that
\begin{align}
\Delta(\bW_\ell^*) = \calL_u(\bW_\ell^*) - \calL_\ell(\bW_\ell^*) &= \frac{r^2}{\tau^2(r+\rho^2)^2} + \exp \left( - \frac{\rho^2}{\tau(r+\rho^2)} \right) \cdot \left[ \exp\Big( \frac{4r\rho^2}{\tau^2(r+\rho^2)^2}\Big) - 1 \right],
\end{align}
and as $\rho \rightarrow \infty$,
\begin{align}\label{eq:Ldiffasymp}
    \bigg|\Delta(\bW^*_{\ell}) - \frac{4r}{\tau^2 e^{1/\tau}} \frac{1}{\rho^2} \bigg| = O(\rho^{-4}). 
\end{align}
From \eqref{eq:Fasymp} and~\eqref{eq:Ldiffasymp}, we have the control for \eqref{ineq:tdiff}:
\begin{equation*}
0 \le t^* - t_\ell^* \le \left( \frac{2r}{1+e^{2/\tau}} - 1 \right)^{-1} \frac{4r}{\tau e^{1/\tau}} \frac{1}{\rho^2} + O(\rho^{-4})
\end{equation*}
This finishes the proof. 


\subsubsection{Approximation bound in shrinkage regime} 
We use the notation $\texttt{C}^2$ to denote the space of twice continuously differentiable functions. We first present the general result as follows.
\begin{lemma}\label{lem:shrink}
Suppose that $f(\xx), f_\ell(\xx), f_u(\xx)$ are defined for $\xx \in \calX \subset \mathbb{R}^q$ with 
\begin{equation*}
f_\ell(\xx) \le f(\xx) \le f_u(\xx).
\end{equation*}
Assume that $f_\ell(\xx) = h(\varphi(\xx))$ where $\varphi: \calX \to \calY \subset \mathbb{R}$ is a surjection satisfying $\varphi \in \texttt{C}^2(\calX)$, and that $h: \calY \to \mathbb{R}$ satisfies $h \in \texttt{C}^2(\calX)$ and is strongly convex. Suppose $\xx_\ell^* \in \calX$ and $y_\ell^*$ satisfy
\begin{equation*}
\varphi(\xx_\ell^*) = y_\ell^*, \qquad y_\ell^* = \argmin_{y \in \calY} h(y).
\end{equation*}
Then, the following inequality holds.
\begin{equation}\label{ineq:convexbnd}
\bignorm{ y_\ell^* - \varphi(\xx^*) }^2 \le \frac{2 \big( f_u(\xx_\ell^*) - f_\ell(\xx_\ell^*) \big)}{\inf_{y \in \calY}  h''(y)  } \; .
\end{equation}
\end{lemma}
\begin{proof}
By the definitions of $f_u, f, f_\ell$ and $\xx_\ell^*, \xx^*$, we have
\begin{equation}\label{ineq:shrinklem1}
f_u(\xx_\ell^*) \ge f(\xx_\ell^*) \ge f(\xx^*) \ge f_\ell(\xx^*) = h(\varphi(\xx^*)) \, .
\end{equation}
Since $h$ is convex and $\varphi(\xx_\ell^*)$ achieves the minimum, we have
\begin{equation}\label{ineq:shrinklem2}
h(\varphi(\xx^*)) - h(\varphi(\xx_\ell^*)) \ge \frac{1}{2} \inf_{y \in \calY}   h''(y) \cdot  \bignorm{\varphi(\xx^*) - \varphi(\xx_\ell^*)}^2.
\end{equation}
We combine \eqref{ineq:shrinklem1}--\eqref{ineq:shrinklem2} and use $h(\varphi(\xx_\ell^*)) = f_\ell(\xx_\ell^*)$, $\varphi(\xx_\ell^*) = y_\ell^*$, which leads to the desired inequality \eqref{ineq:convexbnd}.
\end{proof}
To use the above lemma, we set $f = \calL, f_\ell = \calL_\ell, f_u = \calL_u$, 
\begin{align*}
& h(t) = -\frac{t}{\tau} + \log\left(1 + \exp \Big( -\frac{2}{\tau} + \frac{2rt}{\tau} \Big) \right) - \log 2, \\
& \varphi(\bW) = t(\bW) = \frac{\norm{\bW}_{\mathrm{F}}^2}{r\norm{\bW}_{\mathrm{F}}^2 + \norm{\bW \bmu}^2},
\end{align*}
where we identify $y$ with $t$. We calculate the lower bound on $h^{\prime\prime}(t)$ as follows.
\begin{align*}
&h^\prime(t) = F(t) = -\frac{1}{\tau} + \frac{2r}{\tau} \left[ 1 + \exp \Big(\frac{2}{\tau} - \frac{2rt}{\tau} \Big) \right]^{-1}, \\
&h^{\prime\prime}(t) = F^\prime(t) = \frac{4r^2}{\tau^2} \exp\Big( \frac{2}{\tau} - \frac{2rt}{\tau} \Big) \cdot \left[ 1 + \exp \Big(\frac{2}{\tau} - \frac{2rt}{\tau} \Big) \right]^{-2}.
\end{align*}
We note that 
\begin{equation*}
\exp\left(\frac{2}{\tau}\right) \ge \exp\Big( \frac{2}{\tau} - \frac{2rt}{\tau} \Big) \ge \exp \left( \frac{2-2r}{\tau} \right).
\end{equation*}
So making use of the fact that $\frac{x}{(1+x)^2}$ is decreasing when $x \geq 1$, we deduce
\begin{align*}
\min_{t \in [(r+\rho)^{-1},1]} F^\prime(t) &= \frac{4r^2}{\tau^2} \min \left\{ \frac{e^{2\sigma_{\aug}^2/\tau r}}{(1+e^{2\sigma_{\aug}^2/\tau r})^2}, \frac{e^{2\rho/\tau(r+\rho)}}{(1+e^{2\rho/\tau(r+\rho)})^2}\right\} \\
& = \frac{4r^2}{\tau^2} \frac{e^{2\sigma_{\aug}^2/\tau r}}{(1+e^{2\sigma_{\aug}^2/\tau r})^2} \geq \frac{r^2}{\tau^2} e^{-2\sigma_{\aug}^2/\tau r}.
\end{align*}
Note that we have the freedom to choose any $\xx$ (identified as $\bW$) as long as $\varphi(\bW) = t_\ell^*$. This leads to the error bound
\begin{align*}
\norm{t^* - t_\ell^*}^2 \le \frac{\tau^2}{r^2} e^{2\sigma_{\aug}^2/\tau}\cdot \Delta(\bW). 
\end{align*}
It remains to control $\min_{\bW: \varphi(\bW) = t_\ell^*} \Delta(\bW)$ which will be our target below.

\paragraph{Characterizing $\min_{\bW: \varphi(\bW) = t_\ell^*} \Delta(\bW)$.}
It suffices to provide a feasible solution $\bW$ such that $\varphi(\bW) = t_\ell^*$. To do so, we construct $\bW$ as follows. 
As usual, let $\sum_{j=1}^{p}\sigma_{j}\uu_{j}\vv_{j}^{\top}$ be the singular value decomposition of $\bW$. Here 
$\sigma_j$, $j\in [p]$, are not necessarily ordered. 
Let $\vv_p = \bmu / \norm{\bmu}$. As a result, we have $\|\bW\|_{\mathrm{F}} = \sum_{j} \sigma_{j}^2$ and $\Vert \bW\bmu \Vert = \sigma_{p}^2 \|\bmu\|^2 = \sigma_{p}^2 \rho^2$. 
Recall that in the shrinkage regime, we have 
\begin{align*}
t^*_\ell = \frac{1}{r} \left(1 - \frac{\tau}{2}\log (2r-1)\right) =  \frac{\norm{\bW}_{\mathrm{F}}^2}{r \norm{\bW}_{\mathrm{F}}^2 + \norm{\bW \bmu}^2} \in \left[\frac{1}{r+\rho^2}, \frac{1}{r}\right].
\end{align*}
Consequently, one can verify that any sequence of singular values $\bsigma$ with $S_{p-1} = \sum_{j=1}^{p-1} \sigma_{j}^2$ that obeys
\begin{align}\label{eq:key-relation}
	\frac{S_{p-1}}{\sigma_{p}^2} = \frac{t^*_\ell (\rho^2 + r) -1 }{1 - t^*_\ell r} > 0
\end{align}
will be a feasible solution. In particular,  
\[
\sigma_{1}^2 = \sigma_{2}^2 = \cdots = \sigma_{p-1}^2 = \frac{S_{p-1}}{p-1}.
\]
Now we are ready to compute the quantity of interest
\[
\Delta(\bW) = \sum_{j \le p} \frac{r^2\sigma_j^4}{\tau^2 \alpha^2} + \exp \left( - \frac{\norm{\bW \bmu}^2}{\tau \alpha} \right) \cdot \left[ \exp\Big( \sum_{j \le p} \frac{4r\sigma_j^4 \langle \bmu, \vv_j \rangle^2}{\tau^2 \alpha^2} \Big) - 1 \right].
\]
\begin{enumerate}
\item First, we have 
\begin{align*}
\sum_{j \le p} \frac{r^2\sigma_j^4}{\tau^2 \alpha^2} &= \sum_{j \le p} \frac{r^2 \sigma_j^4}{\tau^2 (r\norm{\bsigma}^2 + \sigma_p^2 \rho^2)^2}  = \frac{r^2  (t_\ell^*)^2}{\tau^2} \cdot \frac{\sum_{j \le p} \sigma_j^4}{\norm{\bsigma}^4} \\
&= \frac{r^2  (t_\ell^*)^2}{\tau^2} \cdot \frac{(p-1) \frac{S^2_{p-1}}{(p-1)^2} + \sigma_p^4}{\left((p-1) \frac{S_{p-1}}{p-1} + \sigma_p^2\right)^2}\\
&= \frac{r^2  (t_\ell^*)^2}{\tau^2 (p-1)} \cdot \frac{ \frac{S^2_{p-1}}{\sigma^2_{p}} + p-1}{ \left(\frac{S_{p-1}}{\sigma_{p}^2} + 1\right)^2  }
\end{align*}
Recall the identity~\eqref{eq:key-relation} and the fact that $r t^*_{\ell} \leq 1$, we can further write the first term as
\begin{align*}
\sum_{j \le p} \frac{r^2\sigma_j^4}{\tau^2 \alpha^2} &= \frac{r^2 (t^*_{\ell})^2}{\tau^2 (p-1)} \frac{(\rho^2 t^*_{\ell}  - 1 +  r t^*_{\ell})^2 + (p-1)(1 - r t^*_{\ell})^2}{\rho^4 (t^*_{\ell})^2}\\
& \leq \frac{r^2}{\tau^2 \rho^4 (p-1)} \cdot \left(\rho^4 (t^*_{\ell})^2 + p(1 - r t^*_{\ell})^2\right)\\ 
& \leq \frac{r^2}{\tau^2 \rho^4 (p-1)} \cdot \left(\frac{\rho^4}{r^2}  + \frac{\tau^2 p}{4} \log^2(2r-1)\right)\\
& = \frac{1}{\tau^2 (p-1)} + \frac{p r^2}{4 \rho^4 (p-1)}\log^2(2r-1) = \frac{1}{\tau^2 (p-1)} + \frac{\sigma_\aug^4}{\rho^4} + O\left(\frac{\sigma_\aug^6}{\rho^4}\right) \qquad \sigma_\aug \rightarrow 0.
\end{align*}
\item Secondly, since $\exp \left( - \frac{\norm{\bW \bmu}^2}{\tau \alpha} \right) \leq 1$ and $\sigma^2_p \rho^2/\alpha = 1 - r t^*_{\ell}$, then 
\begin{align*}
\exp\left( \sum_{j \le p} \frac{4r\sigma_j^4 \langle \bmu, \vv_j \rangle^2}{\tau^2 \alpha^2} \right) - 1 & = \exp\left(\frac{4r\sigma_p^4 \rho^2}{\tau^2 \alpha^2}\right) - 1 = \exp\left(\frac{4r }{\tau^2 \rho^2}\left(1 - r t^*_{\ell}\right)^2 \right) - 1\\
& = \exp\left( \frac{r}{\rho^2} \log^2(2r-1)\right) - 1 = \frac{4\sigma_\aug^4}{\rho^2} + O\left(\frac{\sigma_\aug^6}{\rho^2}\right) \qquad \sigma_\aug \rightarrow 0.
\end{align*}
\end{enumerate}
Combining pieces above, 
\[
\Delta(\bW) \leq \frac{1}{\tau^2 (p-1)} + O\left(\frac{\sigma_\aug^4}{\|\bmu\|^2}\right).
\]
As a result, we can bound the approximation error in the shrinkage regime by
\[
\norm{t^* - t_\ell^*}^2 \le \frac{1}{(1+\sigma_{\aug}^2)^2 (p-1)} e^{2\sigma_{\aug}^2/\tau} + O\left(\frac{\sigma_\aug^4}{\|\bmu\|^2}\right).
\]

\subsection{Proof of Proposition~\ref{prop:finite}}\label{sec:append_finitesample}

Here we provide some informal analysis on the finite-sample scenario. We realize that a rigorous analysis is challenging and is left to future research.

Recall the contrastive loss
\begin{align}\label{eq:ctr_loss_1}
\calL_n(\bW)\coloneqq-\frac{1}{\tau}\EE_{n}\left\{\EE_{\hh^+,\hh\mid \hh_0}\left[\siml^*\left(\bW \hh,\bW \hh^{+}\right) \mid \hh_0\right]\right\}+\log\left(\EE_{n}\left\{ \EE_{\hh^-, \hh \mid \hh_0}\left[e^{\siml^*(\bW \hh,\bW \hh^{-})/\tau} \big| \hh_0\right]\right\}\right),
\end{align}
which can be decomposed as two terms
\begin{align}
    \calL_{n,1}(\bW) &= -\frac{1}{\tau}\EE_{n}\left[\siml^*\left(\bW \hh,\bW \hh^{+}\right)\right],\\
    \calL_{n,2}(\bW) &= \log\left(\EE_{n} \left\{\EE_{\hh^-, \hh \mid \hh_0}\left[e^{\siml^*(\bW \hh,\bW \hh^{-})/\tau} \big| \hh_0\right]\right\}\right).
\end{align}
Denote $\alpha = (1+\sigma^2_{\aug})\Vert \bW \Vert_{\mathrm{F}}^2 + \Vert \bW \bmu \Vert_2^2$ and $\bA = (2\tau \alpha)^{-1} \bW^\top \bW$, with the definition of $\siml^*$, we have
\begin{align*}
    \calL_{n,1}(\bW) & = \frac{1}{2\tau \alpha} \EE_n \left\{\EE_{\hh^+,\hh\mid \hh_0}\left[\Vert \bW(\hh - \hh^+) \Vert_2^2\big| \hh_0\right]\right\} = \frac{\sigma_\aug^2}{\tau\alpha} \| \bW \|_{\mathrm{F}}^2.
\end{align*}
For $\hh^- \sim \frac{1}{2} \calN(-\bmu,(1+\sigma^2_{\aug})\bI) + \frac{1}{2} \calN(\bmu,(1+\sigma^2_{\aug})\bI)$ and $\hh \mid \hh_0 \sim \calN(\hh_0, \sigma_\aug^2 \bI)$, we have
\begin{align}\label{eq:uni_loss}
    &\EE_{\hh,\hh^{-} \mid \hh_0}\left[\exp\left(\frac{\siml^*(\bW \hh,\bW \hh^{-})}{\tau}\right)\bigg| \hh_0\right] = \EE_{\hh,\hh^{-} \mid \hh_0}\left[\exp\left(-\frac{\Vert \bW(\hh - \hh^-) \Vert_2^2}{2\tau\alpha}\right)\bigg| \hh_0\right]\nonumber\\
    & = \frac{1}{2} \EE_{\zz\sim \calN(\bmu,(1+2\sigma^2_{\aug})\bI)}\left[\exp\left(-\frac{\Vert \bW(\hh_0 - \zz) \Vert_2^2}{2\tau\alpha}\right)\right] + \frac{1}{2} \EE_{\zz\sim \calN(-\bmu,(1+2\sigma^2_{\aug})\bI)}\left[\exp\left(-\frac{\Vert \bW(\hh_0 - \zz) \Vert_2^2}{2\tau\alpha}\right)\right].
\end{align}
With $\zz\sim \calN(\bmu,(1+2\sigma^2_{\aug})\bI)$, we have
\begin{align}\label{eq:int}
    &\EE_{\zz \mid \hh_0}\left[\exp\left(-\frac{\Vert \bW(\hh_0 - \zz) \Vert_2^2}{2\tau\alpha}\right)\bigg| \hh_0\right] \nonumber\\
    = &\left(\frac{1}{2\pi (1+2\sigma^2_{\aug})}\right)^{d/2}\int \exp\left(-\frac{\Vert \bW(\hh_0 - \zz) \Vert_2^2}{2\tau\alpha}\right)  \exp\left(-\frac{1}{2(1+2\sigma^2_{\aug})} \Vert \zz-\bmu \Vert_2^2\right) d\zz\nonumber\\
     = &\left(\frac{1}{2\pi (1+2\sigma^2_{\aug})}\right)^{d/2} \int \exp\left(-\frac{1}{2(1+2\sigma^2_{\aug})} \left[(\zz-\bmu)^\top (\zz-\bmu) + (\hh_0-\zz)^\top \tilde{\bA} (\hh_0-\zz)\right]\right) d\zz,
\end{align}
where $\tilde{\bA} = 2(1+2\sigma^2_{\aug}) \bA$. Rearrange \eqref{eq:int}, we have
\begin{align}
    &(\zz-\bmu)^\top (\zz-\bmu) + (\hh_0-\zz)^\top \tilde{\bA} (\hh_0-\zz)\nonumber\\ 
    = &(\zz-\tilde{\aaaa})^\top \bM^{-1} (\zz-\tilde{\aaaa}) - (\bmu+\tilde{\bA}\hh_0)^\top \bM (\bmu+\tilde{\bA}\hh_0) + \bmu^\top \bmu + \hh_0^\top \tilde{\bA}\hh_0.
\end{align}
where $\bM = (\bI+\tilde{\bA})^{-1}$ and $\tilde{\aaaa} = \bM(\bmu+\tilde{\bA}\hh_0)$. Then, the integral in \eqref{eq:int} can be simplified as
\begin{align}
    \eqref{eq:int} = \sqrt{\det(\bM)} \exp\left(\frac{1}{2(1+2\sigma^2_{\aug})}\left[ (\bmu+\tilde{\bA}\hh_0)^\top \bM (\bmu+\tilde{\bA}\hh_0) - \bmu^\top \bmu - \hh_0^\top \tilde{\bA}\hh_0\right]\right).
\end{align}
Note that $\bM \tilde \bA = \bI - \bM = \tilde \bA \bM$,
we can rewrite \eqref{eq:uni_loss} as
\begin{align*}
& \EE_{\hh^{-}, \hh \mid \hh_0}\left[\exp\left(\frac{\siml^*(\bW \hh,\bW \hh^{-})}{\tau}\right)\bigg| \hh_0\right]\\ = &\frac{1}{2} \sqrt{\det(\bM)} \exp\left(\frac{1}{2(1+2\sigma^2_{\aug})}\left[ (\bmu+\tilde{\bA}\hh_0)^\top \bM (\bmu+\tilde{\bA}\hh_0) - \bmu^\top \bmu - \hh_0^\top \tilde{\bA}\hh_0\right]\right)\\
& \qquad + \frac{1}{2}\sqrt{\det(\bM)} \exp\left(\frac{1}{2(1+2\sigma^2_{\aug})}\left[ (-\bmu+\tilde{\bA}\hh_0)^\top \bM (-\bmu+\tilde{\bA}\hh_0) - \bmu^\top \bmu - \hh_0^\top \tilde{\bA}\hh_0\right]\right)\\
= &\frac{1}{2}\sqrt{\det(\bM)} \exp\left(\frac{1}{2(1+2\sigma^2_{\aug})}\left[\bmu^\top(\bM-\bI)\bmu + \hh_0^\top(\bM - \bI)\hh_0 - 2\bmu^\top (\bM-\bI) \hh_0\right]\right)\\
& \qquad \qquad \cdot \left(1+\exp\left(\frac{1}{2(1+2\sigma^2_{\aug})}\left[4\bmu^\top (\bM-\bI) \hh_0\right]\right)\right).
\end{align*}
Further, as $\bI - \bM \succcurlyeq 0$,
\begin{align*}
& \EE_n \left\{\EE_{\hh^{-}, \hh \mid \hh_0}\left[\exp\left(\frac{\siml^*(\bW \hh,\bW \hh^{-})}{\tau}\right)\bigg| \hh_0\right]\right\}\\ = & \frac{1}{2}\sqrt{\det(\bM)} \EE_n \bigg\{ \exp\left(-\frac{1}{2(1+2\sigma_\aug^2)}(\hh_0 - \bmu)^\top (\bI-\bM) (\hh_0 - \bmu)\right)\\ 
& \qquad \qquad  +  \exp\left(-\frac{1}{2(1+2\sigma_\aug^2)}(\hh_0 + \bmu)^\top (\bI-\bM) (\hh_0 + \bmu)\right) \bigg\}.
\end{align*}
Taking the logarithm, the uniformity loss can be simplified as
\begin{align*}
\calL_{n,2}(\bW) &= \log \left(\EE_n \EE_{\hh^{-},\hh \mid \hh_0}\left[\exp\left(\frac{\siml^*(\bW \hh,\bW \hh^{-})}{\tau}\right)\right]\right)\\ & = -\log 2 + \frac{1}{2}\log {\rm det}[\bM] + \log \left(\EE_n S_{\hh_0}\right),
\end{align*}
where
\begin{align*}
S_{\hh_0} &\coloneqq \exp\left(-\frac{1}{2(1+2\sigma_\aug^2)}(\hh_0 - \bmu)^\top (\bI-\bM) (\hh_0 - \bmu)\right)\\ 
& \qquad \qquad +  \exp\left(-\frac{1}{2(1+2\sigma_\aug^2)}(\hh_0 + \bmu)^\top (\bI-\bM) (\hh_0 + \bmu)\right).
\end{align*}
Organizing the terms above, we have
\begin{align*}
\calL_n(\bW) = - \log 2 + \frac{\sigma_\aug^2}{\tau \alpha} \| \bW \|_{\mathrm{F}}^2 + \frac{1}{2}\log {\rm det}[\bM] + \log\left(\E_n S_{\hh_0}\right).
\end{align*}

When $\alpha=(1+\sigma_\aug^2)\|\bW\|_{\mathrm{F}}^{2}+\|\bW \bmu\|_{2}^{2}\gg\sigma_{j}^{2}$ for any singular value $\sigma_j$ of $\bW$, we have $\bM \approx \bI - (\tau \alpha)^{-1}(1+2\sigma_\aug^2)\bW^\top \bW$ and
\[
\log\Big({\rm det}[\bM]\Big) = \log \left(\prod_{j=1}^p \left(1 + \frac{1+2\sigma^2_\aug}{\tau \alpha} \sigma^2_j \right)^{-1}\right) \approx -\frac{1+2\sigma^2_\aug}{\tau \alpha} \| \bW \|_{\mathrm{F}}^2.
\]
We can then approximate $S_n$ by
\begin{align*}
\tilde S_{\hh_0} \coloneqq \exp\left(-\frac{1}{2\tau \alpha}\|\bW(\hh_0-\bmu)\|^2\right) + \exp\left(-\frac{1}{2\tau \alpha}\|\bW(\hh_0+\bmu)\|^2\right).
\end{align*}
As a result, we have the approximation
\begin{align*}
    \calL_{n,2}(\bW) & \approx -\log 2 -\frac{1+2\sigma^2_\aug}{2\tau \alpha} \| \bW \|_{\mathrm{F}}^2 + \log\left(\EE_n \tilde S_{\hh_0}\right) \eqqcolon \tilde \calL_{n,2}(\bW).
\end{align*}
Therefore, recall that $\alpha = (1+\sigma_\aug^2)\|\bW\|_{\mathrm{F}}^2 + \|\bW\bmu\|^2$, we have the following approximation for $\calL_n(\bW)$:
\begin{align}\label{eq:loss_1}
\tilde \calL_n(\bW) &= \calL_{n,1}(\bW) + \tilde \calL_{n,2}(\bW)\nonumber\\
& = -\log 2-\frac{1}{2\tau \alpha} \| \bW \|_{\mathrm{F}}^2 + \log\left(\EE_n \tilde S_{\hh_0}\right).
\end{align}

As $\hh_0 \sim \calN(y\cdot \bmu,\bI)$, then
\[
\log \left(\EE_{\hh_0} \tilde S_{\hh_0}\right) = -\frac{1}{2} \log\left(1+\frac{\sigma_j^2}{\tau \alpha}\right) + \log\left(1 + \exp\left(-\frac{2\sigma_j^2}{\sigma_j^2 + \tau \alpha} (\vv_j^\top \bmu)\right)\right).
\]
Under the assumption that $\alpha \gg \|\bW\|_{\mathrm{F}}^2$, we further have 
\[
\log \left(\EE_{\hh_0} \tilde S_{\hh_0}\right) \approx -\frac{1}{2\tau \alpha}\|\bW\|_{\mathrm{F}}^2 + \log\left(1 + \exp\left(-\frac{2}{\tau \alpha} \|\bW \bmu\|^2\right)\right).
\]
Accordingly, if we replace $\EE_n$ by $\EE_{\hh_0}$ in $\tilde \calL_n(\bW)$, then the loss can be approximated by
\[
-\log 2 - \frac{1}{\tau \alpha}\|\bW\|_{\mathrm{F}}^2 + \log\left(1 + \exp\left(-\frac{2}{\tau \alpha} \|\bW \bmu\|^2\right)\right),
\]
which is exactly the loss function $\tilde \calL(\bW)$ defined in \eqref{eq:approx-loss}.

\section{Proofs for Section~\ref{sec:result}} 

\subsection{Proof of Theorem~\ref{thm:low-d}}\label{sec:proof_low_d}

Denote $\btheta = \bW \bbeta$, $\bOmega = \bW^{-2}$, and $\fullcoef = (\gamma, \btheta^\top)^\top$.
One can rewrite the logistic loss function as 
\begin{align*}
    \tilde{\ell}_n(\fullcoef;\hh_0) & = \ell_{n,logis}(\fullcoef;\hh_0) + \lambda_n \btheta^\top \bOmega \btheta  = \EE_n\left\{\ell(\fullcoef; \hh_0, \lambda_n)\right\}.
\end{align*}
where we define
\[
\ell(\fullcoef; \hh_0, \lambda_n) = \log \left[1 + \exp(-y(\gamma+\hh_0^\top \btheta))\right] + \lambda_n \btheta^\top \bOmega \btheta\] 
and $\ell_{n,logis}(\fullcoef;\hh_0)~= \EE_n \log \left[1 + \exp(-y(\gamma+\hh_0^\top \btheta))\right]$.
We also denote 
\[
\ell^*_n(\fullcoef) = \EE \left\{\ell(\fullcoef; \hh_0, \lambda_n)\right\},
\]
where the loss depends on $n$ through $\lambda_n$. 


Since both $\ell^*_n(\fullcoef)$ and $\tilde{\ell}_n(\fullcoef;\hh_0)$ are strictly convex in $\fullcoef$ and are strongly convex when $\lambda_n \geq 0$, we denote  
\begin{align}
\hat{\fullcoef}_n = (\hat{\gamma},\hat{\btheta}^\top)^\top = \argmin~~\tilde{\ell}_n(\fullcoef;\hh_0),\quad \text{and}\quad \fullcoef_n^* = (\gamma_n^*,\btheta_n^{*\top})^\top = \argmin~~\ell^*_n(\fullcoef).
\end{align}

\paragraph{Step 1: characterizing $\fullcoef_n^*$.}
Under the 2-GMM model, we have the following characterization of the solution $\fullcoef_n^*$ to the population loss $\ell^*_n(\fullcoef)$.

\begin{lemma}\label{lem:solu_logis}
There exists a constant $C>0$ such that the following holds. For any given $(\lambda_n)_{n \ge 1}$ with $\lambda_n \ge 0$, the population loss minimizer $\fullcoef_n^*$ is unique and has the form
\begin{align}
    \fullcoef^*_n = (0, \kappa_{\lambda_n} \bmu^\top)^\top
\end{align}
where $\kappa_{\lambda_n} \in (0,C)$ is the zero of the function
\[
 \psi(\kappa) \coloneqq \EE\left\{\frac{1+(1-\kappa)\exp(\kappa \Vert \bmu \Vert^2 + \kappa \bmu^\top \bG)}{\left(1+\exp(\kappa \Vert \bmu \Vert^2 + \kappa \bmu^\top \bG)\right)^2}\right\} - 2\lambda_n (1+\eta)^{-2} \kappa.
\]
Moreover, if $\lim_{n \to \infty}\lambda_n = \infty$, then $\lim_{n \to \infty} \lambda_n \kappa_{\lambda_n} = \frac{1}{4} (1+\eta)^2$.
\end{lemma}
As a corollary, we see that $\fullcoef_n^*$ has bounded norm.

\paragraph{Step 2: characterizing $\hat{\fullcoef}_n - \fullcoef_n^*$.}
Denote the Fisher information matrix
\begin{align}
    \bI_n(\fullcoef) = \EE_n\left\{\frac{\bar{\hh}_0\bar{\hh}_0^\top \exp(y\bar{\hh}_0^\top \fullcoef)}{(1+\exp(y\bar{\hh}_0^\top \fullcoef))^2}\right\},~~\bI(\fullcoef) = \EE\left\{\frac{\bar{\hh}_0\bar{\hh}_0^\top \exp(y\bar{\hh}_0^\top \fullcoef)}{(1+\exp(y\bar{\hh}_0^\top \fullcoef))^2}\right\}.
\end{align}
We first have the following lemma for consistency.
\begin{lemma}\label{lem:logis_ulln}
For any $(\lambda_n)_{n \geq 1}$ with $\lambda_n \geq 0$ and any constant $R > 0$, 
\begin{align*}
& \sup_{\fullcoef \in \RR^{p+1}:~\|\fullcoef\|\leq R} \bigg| \EE_n \ell(\fullcoef; \hh_0,\lambda_n) - \EE \ell(\fullcoef; \hh_0,\lambda_n) \bigg| \overset{p}{\rightarrow} 0 \,, \\
& \sup_{\fullcoef \in \RR^{p+1}:~\|\fullcoef\|\leq R} \bigg| \bI_n(\fullcoef) - \bI(\fullcoef) \bigg| \overset{p}{\rightarrow} 0 \, .
\end{align*}
As a result, we have
\[
\norm{\hat{\fullcoef}_n - \fullcoef^*_n} \overset{p}{\rightarrow} 0, \qquad \text{as}~n \to \infty
\]
\end{lemma}

Based on the consistency result, we consider the score functions for $\tilde \ell_n(\fullcoef;\hh_0)$ and $\ell^*_n(\fullcoef)$. Denoting $\bar{\bOmega} = {\rm diag}\{0,\bOmega\}$, we obtain
\begin{align}
	0 = \frac{\partial}{\partial \fullcoef}\tilde{\ell}_n(\hat{\fullcoef}_n;\hh_0) = -\EE_n\left\{\frac{\bar{\hh}_0y}{1+\exp(y\bar{\hh}_0^\top \hat{\fullcoef}_n)}\right\} + 2\lambda_n \bar{\bOmega} \hat{\fullcoef}_n.\label{eq:sc1}
\end{align}
\begin{align}
	& 0 = \frac{\partial}{\partial \fullcoef}\ell^*_n(\fullcoef_n^*) = -\EE\left\{\frac{\bar{\hh}_0y}{1+\exp(y\bar{\hh}_0^\top \fullcoef_n^*)}\right\} + 2\lambda_n \bar{\bOmega} \fullcoef_n^*.\label{eq:sc2}
\end{align}
According to Lemma~\ref{lem:logis_ulln}, we have the first-order approximation
\begin{align}
    \frac{\partial}{\partial \fullcoef}\tilde{\ell}_{n,logis}(\hat{\fullcoef}_n;\hh_0) - \frac{\partial}{\partial \fullcoef} \tilde{\ell}_{n,logis}(\fullcoef_n^*;\hh_0) = \bI_n(\fullcoef_n^*)(\hat{\fullcoef}_n - \fullcoef^*) + o_{\PP}(\Vert \hat{\fullcoef}_n - \fullcoef_n^* \Vert).
\end{align}
By calculating the difference between \eqref{eq:sc1} and \eqref{eq:sc2}, we obtain
\begin{align}
-\bI_n(\fullcoef_n^*)(\hat{\fullcoef}_n - \fullcoef_n^*) - 2\lambda_n \bar{\bOmega}(\hat{\fullcoef}_n - \fullcoef_n^*) + (\EE_n - \EE)\left[\frac{\bar{\hh}_0y}{1+\exp(y\bar{\hh}_0^\top \fullcoef_n^*)}\right] + o_{\PP}(\Vert \hat{\fullcoef}_n - \fullcoef_n^* \Vert) = 0.
\end{align}
When $p$ is fixed and $n \rightarrow +\infty$, since $\fullcoef_n^*$ has bounded norm, then by  Lemma~\ref{lem:logis_ulln}, we have $\Vert \left(\bI_n(\fullcoef_n^*) - \bI(\fullcoef_n^*)\right)(\hat{\fullcoef}_n - \fullcoef_n^*) \Vert = o_{\PP}(\Vert \hat{\fullcoef}_n - \fullcoef_n^* \Vert)$.
In addition, denote
\[
\zz_n = \frac{\bar{\hh}_0y}{1+\exp(y\bar{\hh}_0^\top \fullcoef_n^*)},
\]
where $\EE \zz_n = \bzero$ and $cov(\zz_n) = \bI(\fullcoef^*_n)$, then we have the following lemma.
\begin{lemma}\label{lem:l_clt}
We have $\bI(\fullcoef^*_n) \succ \bzero$ and by the Lindeberg-Feller central limit theorem,
\begin{align}
\sqrt{n}\left(\bI(\fullcoef_n^*)\right)^{-1/2}(\EE_n - \EE)\left[\frac{\bar{\hh}_0y}{1+\exp(y\bar{\hh}_0^\top \fullcoef_n^*)}\right] \overset{d}{\longrightarrow} \calN(\bzero,\bI).
\end{align}
\end{lemma}
Let $\bar{\bzeta} = \sqrt{n}(\hat{\fullcoef}_n - \fullcoef_n^*)$, then 
\begin{align}
\left(\bI(\fullcoef_n^*)\right)^{-1/2}(\bI(\fullcoef_n^*) + 2\lambda_n \bar{\bOmega}) \bar{\bzeta} \overset{d}{\longrightarrow} \calN(0,\bI).\label{eq:clt}
\end{align}

\paragraph{Step 3: calculating misclassification error.}
Conditioning on $y=1$, 
\begin{align}
    \texttt{Err}_1 &= \PP(y(\langle \hh_0, \hat{\btheta}\rangle + \hat{\gamma}) < 0|y=1) = \PP(\langle \bmu + \bG, \kappa_{\lambda_n}\bmu+\frac{1}{\sqrt{n}} \bzeta \rangle + \hat{\gamma} < 0),\label{eq:err1_1}
\end{align}
where $\bG \sim \calN(\bzero,\bI)$, then~\eqref{eq:err1_1} can be written using the Gaussian cdf as
\begin{align}
    \eqref{eq:err1_1} & =  \Phi\left(-\frac{\hat{\gamma}+\langle \bmu, \kappa_{\lambda_n}\bmu+\frac{1}{\sqrt{n}} \bzeta \rangle}{\Vert \kappa_{\lambda_n}\bmu+\frac{1}{\sqrt{n}} \bzeta \Vert}\right)\nonumber\\
    & = \Phi\left(-\Vert \bmu \Vert \cos(\bmu,\kappa_{\lambda_n}\bmu+\frac{1}{\sqrt{n}} \bzeta) - \frac{\hat{\gamma}}{\Vert \kappa_{\lambda_n}\bmu+\frac{1}{\sqrt{n}} \bzeta \Vert}\right).\label{eq:err1_2}
\end{align}
By the symmetry of the distribution of $y$ and the fact that $y\hh_0 \sim \bmu+\calN(\bzero,\bI_p)$, if we condition on $y = -1$,
\begin{align}
    \texttt{Err}_{-1} &= \PP(y(\langle \hh_0, \hat{\btheta}\rangle + \hat{\gamma}) < 0|y=-1) = \PP(\langle \bmu + \bG, \kappa_{\lambda_n}\bmu+\frac{1}{\sqrt{n}} \bzeta \rangle - \hat{\gamma} < 0)\nonumber\\
    &=\Phi\left(-\Vert \bmu \Vert \cos(\bmu,\kappa_{\lambda_n}\bmu+\frac{1}{\sqrt{n}} \bzeta) + \frac{\hat{\gamma}}{\Vert \kappa_{\lambda_n}\bmu+\frac{1}{\sqrt{n}} \bzeta \Vert}\right).\label{eq:err2}
\end{align}
Then, we will next analyze the two terms
\[
\cos(\bmu,\kappa_{\lambda_n}\bmu+\frac{1}{\sqrt{n}} \bzeta) \qquad \text{and} \qquad \frac{\hat{\gamma}}{\Vert \kappa_{\lambda_n}\bmu+\frac{1}{\sqrt{n}} \bzeta \Vert}
\] 
to see how they will affect the misclassification error.

\paragraph{Step 4: finalizing conclusions. }
We consider the following regimes.
\paragraph{Case 1: $\limsup_{n \rightarrow +\infty} \lambda_n \leq \bar{\lambda} < +\infty$.}
In this case, by Lemma~\ref{lem:solu_logis}, we have $\liminf_{n \rightarrow +\infty} \kappa_{\lambda_n} \geq c_{\kappa} > 0$. Then, 
\begin{align*}
& \cos(\bmu,\kappa_{\lambda_n}\bmu+\frac{1}{\sqrt{n}} \bzeta) = 1 + o_{\PP}(1)\\
& \frac{\hat{\gamma}}{\Vert \kappa_{\lambda_n}\bmu+\frac{1}{\sqrt{n}} \bzeta \Vert} = \frac{\hat{\gamma}}{\Vert \kappa_{\lambda_n}\bmu \Vert} + o_{\PP}(1).
\end{align*}
Denote $q(t) = \frac{1}{1+e^t}$, since $\hh_0 y\sim\calN(\bmu,\bI_p)$, one can check that for any $\fullcoef = (\gamma, \btheta^\top)^\top$ and $2 \leq j \leq p+1$,
\begin{align}
[\bI(\fullcoef)]_{j,1} = &\frac{1}{2} \EE[q(\hh_0^\top\btheta y + \gamma y)(1-q(\hh_0^\top\btheta y + \gamma y))(\hh_0)_j|y=1]\nonumber\\ +&\frac{1}{2} \EE[q(\hh_0^\top\btheta y + \gamma y)(1-q(\hh_0^\top\btheta y + \gamma y))(\hh_0)_j|y=-1] = 0.
\end{align}
Then, by~\eqref{eq:clt}, we have
\begin{align}
    \sqrt{n i(\btheta^*_n)} \hat{\gamma} \overset{d}{\longrightarrow} \calN(0,1),
\end{align}
where $i(\btheta_n^*) :=  [\bI(\fullcoef_n^*)]_{1,1}$ with $\gamma^*=0$ and $\left[\bI(\btheta_n^*)\right]_{1,1}$ denotes the submatrix consisting of the first column and the last row of $\bI(\btheta_n^*)$.
Since $\liminf_{n \rightarrow +\infty} i(\btheta_n^*) > 0$, we have $\hat{\gamma} = [n i(\btheta_n^*)]^{-1/2} G^\prime + o_{\PP}(1)$ with $G^\prime \sim \calN(0,1)$, then
\begin{align}
    \texttt{Err} = \frac{1}{2}\texttt{Err}_1 + \frac{1}{2} \texttt{Err}_{-1} = \Phi\left(-\Vert \bmu \Vert + \frac{[i(\btheta_n^*)]^{-1/2} G^\prime}{\sqrt{n}\kappa_{\lambda_n}\Vert \bmu \Vert}\right) + o_{\PP}(1).
\end{align}
From the reasoning above, we have $\kappa_{\lambda_n}^{-1} [n i(\btheta_n^*)]^{-1/2} = o(1)$, then
\[
\texttt{Err} = \Phi\left(-\Vert \bmu \Vert\right) + o_{\PP}(1).
\]

\paragraph{Case 2: $\limsup_{n \rightarrow +\infty} \lambda_n = +\infty$.} Without loss of generality, consider $\lambda_n \rightarrow +\infty$ as $n \rightarrow +\infty$. By Lemma~\ref{lem:solu_logis}, $\lambda_n \kappa_{\lambda_n} = \frac{1}{4} (1+\eta)^{-2} + o(1)$, thus we have $\bI(\fullcoef^*_n) = \frac{1}{4} \bI_{p+1} + o(1)$, particularly $i(\btheta^*_n) = \frac{1}{4} + o(1)$. By~\eqref{eq:clt}, we have $\bzeta = O_{\PP}(\kappa_{\lambda_n})$, then similar with before
\begin{align*}
& \cos(\bmu,\kappa_{\lambda_n}\bmu+\frac{1}{\sqrt{n}} \bzeta) = 1 + o_{\PP}(1)\\
& \frac{\hat{\gamma}}{\Vert \kappa_{\lambda_n}\bmu+\frac{1}{\sqrt{n}} \bzeta \Vert} = \frac{2}{\sqrt{n} \kappa_{\lambda_n}\Vert \bmu \Vert} + o_{\PP}(1).
\end{align*}
\begin{enumerate}
\item If $\sqrt{n}/\lambda_n = o(1)$, then $\sqrt{n} \kappa_{\lambda_n} = o(1)$. We have $\EE[\Err] \rightarrow 1$ as $\lambda_n \rightarrow +\infty$.

\item If $\lambda_n = a \cdot \sqrt{n}$ with a positive constant $a$, then $a \sqrt{n} \kappa_n = \frac{1}{4}(1+\eta)^2 + o(1)$ and
\[
\Err = \Phi\left(-\Vert \bmu \Vert + \frac{8 G^\prime}{a (1+\eta)^2\Vert \bmu \Vert}\right) + o_{\PP}(1).
\]
We have the following lemma
\begin{lemma}\label{lem:cdf1}
Let $\tau,~w > 0$ and $U \sim\calN(0,1)$, then
\begin{align}
    \frac{d}{dw}\EE\Phi(-\tau+w\cdot U) = \frac{\tau w}{\sqrt{2\pi}}\frac{1}{(1+w^2)^{3/2}}\exp\left(-\frac{\tau^2}{2(1+w^2)}\right) > 0.
\end{align}
\end{lemma}
Then, $\EE[\Err]$ is increasing in $\frac{8}{a (1+\eta)^2\Vert \bmu \Vert}$, thus is decreasing in $\eta$.

\item If $\lambda_n = a \cdot b_n$ with a constant $a > 0$ and $b_n/\sqrt{n} = o(1)$, then $a b_n \kappa_n = \frac{1}{4}(1+\eta)^2 + o(1)$ and
\[
\Err = \Phi\left(-\Vert \bmu \Vert + \frac{8 G^\prime}{a (1+\eta)^2\Vert \bmu \Vert}\frac{b_n}{\sqrt{n}}\right) + o_{\PP}(1) = \Phi(-\|\bmu\|) + O_{\PP}\left(\frac{b_n}{\sqrt{n}}\right).
\]
As a result, $\EE[\Err] = \Phi(-\|\bmu\|) + O(b_n/\sqrt{n})$.
\end{enumerate}
Combining pieces above, low-dimensional logistic regression with $\ell_2$ regularization has decreasing prediction error in the expansion parameter $\eta$ only if $\lambda_n \asymp \sqrt{n}$. Otherwise, $\texttt{Err}$ asymptotically remains constant in $\eta$.

\subsubsection{Proof of Lemma~\ref{lem:solu_logis}}
Since the population loss function $\tilde{\ell}(\fullcoef)$ is strongly convex in $\fullcoef$, it suffices to show that (i) there exists some $\kappa := \kappa_{\lambda_n} \in (0,C)$ such that $\fullcoef = (0,\kappa \bmu^\top)^\top$ satisfies the score equation
\begin{align}\label{eq:sc_eq}
    \EE\left\{\frac{\bar{\hh}_0y}{1+\exp(y\bar{\hh}_0^\top \fullcoef)}\right\} = 2\lambda_n \bar{\bOmega} \fullcoef;
\end{align}
and that (ii) $\lim_{n \to \infty} \lambda_n \kappa_{\lambda_n} = \frac{1}{4} (1+\eta)^{-2}.$ Due to the block structure of $\bar{\bOmega}$, we separate the verification process into two parts.

\paragraph{Step 1: check the first coordinate.}
Clearly, the first coordinate on the right-hand side of~\eqref{eq:sc_eq} is zero. When it comes to the left-hand side, 
we have
\begin{align}\label{eq:gamma}
    \EE\left\{\frac{y}{1+\exp(y\bar{\hh}_0^\top \fullcoef)}\right\} = \frac{1}{2} \EE\left\{\frac{1}{1+\exp(\kappa \bmu^\top(\bmu+\bG))}\right\} + \frac{1}{2} \EE\left\{\frac{-1}{1+\exp( \kappa \bmu^\top(\bmu+\bG))}\right\} = 0.
\end{align}
Here, $\bG \sim \calN(\bzero,\bI_p)$. Thus, \eqref{eq:gamma} holds in the first coordinate for $\fullcoef = (0,\kappa \bmu^\top)^\top$.

\paragraph{Step 2: check the remaining coordinates.}
Now we move on to the remaining coordinates, i.e., the coordinates corresponding to 
$\hh_0$.

The right-hand side of~\eqref{eq:sc_eq} obeys
\begin{align}
   2\lambda_n \bOmega \btheta = 2\lambda_n (1+\eta)^{-2} \kappa \bmu,
\end{align}
by the definition of $\bOmega$.
For the left hand side, since $\hh_0 y \overset{d}{=} \bmu + \bG \sim \calN(\bmu, \bI_p)$, we have 
\begin{align}
    \EE\left\{\frac{\hh_0 y}{1+\exp(y\kappa \hh_0^\top \bmu)}\right\} &= \EE\left\{\frac{\bmu + \bG}{1+\exp(\kappa \Vert \bmu \Vert^2 + \kappa \bmu^\top \bG)}\right\}\\ & = \EE\left\{\frac{1}{1+\exp(\kappa \Vert \bmu \Vert^2 + \kappa \bmu^\top \bG)}\right\}\bmu + \EE\left\{\frac{\bG}{1+\exp(\kappa \Vert \bmu \Vert^2 + \kappa \bmu^\top \bG)}\right\}.
\end{align}
Apply Stein's identity to see that 
\begin{align}
    \EE\left\{\frac{\bG}{1+\exp(\kappa \Vert \bmu \Vert^2 + \kappa \bmu^\top \bG)}\right\} = - \kappa \EE\left\{\frac{\exp(\kappa \Vert \bmu \Vert^2 + \kappa \bmu^\top \bG)}{\left(1+\exp(\kappa \Vert \bmu \Vert^2 + \kappa \bmu^\top \bG)\right)^2}\right\}\bmu,
\end{align}
which leads to the conclusion 
\begin{equation}\label{eq:someconclu}
\EE\left\{\frac{\hh_0 y}{1+\exp(y\kappa \hh_0^\top \bmu)}\right\} = 
 \EE\left\{\frac{1+(1-\kappa)\exp(\kappa \Vert \bmu \Vert^2 + \kappa \bmu^\top \bG)}{\left(1+\exp(\kappa \Vert \bmu \Vert^2 + \kappa \bmu^\top \bG)\right)^2}\right\} \, \bmu.
\end{equation}
Then it boils down to showing that there exists a constant $\kappa > 0$ such that
\begin{align}
    \EE\left\{\frac{1+(1-\kappa)\exp(\kappa \Vert \bmu \Vert^2 \kappa \bmu^\top \bG)}{\left(1+\exp(\kappa \Vert \bmu \Vert^2 + \kappa \bmu^\top \bG)\right)^2}\right\} = 2\lambda_n (1+\eta)^{-2} \kappa.
\end{align}
Recall the function
\begin{align}
    \psi(\kappa) = \EE\left\{\frac{1+(1-\kappa)\exp(\kappa \Vert \bmu \Vert^2 + \kappa \bmu^\top \bG)}{\left(1+\exp(\kappa \Vert \bmu \Vert^2 + \kappa \bmu^\top \bG)\right)^2}\right\} - 2\lambda_n (1+\eta)^{-2} \kappa.
\end{align}
We see that $\psi(0) = \frac{1}{2} > 0$. 

(i) First, we claim that there exists $C>0$ depending solely on $\norm{\bmu}$ such that $\psi(C) < 0$. Once this is proved, since $\psi(\kappa)$ is a continuous function, by the intermediate value theorem, there must exist $\kappa \in (0,C)$ such that $\psi(\kappa)=0$. 

To prove this claim, we use the identity \eqref{eq:someconclu} to express $\psi(\kappa)$ into 
\begin{align*}
\psi(\kappa) &= \norm{\bmu}^{-2} \cdot \EE\left\{\frac{\norm{\bmu}^2 + \bG^\top \bmu}{1+\exp(\kappa \Vert \bmu \Vert^2 + \kappa \bmu^\top \bG)}\right\}  - 2\lambda_n (1+\eta)^{-2} \kappa \\
& \le \norm{\bmu}^{-2} \cdot \EE\left\{\frac{\norm{\bmu}^2 + \norm{\bmu}G}{1+\exp(\kappa \Vert \bmu \Vert^2 + \kappa \norm{\bmu}G)}\right\} 
\end{align*}
where $G \sim \mathcal{N}(0,1)$. Denote an event $A = \{ \norm{\bmu}^2 + \norm{\bmu} G < 0\}$. For any fixed $\norm{\bmu} > 0$, the probability $\PP(A) \in (0,1)$. By the dominated convergence theorem,
\begin{align*}
&\lim_{\kappa \to \infty} \EE\left\{\frac{(\norm{\bmu}^2 + \norm{\bmu}G)_+}{1+\exp(\kappa \Vert \bmu \Vert^2 + \kappa \norm{\bmu}G)}\right\} = 0, \\
&\lim_{\kappa \to \infty} \EE\left\{\frac{(\norm{\bmu}^2 + \norm{\bmu}G)_-}{1+\exp(\kappa \Vert \bmu \Vert^2 + \kappa \norm{\bmu}G)}\right\} = \EE\big[  (\norm{\bmu}^2 + \norm{\bmu}G)_- \big] > 0.
\end{align*}
Combining the two parts, we deduce that $\limsup_{\kappa \to \infty} \psi(\kappa) \le - \EE\big[  (\norm{\bmu}^2 + \norm{\bmu}G)_- \big] < 0$. 

(ii) Since for any $a \in \mathbb{R}$, $a < \kappa^{-1} + a \le \kappa^{-1} \exp(\kappa a)$, we have
\begin{equation*}
\EE\left\{\frac{\norm{\bmu}^2 + \norm{\bmu}G}{1+\exp(\kappa \Vert \bmu \Vert^2 + \kappa \norm{\bmu}G)}\right\} \le \kappa^{-1} \EE\left\{\frac{\exp(\kappa \Vert \bmu \Vert^2 + \kappa \norm{\bmu}G)}{1+\exp(\kappa \Vert \bmu \Vert^2 + \kappa \norm{\bmu}G)}\right\} < \kappa^{-1}.
\end{equation*}
The equation $\psi(\kappa_{\lambda_n})=0$ must imply 
\begin{equation*}
2\lambda_n(1+\eta)^{-2} \kappa_{\lambda_n} < \frac{1}{\kappa_{\lambda_n} \norm{\bmu}^2}.
\end{equation*}
If $\lambda_n \to \infty$ as $n \to \infty$, then we must have $\kappa_{\lambda_n} = o(1)$. Taking the limit $\lim_{n \to \infty}\psi(\kappa_{\lambda_n})$, we get
\begin{equation*}
\frac{1}{2} - 2(1+\eta)^{-2} \lim_{n \to \infty} \lambda_n \kappa_{\lambda_n} = 0
\end{equation*}
so $\lim_{n \to \infty}  \lambda_n \kappa_{\lambda_n} = \frac{1}{4} (1+\eta)^{2}$.

Combining the two steps finishes the proof. 

\subsubsection{Proof of Lemma~\ref{lem:logis_ulln}}
Recall that for any $\lambda_n \geq 0$ and any $\fullcoef = (\gamma,\btheta^\top)^\top \in \RR^{p+1}$,
\[
\ell(\fullcoef; \hh_0, \lambda_n) = \log \left[1 + \exp(-y(\gamma+\hh_0^\top \btheta))\right] + \lambda_n \btheta^\top \bOmega \btheta.
\]
\paragraph{Step 1: Applying the uniform law of large numbers.}
Since the term $\lambda_n \btheta^\top \bOmega \btheta$ is canceled out, we need to show that for any $R$ and the compact ball $B_R = \{\|\fullcoef\| \leq R\}$,
\begin{align}\label{eq:logis_ulln_eq}
\sup_{\fullcoef \in \bB_R} \bigg|(\EE_n - \EE) \left\{\log \left[1 + \exp(-y(\gamma+\hh_0^\top \btheta))\right]\right\}\bigg| \overset{p}{\rightarrow} 0.
\end{align}
We see that $\ell(\fullcoef; \hh_0, \lambda_n)$ is continuous in $\fullcoef$ for any $\hh_0$ and is measurable in $(\hh_0,y)$ for any $\fullcoef$.
In addition, for any $R$,
\[
\log \left[1 + \exp(-y(\gamma+\hh_0^\top \btheta))\right] \leq \log \left[1 + \exp(R\|\bar \hh_0\|)\right] \quad \text{with} \quad \EE\left\{\log \left[1 + \exp(R\|\bar\hh_0\|)\right]\right\} < +\infty.
\]
Then, we can apply the uniform law of large numbers \citep[Lemma~2.4]{newey1994large} to show~\eqref{eq:logis_ulln_eq}.

It is similar for the convergence of Fisher information as for any $\fullcoef \in \RR^{p+1}$,
\[
\frac{\bar\hh_0 \bar\hh_0^\top \exp(y\bar\hh_0^\top \fullcoef)}{(1+\exp(y\bar\hh_0^\top \fullcoef))^2} \leq \frac{1}{4} \bar\hh_0 \bar\hh_0^\top\quad\text{with}\quad\EE\left\{\frac{1}{4} \bar\hh_0 \bar\hh_0^\top\right\} = \frac{1}{4}\bI_{p+1}.
\]
Then, we have
\[
\sup_{\fullcoef \in \bB_R} \bigg| \bI_n(\fullcoef) - \bI(\fullcoef) \bigg| \overset{p}{\rightarrow} 0.
\]

\paragraph{Step 2: Proving the convergence of the minimizer.} 
With $\hat{\fullcoef}_n$ and $\fullcoef^*_n$ defined before, there exists a constant $R > 0$ such that $\PP(\hat{\fullcoef}_n \notin \bB_R) = o(1)$. Therefore,
\begin{align*}
\ell^*_n(\hat{\fullcoef}_n) - \ell^*_n(\fullcoef^*_n) = \ell^*_n(\hat{\fullcoef}_n) - \tilde{\ell}_{n}(\hat{\fullcoef}_n;\hh_0) + \tilde{\ell}_{n}(\hat{\fullcoef}_n;\hh_0) - \tilde{\ell}_n(\fullcoef^*_n;\hh_0) + \tilde{\ell}_n(\fullcoef^*_n;\hh_0) - \ell^*_n(\fullcoef^*_n).
\end{align*}
Here, we have $\ell^*_n(\hat{\fullcoef}_n) - \tilde{\ell}_{n}(\hat{\fullcoef}_n;\hh_0) \leq \sup_{\fullcoef \in \bB_R} \big| (\EE_n - \EE) \ell(\fullcoef;\hh_0,\lambda_n) \big| \overset{p}{\rightarrow} 0$, $\tilde{\ell}_{n}(\hat{\fullcoef}_n;\hh_0) - \tilde{\ell}_n(\fullcoef^*_n;\hh_0) \leq 0$ by the optimality of $\hat{\fullcoef}_n$, and $\tilde{\ell}_n(\fullcoef^*_n;\hh_0) - \ell^*_n(\fullcoef^*_n) = o_{\PP}(1)$ by the law of large numbers. Consequently, we have $\ell^*_n(\hat{\fullcoef}_n) - \ell^*_n(\fullcoef^*_n) \leq o_{\PP}(1)$. Since $\ell^*_n(\fullcoef)$ is strictly convex and by definition $\fullcoef^*_n = \argmin_{\fullcoef \in \RR^{p+1}} \ell^*_n(\fullcoef)$, we further have $\hat{\fullcoef}_n - \fullcoef^*_n = o_{\PP}(1)$. 

\subsubsection{Proof of Lemma~\ref{lem:l_clt}}
Denoting
\[
\zz_{n,i} = \frac{\bar{\hh}_{0,i} y_i}{1+\exp(y_i\bar{\hh}_{0,i}^\top \fullcoef_n^*)},
\]
then it suffices to show that for any $\veps > 0$,
\[
\frac{1}{n} \sum_{i=1}^n \EE\left\{\|\zz_{n,i}\|^2 \ind\{\|\zz_{n,i}\| \geq \veps \sqrt{n}\}\right\} = \EE\left\{\|\zz_{n,1}\|^2 \ind\{\|\zz_{n,1}\| \geq \veps \sqrt{n}\}\right\} \overset{p}{\rightarrow} 0.
\]
Let $Z_{n,i} = \|\zz_{n,i}\|^2$, then $\EE Z_{n,i} = \tr (\EE(\zz_{n,i} \zz^\top_{n,i})) = \tr(\bI(\fullcoef^*_n))$. Let $W_{n,i} = Z_{n,i}\ind\{Z_{n,i} \geq \veps^2 n\}$. We can verify that $0 \leq W_{n,i} \leq Z_{n,i}$ and for any $\delta > 0$, by Markov's inequality,
\[
\PP\left(W_{n,i} \geq \delta\right) = \PP\left(Z_{n,i} \geq \epsilon^2 n\right) \leq \frac{\tr(\bI(\fullcoef^*_n))}{n \veps^2} \rightarrow 0, \quad \text{as~} n \rightarrow 0,
\]
which shows that $W_{n,i} \overset{p}{\rightarrow} 0$. Then, by the dominated convergence theorem, we have
\[
\lim_{n \rightarrow +\infty} \EE W_{n,i} = \EE \lim_{n \rightarrow +\infty} W_{n,i} = 0,
\]
which verifies Lindeberg's condition. By the Lindeberg-Feller theorem for multivariate random variables, we have
\[
\sqrt{n} \left[\bI(\fullcoef^*_n)\right]^{-1/2}(\EE_n \zz_{n,i} - \EE \zz_{n,i}) \overset{d}{\rightarrow} \calN(\bzero, \bI).
\]

\subsection{High-dimensional regime: max-margin classifier}\label{sec:proof_high_d}

Written in terms of the input data $\{\hh_{0,i}\}$, the max-margin classifier~\eqref{opt:maxmargin} is equivalent to
\begin{align}
\begin{split}\label{opt:maxmargin2}
    \max_{\bbeta\in\RR^p} &\quad \min_{i\leq n} y_i \langle \bbeta, \hh_{0,i} \rangle\\
    {\rm subject~to} &\quad \Vert \bbeta \Vert_{\bOmega} = \sqrt{\bbeta^\top \bOmega \bbeta} \leq 1,
\end{split}
\end{align}
where $\bOmega \coloneqq (\bI+\eta \bmu_0 \bmu_0^\top)^{-2} = \bI - \left(1-\frac{1}{(1+\eta)^2}\right)\bmu_0 \bmu_0^\top$, and $\bmu_0 = \frac{1}{\Vert \bmu \Vert} \bmu$.

\subsubsection{Preliminaries}
Before presenting the proof of Theorem~\ref{thm:CGMT}, we first introduce important quantities and their equivalence.

Define a compact set $\Theta = \{\bbeta \in \RR^p: \bbeta^\top \bOmega \bbeta \leq 1\}$. Similar to the development in \citet{montanari2019generalization}, for any positive margin $\kappa > 0$, we have the equivalence
\begin{align}\label{eq:maxmargin_form}
    \calE_{n,p,\kappa} &= \left\{{\rm there~exists} ~\bbeta \in \Theta : y_i \langle \hh_{0,i},\bbeta \rangle \geq \kappa \text{~for~all~} i \leq n\right\}\nonumber\\ &= \left\{\min_{\bbeta\in\Theta} \max_{\blambda:\Vert \blambda \Vert \leq 1, \by \odot \blambda \geq 0}\frac{1}{\sqrt{p}}\blambda^\top(\kappa \by - \bH_0\bbeta) = 0\right\}.
\end{align}
To simplify the notation, for any set $\bB_p \subseteq \mathbb{R}^{p}$, we define the quantity
\begin{align}
    \xi_{n,p,\kappa}(\bB_p) = \min_{\bbeta\in\bB_p} \max_{\blambda:\Vert \blambda \Vert \leq 1, \by \odot \blambda \geq 0}\frac{1}{\sqrt{p}}\blambda^\top(\kappa \by - \bH_0\bbeta).
\end{align}
In particular, we set $\xi_{n,p,\kappa} \coloneqq \xi_{n,p,\kappa}(\Theta)$, which obeys 
\begin{align}
    \{\xi_{n,p,\kappa} > 0\} \Longleftrightarrow \calE_{n,p,\kappa}^c~~{\rm and}~~\{\xi_{n,p,\kappa} = 0\} \Longleftrightarrow \calE_{n,p,\kappa}.
\end{align}
As a result, our goal is to analyze whether $\xi_{n,p,\kappa}$ is positive.

\paragraph{Step 1: Applying CGMT. }

Recall the Gaussian mixture model where $\bH_0 = \by \bmu^\top + \bG$ and each entry of $\bG \in \RR^{n \times p}$ is independently drawn from $\calN(0,1)$. As a result, we have $\blambda^\top(\kappa \by - \bH_0\bbeta) = \blambda^\top(\kappa \by - \langle \bmu, \bbeta \rangle\by - \bG \bbeta)$, 
which is a bilinear form of the Gaussian random matrix $\bG$. 
We plan to use Gordan's comparison inequality to simplify the calculation of $\xi_{n,p,\kappa}(\bB_p)$. To do so, we introduce another quantity for an arbitrary set $\bB_p$:
\begin{align}\label{eq:xi}
    \xi^{(1)}_{n,p,\kappa}(\bB_p) = \min_{\bbeta\in\bB_p} \max_{\blambda:\Vert \blambda \Vert \leq 1, \by \odot \blambda \geq 0}\frac{1}{\sqrt{p}}\left\{\Vert \blambda \Vert \; \bg^\top \bbeta + \Vert \bbeta \Vert \; \bh^\top \blambda + \kappa \blambda^\top \by - \langle \bmu, \bbeta \rangle \blambda^\top \by\right\},
\end{align}
where $\bg \sim \calN(\bzero,\bI_p)$ and $\bh \sim \calN(\bzero,\bI_n)$ are two independent Gaussian vectors. Similar to before, we denote $\xi^{(1)}_{n,p,\kappa} \coloneqq \xi^{(1)}_{n,p,\kappa}(\Theta)$. 
The following lemma connects $\xi_{n,p,\kappa}(\bB_p)$ with $\xi^{(1)}_{n,p,\kappa}(\bB_p)$, which is a simple corollary of Gordon's comparison inequality~\citep{gordon1988milman, pmlr-v40-Thrampoulidis15, montanari2019generalization}.
\begin{lemma}\label{lem:cgmt1}
For any $t \in \RR$ and any compact set $\bB_p$,
\begin{align}
    &\PP\left(\xi_{n,p,\kappa} \leq t\right) \leq 2\PP\left(\xi^{(1)}_{n,p,\kappa} \leq t\right)~~\text{and}~~\PP\left(\xi_{n,p,\kappa} \geq t\right) \leq 2\PP\left(\xi^{(1)}_{n,p,\kappa} \geq t\right).\\
    & \PP\left(\xi_{n,p,\kappa}(\bB_p) \leq t\right) \leq 2\PP\left(\xi^{(1)}_{n,p,\kappa}(\bB_p) \leq t\right).
\end{align}
If in addition $\bB_p$ is convex, then
\begin{align}
    \PP\left(\xi_{n,p,\kappa}(\bB_p) \geq t\right) \leq 2\PP\left(\xi^{(1)}_{n,p,\kappa}(\bB_p) \geq t\right).
\end{align}
\end{lemma}

\paragraph{Step 2: Connecting $\xi^{(1)}_{n,p,\kappa}$ with $\tilde{\xi}_{n,p,\kappa}^{(1)}$.}

We then move on to characterizing $\xi^{(1)}_{n,p,\kappa}$ by translating the min-max problem into a single minimization problem. 
For any set $\bB \subseteq \mathbb{R} \times \mathbb{R}_{+}$, 
we define a quantity
\begin{align}
    \tilde{\xi}_{n,p,\kappa}^{(1)}(\bB) = \min_{(\gamma,z)\in\bB}  \frac{1}{\sqrt{p}}\left\{-z\Vert \bg \Vert + \Big\Vert\left(\sqrt{\gamma^2+z^2} \; \bh \odot \by+(\kappa - \gamma \Vert \bmu \Vert)\mathbf{1}_n \right)_+\Big\Vert \right\}.\label{eq:xi_3}
\end{align}
In particular, we set $\tilde{\xi}_{n,p,\kappa}^{(1)} \coloneqq \tilde{\xi}_{n,p,\kappa}^{(1)}(\tilde{\Theta}_{\eta})$, where
$\tilde{\Theta}_{\eta} \coloneqq \{(\gamma,z)\in\RR \times \RR_+:~(1+\eta)^{-2}\gamma^2+z^2 \leq 1\}$.
We have the following relation between $\tilde{\xi}_{n,p,\kappa}^{(1)}$ and $\xi_{n,p,\kappa}^{(1)}$.
\begin{lemma}\label{lem:reduction2}
With the definition above, we have
$$\Big|\left(\tilde{\xi}_{n,p,\kappa}^{(1)}\right)_+ - \xi_{n,p,\kappa}^{(1)}\Big| \overset{p}{\longrightarrow} 0.$$
\end{lemma}

\noindent See Appendix~\ref{sec:proof-lem-reduction2} for the proof.
\medskip

\paragraph{Step 3: Connecting $\tilde{\xi}_{n,p,\kappa}^{(1)}$ with $\xi_{n,p,\kappa}^{(2)}$.}

It turns out that $\tilde{\xi}_{n,p,\kappa}^{(1)}$ can be further simplified for analytical purposes. 
For any set $\bB \subseteq \mathbb{R} \times \mathbb{R}_{+}$, 
we define a new quantity
\begin{align}
     \xi_{n,p,\kappa}^{(2)}(\bB) = \min_{(\gamma,z)\in\bB} \left\{-z + \sqrt{\delta}\sqrt{\EE\left(\sqrt{\gamma^2+z^2}G+\kappa - \gamma \Vert \bmu \Vert\right)_+^2}\right\}.
\end{align}
Similar to before, we simply denote $\xi_{n,p,\kappa}^{(2)} \coloneqq \xi_{n,p,\kappa}^{(2)}(\tilde{\Theta}_{\eta})$.
The two quantities of interest can be related via the uniform law of large numbers as is shown in the following lemma.


\begin{lemma}\label{lem:ulln}
With the definition of $\xi_{n,p,\kappa}^{(2)}$, we have $$\bigg | \xi_{n,p,\kappa}^{(1)} - \left(\xi_{n,p,\kappa}^{(2)}\right)_+ \bigg | \overset{p}{\longrightarrow} 0.$$
\end{lemma}

\noindent See Appendix~\ref{sec:proof-lem-ulln} for the proof.
\medskip

Following the chain of equivalence, we switch our goal to study $\xi_{n,p,\kappa}^{(2)}$.

\paragraph{Step 4: Analyzing $\xi_{n,p,\kappa}^{(2)}$.}

In view of the chain of arguments above, analyzing the margin boils down to analyzing the positivity of $\xi_{n,p,\kappa}^{(2)}$, which is equivalent to the positivity of the following function 
\begin{align}\label{eq:F}
    F_{\delta}(\gamma,z,\kappa) \coloneqq -z^2 + \delta \EE\left(\sqrt{\gamma^2+z^2}G+\kappa - \gamma \Vert \bmu \Vert\right)_+^2.
\end{align}
To do so, the following lemma is useful in which $\bB_v \coloneqq \{\bbeta \in \RR^p: \Vert \bP^{\perp}_{\bmu} \bbeta \Vert \leq v \Vert \bP_{\bmu} \bbeta \Vert\}$  is a convex cone for each choice of $v \geq 0$. 
\begin{lemma}\label{prop:cgmt_general}
Define the event 
\begin{align}
    \calE_{n,p,\kappa,v} \coloneqq \left\{\text{there~exists~} \bbeta \in \Theta \cap \bB_v: y_i \langle \hh_{0,i},\bbeta \rangle \geq \kappa \text{~for~all~} i \leq n\right\}.
\end{align}
\begin{enumerate}
    \item If there exists $(\gamma,z) \in \tilde{\Theta}_{\eta} \cap \{(\gamma,z): z \leq v\gamma\}$ such that $F_{\delta}(\gamma,z,\kappa) < 0$, then
    \begin{align}
        \PP\left(\calE_{n,p,\kappa,v}\right) = 1 - o(1).
    \end{align}
    \item If for any $(\gamma,z) \in \tilde{\Theta}_{\eta} \cap \{(\gamma,z): z \leq v\gamma\}$, we have $F_{\delta}(\gamma,z,\kappa) > 0$, then
    \begin{align}
        \PP\left(\calE^c_{n,p,\kappa,v}\right) = 1 - o(1).
    \end{align}
\end{enumerate}
\end{lemma}

\noindent See Appendix~\ref{sec:proof-prop-cgmt-general} for the proof.
\medskip

\begin{remark}
This proposition still holds when $v = +\infty$. This will be useful for proving Theorem~\ref{thm:CGMT}~(1).
\end{remark}

\paragraph{Step 5: Analyzing $F_{\delta}(\gamma,z,\kappa)$ via an equivalent form.}
Lemma~\ref{prop:cgmt_general} motivates us to focus on the positivity of function $F_{\delta}$~\eqref{eq:F}. 
To this end, we define a helpful function $f_{\delta} : {\mathbb{R}}_{+} \times \mathbb{R}_{+} \times \mathbb{R}_{+} \mapsto \mathbb{R}$
\[
f_{\delta}(u,\kappa,c)=-u^{2}+\delta\mathbb{E}\left[\left(\sqrt{1+u^{2}}G+\kappa\sqrt{u^{2}+c}-\|\mu\|\right)_{+}^{2}\right].
\]
Fix any $c > 0$, we consider the optimization problem
\begin{align}\label{eq:opt_new}
    \sup_{\kappa \geq 0} &\quad \kappa\\\nonumber
    {\rm subject~to} &\quad \inf_{u \geq 0}\; f_{\delta}(u,\kappa,c) \leq 0.
\end{align}
In fact, $c=\frac{1}{(1+\eta)^2}$ is just a reparametrization of $\eta$.
Denote by $\kappa^* = \kappa^*(c)$ the corresponding maximizer and $u^* = u^*(c)$, the smallest minimizer of $\inf_{u \geq 0}\; f_{\delta}(u,\kappa^*(c),c)$. 
The following lemma demonstrates that both are well defined.
\begin{lemma}\label{lem:solution}
For every $c>0$, one has $0 \leq \kappa^*(c) < +\infty$, and $0<u^*(c)<+\infty$.
And they attain the corresponding maximum and minimum.
\end{lemma}

\noindent See Appendix~\ref{sec:proof-g-F-connections} for the proof.  
\medskip

Both $\kappa^*$ and $u^*$ play an essential role in our later developments. 
%

\subsubsection{Proof of Theorem~\ref{thm:CGMT} (1).} 

We plan to invoke Lemma~\ref{prop:cgmt_general} part 2 to prove this claim. 
More precisely, we aim to show that for any fixed $\kappa >0$, when $\delta >\delta^*(\rho)$ with\footnote{To see that $\delta^*(\rho)$ is well-defined, let $\alpha = 1/u$. Then we need to check if $$\inf_{\alpha \geq 0} \EE\left\{\left(\sqrt{1+\alpha^2}G-\Vert \bmu \Vert \alpha\right)_+^2\right\} > 0.$$ 
At $\alpha=0$, $\EE[(G)_+^2]>0$ indicates that the infimum is strictly positive in the neighborhood of $0$; for $\alpha \geq C_{\alpha}$,  $\inf_{\alpha \geq 0} \EE\{\left(\sqrt{1+\alpha^2}G-\Vert \bmu \Vert \alpha\right)_+^2\} \geq \inf_{\alpha \geq 0} \alpha^2 \EE\{\left(G-\Vert \bmu \Vert\right)_+^2\} \geq C_{\alpha}^2\EE\{\left(G-\Vert \bmu \Vert\right)_+^2\} >0$. Thus, $\delta^*(\rho) < +\infty$ is well-defined. As a final remark, when $\Vert \bmu \Vert=0$ (i.e., the pure noise case), we have $\delta^*(0)=2$. As a result, $\delta^*(\rho) \geq 2$ for $\rho \geq 0$.} 
\begin{align}
    \delta^*(\rho) = \left(\inf_{u>0} \frac{1}{u^2}\EE\left\{\left(\sqrt{1+u^2}G - \Vert \bmu \Vert\right)_+^2\right\}\right)^{-1},
\end{align}
one has
$\inf_{(\gamma,z) \in \tilde{\Theta}_{\eta}} F_{\delta}(\gamma, z, \kappa)> 0$. 

\paragraph{Case 1: $\gamma=0$. }In this case, we have 
$
F_{\delta}(0,z,\kappa) =-z^{2}+\delta\mathbb{E}\left[\left(zG+\kappa\right)_{+}^{2}\right].
$
If in addition $z=0$, one has $F_{\delta}(0,0,\kappa)>0$. Otherwise for any $0<z\leq1$,
one has
\[
F_{\delta}(0,z,\kappa)=z^{2}\left(-1+\delta\mathbb{E}\left[\left(G+\frac{\kappa}{z}\right)_{+}^{2}\right]\right).
\]
Note that since $z\leq1$, we have 
\begin{align*}
-1+\delta\mathbb{E}\left[\left(G+\frac{\kappa}{z}\right)_{+}^{2}\right] & \geq-1+\delta\mathbb{E}\left[\left(G+\kappa\right)_{+}^{2}\right]\\
 & >-1+\delta\mathbb{E}\left[\left(G\right)_{+}^{2}\right]\qquad\text{since }\kappa>0\\
 & >-1+\delta^{*}(\rho)\frac{1}{2}\geq0,
\end{align*}
where the last relation uses the fact that $\delta^{*}(\rho) \geq 2$.

\paragraph{Case 2: $\gamma>0$. }
 In this case, we have
 \[
F_{\delta}(\gamma,z,\kappa)=\gamma^{2}\left\{ -u^{2}+\delta\mathbb{E}\left[\left(\sqrt{1+u^{2}}G+\frac{\kappa}{\gamma}-\|\mu\|\right)_{+}^{2}\right]\right\},
\]
where we denote $u=z / \gamma$. It is easy to see that when $u=0$, one has $F_{\delta}(\gamma,z,\kappa) > 0$.
Therefore we focus on the case when $u > 0$.

By the definition of $\delta^{*}(\rho)$, we know that for any $u>0$,
\[
\frac{1}{\delta}<\frac{1}{\delta^{*}(\rho)}\leq\frac{1}{u^{2}}\mathbb{E}\left[\left(\sqrt{1+u^{2}}G-\|\mu\|\right)_{+}^{2}\right].
\]
As a result, for any $u>0$, one has 
\[
-u^{2}+\delta\mathbb{E}\left[\left(\sqrt{1+u^{2}}G-\|\mu\|\right)_{+}^{2}\right]>0,
\]
which further implies 
\[
-u^{2}+\delta\mathbb{E}\left[\left(\sqrt{1+u^{2}}G+\frac{\kappa}{\gamma}-\|\mu\|\right)_{+}^{2}\right]>0.
\]
This proves that $F_{\delta}(\gamma,z,\kappa)>0$ for all $\gamma,z$.

\subsubsection{Proof of Theorem~\ref{thm:CGMT} (2).} 
In view of Part 1 of Lemma~\ref{prop:cgmt_general}, we only need to show that
for some $\kappa>0$, there exist some $(\gamma,z)$ such that $F_{\delta}(\gamma,z,\kappa)<0$.
Since $\delta<\delta^{*}(\rho)$, we have 
\[
\frac{1}{\delta}>\frac{1}{u^{2}}\mathbb{E}\left[\left(\sqrt{1+u^{2}}G-\|\mu\|\right)_{+}^{2}\right]
\]
for some $u>0$. As a result, we have
\[
-u^{2}+\delta\mathbb{E}\left[\left(\sqrt{1+u^{2}}G-\|\mu\|\right)_{+}^{2}\right]<0.
\]
Due to continuity, we know that there exists some $L>0$, such that
\[
-u^{2}+\delta\mathbb{E}\left[\left(\sqrt{1+u^{2}}G+\kappa\sqrt{u^{2}+c}-\|\mu\|\right)_{+}^{2}\right]<0,\qquad\text{for all }\kappa<L.
\]
Let $\gamma=\frac{1}{\sqrt{u^{2}+c}}$, and $z=u\cdot\gamma$. We
then have
\[
F_{\delta}(\gamma,z,\kappa)=\gamma^{2}\left\{ -u^{2}+\delta\mathbb{E}\left[\left(\sqrt{1+u^{2}}G+\frac{L}{2}\sqrt{u^{2}+c}-\|\mu\|\right)_{+}^{2}\right]\right\} <0.
\]
This finishes the proof.

We then turn to the claim regarding the margin. 
Recall that we work in the regime where $\delta<\delta^{*}(\rho)$.
Fix any $\varepsilon>0$. We aim to prove that $\mathcal{E}_{n,p,\kappa^{*}+\varepsilon}$ does
not hold and that $\calE_{n,p,\kappa^*-\veps}$ holds. 

\paragraph{Step 1: Quantifying $\mathcal{E}_{n,p,\kappa^{*}+\varepsilon}$.}
Similar to the proof of Theorem~\ref{thm:CGMT} (1), it suffices to prove that  for all $\gamma,z$, 
$
F_{\delta}(\gamma,z,\kappa^{*}+\varepsilon)>0.
$

\paragraph{Case 1: $\gamma=0$. }It is simple to see that 
$F_{\delta}(0,0,\kappa^{*}+\varepsilon)>0$. Therefore we focus on the case for any $0<z\leq1$.
Using similar arguments as in Theorem~\ref{thm:CGMT} (1), we have
\[
F_{\delta}(0,z,\kappa^{*}+\varepsilon)=z^{2}\left(-1+\delta\mathbb{E}\left[\left(G+\frac{\kappa^{*}+\varepsilon}{z}\right)_{+}^{2}\right]\right).
\]
Note that since $z\leq1$, we have 
\begin{align*}
-1+\delta\mathbb{E}\left[\left(G+\frac{\kappa^{*}+\varepsilon}{z}\right)_{+}^{2}\right] & \geq-1+\delta\mathbb{E}\left[\left(G+\kappa^{*}\right)_{+}^{2}\right].
\end{align*}
When $2 \leq  \delta < \delta^*(\rho)$, we have $\kappa^* \geq 0$, and hence $F_{\delta}(0,z,\kappa^{*}+\varepsilon) > 0$.
When $\delta < 2$, since $\kappa^* \geq \kappa(\delta)$, where $\kappa(\delta)$ is the unique solution to
\begin{align}
    \frac{1}{\delta} = \EE\left[\left(G+\kappa\right)_+^2\right],
\end{align}
we again have  $F_{\delta}(0,z,\kappa^{*}+\varepsilon) > 0$.

\paragraph{Case 2: $\gamma>0$. }
 In this case, we have
 \[
F_{\delta}(\gamma,z,\kappa^{*}+\varepsilon)=\gamma^{2}\left\{ -u^{2}+\delta\mathbb{E}\left[\left(\sqrt{1+u^{2}}G+\frac{\kappa^{*}+\varepsilon}{\gamma}-\|\mu\|\right)_{+}^{2}\right]\right\},
\]
where we denote $u=z / \gamma$. By the definition of $\kappa^{*}$, we know
that for all $u\geq 0$, 
\[
f_{\delta}(u,\kappa^{*}+\varepsilon,c)=-u^{2}+\delta\mathbb{E}\left[\left(\sqrt{1+u^{2}}G+(\kappa^{*}+\varepsilon)\sqrt{u^{2}+c}-\|\mu\|\right)_{+}^{2}\right]>0.
\]
For all $0 < \gamma \leq \frac{1}{\sqrt{u^2 + c}}$, we further have 
\[
-u^{2}+\delta\mathbb{E}\left[\left(\sqrt{1+u^{2}}G+\frac{\kappa^{*}+\varepsilon}{\gamma}-\|\mu\|\right)_{+}^{2}\right]>-u^{2}+\delta\mathbb{E}\left[\left(\sqrt{1+u^{2}}G+(\kappa^{*}+\varepsilon)\sqrt{u^{2}+c}-\|\mu\|\right)_{+}^{2}\right]>0.
\]
This finishes the claim.

\paragraph{Step 2: Quantifying $\calE_{n,p,\kappa^*-\veps}$.}
Similarly, for $\kappa^{*}-\varepsilon$, since $f_{\delta}(u,\kappa,c)$ is increasing in $\kappa$ and $\inf_{u \geq 0} f_{\delta}(u,\kappa^*,c) = 0$, there exists
some $u>0$, such that 
\[
-u^{2}+\delta\mathbb{E}\left[\left(\sqrt{1+u^{2}}G+(\kappa^{*}-\varepsilon)\sqrt{u^{2}+c}-\|\mu\|\right)_{+}^{2}\right]<0.
\]
Going through the same argument as in part 1 completes the proof of this part. 

Combine the arguments above, for any $\veps > 0$, we see that with probability approaching 1, $\kappa^{*}-\varepsilon<\hat{\kappa}<\kappa^{*}+\varepsilon$, which proves that $\hat{\kappa} \overset{p}{\rightarrow} \kappa^*$.

\subsubsection{Proof of Theorem~\ref{thm:CGMT} (3).} 
We start with Lemma~\ref{lem:part3} showing that $u^* = u^*(c) = u^*(\rho,\eta)$, defined previously, is indeed the limit of $\hat{u} := \frac{\Vert \bP_{\bmu}^{\perp} \hat{\bbeta} \Vert}{\Vert \bP_{\bmu} \hat{\bbeta} \Vert}$. 

\begin{lemma}\label{lem:part3}
Assume that $\delta < \delta^*(\rho)$. 
Fix any $c > 0$ 
and any $\veps > 0$, 
\begin{align}
    \calE_{n,p,\kappa^*,u^*+\veps}\qquad \text{and} \qquad \calE^c_{n,p,\kappa^*,u^*-\veps} \qquad \text{hold}.
\end{align}
As a result, we have
\begin{align}
    \hat{u} := \frac{\Vert \bP_{\bmu}^{\perp} \hat{\bbeta} \Vert}{\Vert \bP_{\bmu} \hat{\bbeta} \Vert} \overset{p}{\rightarrow} u^*.
\end{align}
\end{lemma}

\noindent See Appendix~\ref{sec:proof-part3} for the proof.
\medskip

Based on the lemma, recall the definition $\rho = \norm{\bmu}^2$, $u^*(\rho,\eta) = u^*(c)$, and
\begin{align}
    \texttt{Err}(\hat{\bbeta};\eta) = \PP\left(y^\prime \langle \hh_0^\prime, \hat{\bbeta} \rangle < 0\right) = \Phi\left(-\langle \bmu, \frac{\hat{\bbeta}}{\Vert \hat{\bbeta} \Vert}\rangle\right),\qquad \texttt{Err}^*(\eta) = \Phi\left(-\Vert \bmu \Vert \sqrt{\frac{1}{1+u^{*2}(\rho,\eta)}}\right).
\end{align}
From Lemma~\ref{lem:part3}, we have
\begin{align}
    \Bigg|\langle \bmu, \frac{\hat{\bbeta}}{\Vert \hat{\bbeta} \Vert}\rangle  -  \Vert \bmu \Vert \sqrt{\frac{1}{1+u^{*2}(\rho,\eta)}}\Bigg| = o_{\PP}(1),
\end{align}
which implies that
\begin{align}
    \texttt{Err}(\hat{\bbeta};\eta) \overset{p}{\rightarrow}\texttt{Err}^*(\eta) = \Phi\left(-\Vert \bmu \Vert \sqrt{\frac{1}{1+u^{*2}(\rho,\eta)}}\right).
\end{align}
Since $\texttt{Err}^*(\eta)$ is monotonically decreasing in $u^*(\rho,\eta)$, then asymptotically, the classification error of the original max-margin solution is monotonically decreasing in $u^*(\rho,\eta)$. 
It remains to show that $u^*(\rho,\eta)$ is monotonically decreasing in $\eta$, which is equivalent to showing that $u^*(c)$ is monotonically increasing in $c$. This is provided in the following lemma. 
\begin{lemma}\label{lem:monot} 
With the definition of $(u^*,\kappa^*,c)$, for any $c > 0$, 
\begin{enumerate}
    \item $u^* = u^*(c)$ is a continuous function in $c$;
    \item $u^* = u^*(c)$ is monotonically increasing in $c$.
\end{enumerate}
\end{lemma}
\noindent See Appendix~\ref{sec:proof-monot} for the proof.
\medskip

These finish the proof of Theorem~\ref{thm:CGMT}.

\subsubsection{Proof of Lemma~\ref{lem:reduction2}}\label{sec:proof-lem-reduction2}
Recall that
\begin{align}
    \xi^{(1)}_{n,p,\kappa}(\bB_p) = & \min_{\bbeta\in\bB_p} \max_{\blambda:\Vert \blambda \Vert \leq 1, \by \odot \blambda \geq 0}\frac{1}{\sqrt{p}}\left\{\Vert \blambda \Vert \gggg^\top \bbeta + \Vert \bbeta \Vert \bh^\top \blambda + \kappa \blambda^\top \by - \langle \bmu, \bbeta \rangle \blambda^\top \by\right\}.
\end{align}
Denote 
\[
\ell(r_{\blambda}, \blambda_{0}) \coloneqq \frac{1}{\sqrt{p}}\left\{\Vert \blambda \Vert \gggg^\top \bbeta + \Vert \bbeta \Vert \bh^\top \blambda + \kappa \blambda^\top \by - \langle \bmu, \bbeta \rangle \blambda^\top \by\right\}.
\]

\paragraph{Step 1: inner maximization. }
We focus on the inner maximization problem, which is equivalent to 
\begin{align}
\max_{r_{\blambda} \in [0,1], \Vert \blambda_{0} \Vert=1,\by \odot \blambda_{0} \geq 0} \quad \ell(r_{\blambda}, \blambda_{0}) = \frac{1}{\sqrt{p}} r_{\blambda}\bg^\top \bbeta + \frac{1}{\sqrt{p}} r_{\blambda} \left[\Vert \bbeta \Vert \bh \odot \by + (\kappa - \gamma \Vert \bmu \Vert) \mathbf{1}_n\right]^\top (\by \odot \blambda_{0}),
\end{align}
where $\gamma = \langle \bmu, \bbeta \rangle / \norm{\bmu}$.

When $r_{\blambda} = 0$, we have $\ell(r_{\blambda}, \blambda_{0})= 0$.  When $r_{\blambda} > 0$, since
$\Vert \by \odot \blambda_0 \Vert = \Vert \blambda_0 \Vert = 1$ and $\by \odot \blambda_0 \geq 0$, we have for any vector $\bnu$ with at least one nonnegative coordinate,
\begin{align}\label{eq:non-negative}
    \max_{\Vert \blambda_{0} \Vert=1,\by \odot \blambda_{0} \geq 0} \bnu^\top \left(\by \odot \blambda_{0}\right) =  \Vert \bnu_+ \Vert.
\end{align}
In fact, letting  $\bnu =\Vert \bbeta \Vert \bh \odot \by + (\kappa - \gamma \Vert \bmu \Vert) \mathbf{1}_n$, we see that $\bnu$ has i.i.d.~coordinates, and for each coordinate, the probability for it to be negative is given by
\begin{align}
    \PP\left(\Vert \bbeta \Vert g + \kappa - \gamma \Vert \bmu \Vert < 0\right) = \Phi\left(-\frac{\kappa}{\Vert \bbeta \Vert} + \frac{\gamma \Vert \bmu \Vert}{\Vert \bbeta \Vert}\right) \leq \Phi\left(\frac{\gamma \Vert \bmu \Vert}{\Vert \bbeta \Vert}\right) \leq \Phi(\|\bmu\|).
\end{align}
As a result, with probability at least $1-\Phi(\|\bmu\|)^n$, $\bnu$ has at least one nonnegative entry, and hence~\eqref{eq:non-negative} holds.

Combine the two cases to arrive at the conclusion that: with probability at least $1-\Phi(\|\bmu\|)^n$, 
\begin{align}
\max_{r_{\blambda} \in [0,1], \Vert \blambda_{0} \Vert=1,\by \odot \blambda_{0} \geq 0} \quad \ell(r_{\blambda}, \blambda_{0})  = \left(\frac{1}{\sqrt{p}}\left\{ \gggg^\top \bbeta + \Big\Vert \left(\Vert \bbeta \Vert \bh \odot \by + (\kappa - \gamma \Vert \bmu \Vert) \mathbf{1}_n\right)_+ \Big\Vert\right\}\right)_+,
\end{align}
where we recall that $\gamma = \langle \bmu, \bbeta \rangle / \norm{\bmu}$.

\paragraph{Step 2: outer minimization. }
Now we turn to the outer minimization over $\bbeta \in \Theta = \{\bbeta \in \RR^p: \bbeta^\top \bOmega \bbeta \leq 1\}$.
Using a reparametrization $\bbeta=\gamma\bmu_0 + \bbeta_{\perp}$, where $\bbeta_{\perp}=r\cdot \ww$ with $z=\Vert\bbeta_{\perp}\Vert$ and $\ww \perp \bmu$, $\Vert \ww \Vert=1$,  we see that with high probability $\xi^{(1)}_{n,p,\kappa}$ is equal to
\begin{align}
    \bar{\xi}_{n,p,\kappa}^{(1)} =\frac{1}{\sqrt{p}}\min_{(\gamma, z)\in \tilde{\Theta}_{\eta}} \min_{\ww \perp \bmu, \Vert \ww\Vert=1}\left\{\gamma g^\top\bmu_0+z \gggg^\top \ww + \Big\Vert\left(\sqrt{\gamma^2+z^2} \bh \odot \by+(\kappa - \gamma \Vert \bmu \Vert)\mathbf{1}_n \right)_+\Big\Vert \right\}_+.
\end{align}
Here we recall that $ \tilde{\Theta}_{\eta} = \{(\gamma,z)\in\RR \times \RR_+:~(1+\eta)^{-2}\gamma^2+z^2 \leq 1\}$. 

We can solve the inner minimization, that is 
\begin{align*}
 &\min_{\ww \perp \bmu, \Vert \ww\Vert=1}\left\{\gamma g^\top\bmu_0+z \gggg^\top \ww + \Big\Vert\left(\sqrt{\gamma^2+z^2} \bh \odot \by+(\kappa - \gamma \Vert \bmu \Vert)\mathbf{1}_n \right)_+\Big\Vert \right\}_+ \\
 &\quad =  \left\{\gamma g^\top\bmu_0 - z \| \bP_{\bmu}^\perp \gggg \| + \Big\Vert\left(\sqrt{\gamma^2+z^2} \bh \odot \by+(\kappa - \gamma \Vert \bmu \Vert)\mathbf{1}_n \right)_+\Big\Vert \right\}_+
\end{align*}
Note that $\gamma g^\top\bmu_0 / \sqrt{p} = o_{\PP}(1)$. This together with the boundedness of  $\tilde{\Theta}_{\eta}$ implies that 
\[
\bar{\xi}_{n,p,\kappa}^{(1)} - \min_{(\gamma,z)\in\tilde{\Theta}_{\eta}}  \frac{1}{\sqrt{p}}\left\{-r\Vert \bg \Vert + \Big\Vert\left(\sqrt{\gamma^2+r^2} \; \bh \odot \by+(\kappa - \gamma \Vert \bmu \Vert)\mathbf{1}_n \right)_+\Big\Vert \right\}_+ = o_{\PP}(1).
\]
Note that $\min_{(\gamma,r)\in\tilde{\Theta}_{\eta}}  \frac{1}{\sqrt{p}}\left\{-z\Vert \bg \Vert + \Big\Vert\left(\sqrt{\gamma^2+z^2} \; \bh \odot \by+(\kappa - \gamma \Vert \bmu \Vert)\mathbf{1}_n \right)_+\Big\Vert \right\}_+ = \left(\tilde{\xi}_{n,p,\kappa}^{(1)}\right)_+$. 

Combine all the pieces together to complete the proof. 
%


\subsubsection{Proof of Lemma~\ref{lem:ulln}}\label{sec:proof-lem-ulln}
Define two functions
\begin{align}
    & \phi^{(1)}_{n,\bmu,\delta,\kappa}(\gamma,z) = \left[\EE\left(\sqrt{\gamma^2+z^2} G+\kappa - \gamma \Vert \bmu \Vert\right)_+^2\right]^{1/2},\\
    & \phi^{(2)}_{n,\bmu,\delta,\kappa}(\gamma,z;\bh) = \frac{1}{\sqrt{n}} \Big\Vert\left(\sqrt{\gamma^2+z^2} \tilde{\bh} +(\kappa - \gamma \Vert \bmu \Vert)\mathbf{1}_n \right)_+\Big\Vert,
\end{align}
where $\tilde \hh \sim \calN(\bzero,\bI_n)$ and $G \sim \calN(0,1)$. By the uniform law of large numbers \citep{jennrich1969asymptotic} (see also~\citep{montanari2019generalization} (Lemma 6.2)), we have
\begin{align}
    \sup_{(\gamma,z)\in\tilde{\Theta}_{\eta}} \Bigg| \phi^{(1)}_{n,\bmu,\delta,\kappa}(\gamma,z) - \phi^{(2)}_{n,\bmu,\delta,\kappa}(\gamma,z;\hh) \Bigg| = o_{\PP}(1),
\end{align}

In addition, since $p^{-1/2}\Vert \bg \Vert = 1 + o_{\PP}(1)$, we have
\begin{align}
    \bigg | \tilde{\xi}_{n,p,\kappa}^{(1)} - \xi_{n,p,\kappa}^{(2)} \bigg | \overset{p}{\longrightarrow} 0,
\end{align}
which implies the desired claim.

\subsubsection{Proof of Lemma~\ref{prop:cgmt_general}}\label{sec:proof-prop-cgmt-general}
\paragraph{Step 1: establish equivalence.} We notice that for the convex set $\Theta \cap \bB_{v}$,
\begin{align}
    \calE_{n,p,\kappa,v}  = \calE_{n,p,\kappa}(\Theta \cap \bB_{v}).
\end{align}
Denote the quantity
\begin{align}
    \xi_{n,p,\kappa,v} &= \xi_{n,p,\kappa}(\bB_v \cap \Theta) = \min_{\bbeta \in \bB_v \cap \Theta} \max_{\Vert \blambda \Vert \leq 1, \by \odot \blambda \geq 0} \frac{1}{\sqrt{p}} \blambda^\top (\kappa \by - \bH_0 \bbeta),
\end{align}
for which we have
\begin{align}
    \{\xi_{n,p,\kappa,v} > 0\} \Longleftrightarrow \calE_{n,p,\kappa,v}^c~~{\rm and}~~\{\xi_{n,p,\kappa,v} = 0\} \Longleftrightarrow \calE_{n,p,\kappa,v}.
\end{align}
Following the equivalence established in Lemma~\ref{lem:cgmt1}, \ref{lem:reduction2} and \ref{lem:ulln}, denote
\begin{align}
    \xi^{(1)}_{n,p,\kappa,v} &= \xi^{(1)}_{n,p,\kappa}(\bB_v \cap \Theta)\\
    \xi^{(2)}_{n,p,\kappa,v} &= \xi^{(2)}_{n,p,\kappa}(\tilde{\Theta}_{\eta} \cap \{(\gamma,z): z \leq v\gamma\}).
\end{align}
Then, it is directly implied by Lemma~\ref{lem:cgmt1} that for any $t \in \RR$,
\begin{align}\label{eq:equiv1}
    \PP(\xi_{n,p,\kappa,v} \leq t) \leq 2\PP(\xi^{(1)}_{n,p,\kappa,v} \leq t), \qquad \PP(\xi_{n,p,\kappa,v} \geq t) \leq 2\PP(\xi^{(1)}_{n,p,\kappa,v} \geq t).
\end{align}
By Lemma~\ref{lem:reduction2} and \ref{lem:ulln}, we further have
\begin{align}\label{eq:equiv2}
    \bigg | \xi_{n,p,\kappa,v}^{(1)} - \left(\xi_{n,p,\kappa,v}^{(2)}\right)_+ \bigg | \overset{p}{\longrightarrow} 0.
\end{align}

\paragraph{Step 2: analyze $F_{\delta}(\gamma,z,\kappa)$.} We consider the following two cases.
\begin{enumerate}
  \item By the definition of $F_{\delta}(\gamma,z,\kappa)$, if there exists $(\gamma,z) \in \tilde{\Theta} \cap \{(\gamma,z): z \leq v\gamma\}$ such that $F_{\delta}(\gamma,z,\kappa) < 0$, then 
  \begin{align}
      \xi_{n,p,\kappa,v}^{(2)} = \min_{(\gamma,z) \in \tilde{\Theta} \cap \{(\gamma,z): z \leq v\gamma\}} \left\{-z + \sqrt{\delta} \sqrt{\EE\left(\sqrt{\gamma^2 + z^2} G + \kappa - \gamma\Vert \bmu\Vert\right)_+^2}\right\} < 0,
  \end{align}
  which, according to \eqref{eq:equiv2}, implies that $\PP(\xi^{(1)}_{n,p,\kappa,v} > 0) = o(1)$ and for any $\veps > 0$,
  \begin{align*}
  \PP(\xi^{(1)}_{n,p,\kappa,v} \geq \veps) = o(1).
  \end{align*}
  As $\PP(\xi_{n,p,\kappa,v} > t)$ is right-continuous in $t$, then
  \begin{align}\label{eq:xi1_negative}
  \PP(\xi_{n,p,\kappa,v} > 0) = \lim_{t \rightarrow 0+}\PP(\xi_{n,p,\kappa,v} > t).
  \end{align}
  Combining \eqref{eq:equiv1} with the definition of $\calE^c_{n,p,\kappa,v}$, we have
  \begin{align*}
      \PP(\calE^c_{n,p,\kappa,v}) & = \PP(\xi_{n,p,\kappa,v} > 0) = \lim_{\veps \rightarrow 0+}\PP(\xi_{n,p,\kappa,v} > \veps)\leq \lim_{\veps \rightarrow 0+}\PP(\xi_{n,p,\kappa,v} \geq \veps) \leq 2 \lim_{\veps \rightarrow 0+}\PP(\xi^{(1)}_{n,p,\kappa,v} \geq \veps).
  \end{align*}
  In addition, we verify the validity of changing the order of $\lim_{n\rightarrow +\infty, n/p\rightarrow \delta}$ and $\lim_{\veps \rightarrow 0+}$. 
  Let $q_{n,p}(t) = \PP(\xi^{(1)}_{n,p,\kappa,v} \geq t)$ and $\tilde q_{n,p}(t) = \lim_{s \rightarrow t+} q_{n,p}(s) = \PP(\xi^{(1)}_{n,p,\kappa,v} > t)$, which satisfies that $q_{n,p}(t)$ is left-continuous, $\tilde q_{n,p}(t)$ is right-continuous, and for any $t > 0$
  \[
  \lim_{n\rightarrow +\infty, n/p\rightarrow \delta} q_{n,p}(t) = \lim_{n\rightarrow +\infty, n/p\rightarrow \delta} \tilde q_{n,p}(t) = 0.
  \]
  Further, we have 
  \[
  \lim_{n\rightarrow +\infty, n/p\rightarrow \delta} \tilde q_{n,p}(0) = \lim_{n\rightarrow +\infty, n/p\rightarrow \delta} \PP(\xi^{(1)}_{n,p,\kappa,v} > 0) = 0.
  \]
  By the fact that $\tilde q_{n,p}(t)$ in non-increasing in $t$, we have
  \[
  \lim_{n\rightarrow +\infty, n/p\rightarrow \delta} \sup_{t \geq 0} |\tilde q_{n,p}(t) - 0|,
  \]
  which indicates that $\tilde q_{n,p}(t)$ uniformly converges to $0$. Therefore, by Moore-Osgood theorem,
  \[
  \lim_{n\rightarrow +\infty, n/p\rightarrow \delta}\lim_{\veps \rightarrow 0+} q_{n,p}(\veps) = \lim_{n\rightarrow +\infty, n/p\rightarrow \delta}\lim_{\veps \rightarrow 0}\tilde q_{n,p}(\veps) = \lim_{\veps \rightarrow 0} \lim_{n\rightarrow +\infty, n/p\rightarrow \delta} \tilde q_{n,p}(\veps) = 0,\label{eq:change_limit}
  \]
  thus, we have
  \begin{align*}
  \lim_{n\rightarrow +\infty, n/p\rightarrow \delta}\lim_{\veps \rightarrow 0+}\PP(\xi^{(1)}_{n,p,\kappa,v} \geq \veps) = 0 \qquad \Longrightarrow \qquad \lim_{n\rightarrow +\infty, n/p\rightarrow \delta} \PP(\calE^c_{n,p,\kappa,v}) = 0.
  \end{align*}
  \item If, instead, for any $(\gamma,z) \in \tilde{\Theta} \cap \{(\gamma,z): z \leq v\gamma\}$, we have $F_{\delta}(\gamma,z,\kappa) > 0$. Since the set $(\gamma,z) \in \tilde{\Theta} \cap \{(\gamma,z): z \leq v\gamma\}$ is convex and compact, then
  \begin{align}
      \xi_{n,p,\kappa,v}^{(2)} = \min_{(\gamma,z) \in \tilde{\Theta} \cap \{(\gamma,z): z \leq v\gamma\}} \left\{-z + \sqrt{\delta} \sqrt{\EE\left(\sqrt{\gamma^2 + z^2} G + \kappa - \gamma\Vert \bmu\Vert\right)_+^2}\right\} > 0,
  \end{align}
  which implies that
  \begin{align}
      \PP(\calE_{n,p,\kappa,v}) = \PP(\xi_{n,p,\kappa,v} = 0) \leq 2\PP(\xi^{(1)}_{n,p,\kappa,v} \leq 0) = o(1).
  \end{align}
\end{enumerate}


\subsubsection{Proof of Lemma~\ref{lem:solution}}\label{sec:proof-g-F-connections}
In this section, we will first show the relationship between $F_{\delta}(\gamma,z,\kappa)$ and $f_{\delta}(u,\kappa,c)$. In the next step, we show the boundedness of $\kappa^*(c)$ and $u^*(c)$.
\paragraph{Step 1: Relationship between $F_{\delta}(\gamma,z,\kappa)$ and $f_{\delta}(u,\kappa,c)$.}
Consider the properties of $F_{\delta}(\gamma,z,\kappa)$. Since $F_{\delta}(\gamma,z,\kappa) < F_{\delta}(-\gamma,z,\kappa) $ for $\gamma > 0$, in the following we only consider $\gamma \in \RR_+$.
Denote $u = z/\gamma$, then $F_{\delta}(\gamma,z,\kappa)$ can be written as a function of $(\gamma,u,\kappa)$:
\begin{align}
    \tilde{F}_{\delta}(\gamma,u,\kappa) = \gamma^2 \left(-u^2+\delta \EE\left\{\left(\sqrt{1+u^2}G+\frac{\kappa}{\gamma}-\Vert \bmu \Vert\right)_+^2\right\}\right),
\end{align}
where $(\gamma,u)\in\tilde{\Omega} = \{(\gamma,u)\in\RR_+\times\bar{\RR}_+: \gamma^2\left(u^2+c\right)\leq 1\}$ with $c = (1+\eta)^{-2}$ and $\bar{\RR}_+ = \RR_+ \cup \{\infty\}$. We have the following problem
\begin{align}
    \begin{cases}
    \sup &\kappa\\
    {\rm subject~to} &\exists (\gamma,u)\in\tilde{\Omega}, ~\tilde{F}_{\delta}(\gamma,u,\kappa) < 0.
    \end{cases}\label{eq:opt_Ftilde}
\end{align}
Note that when $\kappa$ approaches the supremum, for any $(\gamma^\prime,u)\notin\partial \tilde{\Omega}$ such that $\tilde{F}_{\delta}(\gamma^\prime,u,\kappa) < 0$,
there exists
$(\gamma,u)\in\partial \tilde{\Omega}$ with $\gamma > \gamma^\prime$ at which $\tilde{F}_{\delta}(\gamma,u,\kappa) < \tilde{F}_{\delta}(\gamma^\prime,u,\kappa)$. The reason is that when $\gamma^2 > (\gamma^{\prime})^2$ and $\gamma^{-2} \tilde{F}_{\delta}(\gamma,u,\kappa) < (\gamma^\prime)^{-2}\tilde{F}_{\delta}(\gamma^\prime,u,\kappa) < 0$, we have $\tilde{F}_{\delta}(\gamma,u,\kappa) < \tilde{F}_{\delta}(\gamma^\prime,u,\kappa)$.
Then, the above shows that the minimizer $(\gamma, u)$ must be located on the boundary $\partial \tilde{\Omega}$, that is $\gamma^{-2} = u^2+c$. If we also view $c$ as a variable, recall the definition $f_{\delta}(u,\kappa,c)$ with
\begin{align}
    g_{\delta}(u,\kappa,c) &=  \tilde{F}_{\delta}((u^2+c)^{-1/2},u,\kappa)\nonumber\\
    & = \frac{1}{u^2+c} \cdot \left\{-u^2 + \delta \EE\left\{\left(\sqrt{1+u^2}G+\kappa\sqrt{u^2+c} - \Vert \bmu \Vert\right)_+^2\right\}\right\}
\end{align}
and the optimization problem in \eqref{eq:opt_new}
\begin{align}
    \max &\quad \kappa\\\nonumber
    {\rm subject~to} &\quad \inf_{u \geq 0}\; f_{\delta}(u,\kappa,c) \leq 0,
\end{align}
where $\kappa^* = \kappa^*(c)$ is the corresponding maximizer and $u^* = u^*(c)$, short for $u^*(\rho,\eta)$, is the smallest minimizer of $\min_{u \geq 0}\; f_{\delta}(u,\kappa,c)$. 

We have the following lemma.
\begin{lemma}\label{lem:equi_f_F}
If $\tilde{\kappa}$ is the maximizer to \eqref{eq:opt_Ftilde} and $(\tilde{u}, \tilde{\gamma})$ is the minimizer  to $F_{\delta}(\gamma, u\gamma, \tilde{\kappa})$ such that $\tilde{u}$ is the smallest minimizer for $u$, then 
\begin{align}
    \tilde{\kappa} = \kappa^*,~~~~\tilde{u} = u^*,~~~~\tilde{\gamma} = \frac{1}{\sqrt{u^{*2}+c}}.
\end{align}
\end{lemma}

\begin{proof}
Fix $c\in \RR_+$ and for any $\kappa \in \RR_+$, 
we obtain from the reasoning above that for any $(\gamma,u)$ such that $\tilde{F}_{\delta}(\gamma,u,\kappa) \leq 0$, 
\[
\tilde{F}_{\delta}(\gamma,u,\kappa) \geq \tilde{F}_{\delta}((u^2+c)^{-1/2},u,\kappa) = g_{\delta}(u,\kappa,c).
\]
As a result, when $\inf_{(\gamma,u) \in \tilde{\Omega}} \tilde{F}_{\delta}(\gamma,u,\kappa) \leq 0$,
\[
\inf_{(\gamma,u) \in \tilde{\Omega}} \tilde{F}_{\delta}(\gamma,u,\kappa) \geq \inf_{u \geq 0} g_{\delta}(u,\kappa,c).
\]
In addition, for the other direction,
\[
\inf_{(\gamma,u) \in \tilde{\Omega}} \tilde{F}_{\delta}(\gamma,u,\kappa) \leq \inf_{(\gamma,u) \in \tilde{\Omega}, \gamma = (u^2+c)^{-2}} \tilde{F}_{\delta}(\gamma,u,\kappa) = \inf_{u \geq 0} g_{\delta}(u,\kappa,c).
\]
Combining the two inequalities, we have
\begin{align}
    \inf_{(\gamma,u) \in \tilde{\Omega}} \tilde{F}_{\delta}(\gamma,u,\kappa) = \inf_{u \geq 0} g_{\delta}(u,\kappa,c),
\end{align}
which implies that the following feasible sets are the same
\begin{align}
    \left\{\kappa > 0: \inf_{(\gamma,u) \in \tilde{\Omega}} \tilde{F}_{\delta}(\gamma,u,\kappa) < 0\right\} = \left\{\kappa > 0: \inf_{u \geq 0} g_{\delta}(u,\kappa,c) < 0 \right\}.
\end{align}
Then, it is equivalent to show that \eqref{eq:opt_new} has the same solution with the following problem
\begin{align}\label{eq:opt1}
    \begin{cases}
    \sup &\kappa\\
    {\rm subject~to} &\inf_{u \geq 0} g_{\delta}(u,\kappa,c) <0.
    \end{cases}
\end{align}
We denote $\tilde{u}$ and $\tilde{\kappa}$ as the solution to \eqref{eq:opt} such that $\tilde{u}$ is the smallest one if the minimizer is not unique.
Since $\bar{\RR}_+$ is a closed set, then $\inf_{u \geq 0} g_{\delta}(u,\kappa,c) = \min_{u \geq 0} g_{\delta}(u,\kappa,c)$. We further have
\begin{align}
    \left\{\kappa > 0: \inf_{u \geq 0} g_{\delta}(u,\kappa,c) < 0 \right\} \subseteq \left\{\kappa > 0: \inf_{u \geq 0} g_{\delta}(u,\kappa,c) \leq 0 \right\},
\end{align}
indicating that 
\begin{align}\label{eq:sup_upper}
    \kappa^* = \sup\left\{\kappa > 0: \inf_{u \geq 0} g_{\delta}(u,\kappa,c) \leq 0 \right\} \geq \sup\left\{\kappa > 0: \inf_{u \geq 0} g_{\delta}(u,\kappa,c) < 0 \right\} = \tilde{\kappa}.
\end{align}
Recall that the solution $(u^*,\kappa^*)$ for \eqref{eq:opt_new} has the property that $u^*$ is the smallest minimizer. Since $g_{\delta}(u,\kappa,c)$ is strictly increasing in $\kappa$, then for any $\veps > 0$, there exists $\kappa_{\veps} \geq \kappa^* - \veps$ such that 
\begin{align}
    \inf_{u \geq 0} g_{\delta}(u,\kappa_{\veps},c) < \inf_{u \geq 0} g_{\delta}(u,\kappa^*,c) \leq 0,
\end{align}
which means
\begin{align}\label{eq:feasible_kap}
    \kappa_{\veps} \in \left\{\kappa > 0: \inf_{u \geq 0} g_{\delta}(u,\kappa,c) < 0 \right\}.
\end{align}
If $\tilde \kappa < \kappa^*$, let $\epsilon = (\kappa^* - \tilde \kappa)/2$. Then, $\kappa_{\veps} > \tilde \kappa$ and \eqref{eq:feasible_kap} shows that there is a feasible $\kappa$ for \label{eq:opt} that is larger than $\tilde \kappa$. This draws the contradiction and we have
\begin{align}
    \tilde{\kappa} = \sup\left\{\kappa > 0: \inf_{u \geq 0} g_{\delta}(u,\kappa,c) < 0 \right\} = \kappa^*.
\end{align}

By definition, $(\tilde{u},\tilde{\gamma})$ is the minimizer to $F_{\delta}(\gamma,u\gamma,\kappa^*)$ where $\tilde{u}$ is the smallest one if the minimizer is not unique, then by the reasoning above that the minimizer locates on the boundary of $\tilde{\Theta}_{\eta}$, it is equivalent that $\tilde{u}$ is the minimizer to $g_{\delta}(u,\kappa^*,c)$. By the uniqueness of the smallest minimizer to $g_{\delta}(u,\kappa^*,c)$, we have $\tilde{u} = u^*$. By the boundary equation, we have $\tilde{\gamma} = 1/\sqrt{u^{*2} + c}$.
\end{proof}


\paragraph{Step 2: Proving $\kappa^*(c) < +\infty$.}
It suffices to provide some $\bar{\kappa} < +\infty$ such that 
$\inf_{u \geq 0} f_{\delta}(u,\bar{\kappa},c) > 0.$ Recall that
\begin{align*}
f_{\delta}(u,\kappa,c) &= -u^2 + \delta\EE\left[\left(\sqrt{1+u^2}G + \kappa \sqrt{u^2+c} - \|\bmu\|\right)_+^2\right]\\ 
&= u^2\left\{-1 + \delta \frac{1+u^2}{u^2} \EE\left[\left(G + \kappa \sqrt{\frac{u^2+c}{u^2+1}} - \sqrt{\frac{1}{1+u^2}} \|\bmu\|\right)_+^2\right]\right\}\\
& \geq u^2\left\{-1 + \delta \EE\left[\left(G + \kappa \sqrt{\frac{u^2+c}{u^2+1}} - \sqrt{\frac{1}{1+u^2}} \|\bmu\|\right)_+^2\right]\right\}.
\end{align*}
Since $\EE[(G + a)_+^2]$ is increasing in $a$ and
\[
\kappa \sqrt{\frac{u^2+c}{u^2+1}} - \sqrt{\frac{1}{1+u^2}} \|\bmu\| \geq (c \wedge 1)\kappa - \sqrt{\frac{1}{1+u^2}} \|\bmu\|.
\]
Then, it suffices to show that there exists some $0 \leq \bar \kappa < +\infty$ such that
\[
\inf_{u \geq 0} \left\{-1 + \delta \EE\left[\left(G + (c \wedge 1) \bar \kappa   - \sqrt{\frac{1}{1+u^2}} \|\bmu\|\right)_+^2\right]\right\} = -1 + \delta \EE\left[\left(G + (c \wedge 1) \bar \kappa   -  \|\bmu\|\right)_+^2\right] > 0.
\]
Since $c \wedge 1 > 0$, then $\EE[(G + (c \wedge 1) \kappa - \|\bmu\|)_+^2]$ is an increasing and continuous function in $\kappa$, and
\[
\lim_{\kappa \rightarrow +\infty} \left\{-1 + \delta\EE[(G + (c \wedge 1) \kappa - \|\bmu\|)_+^2]\right\} = +\infty.
\]
Consequently, there exists $\bar \kappa < +\infty$ such that $-1 + \delta\EE[(G + (c \wedge 1) \bar \kappa - \|\bmu\|)_+^2] > 0$,
which proves that $\kappa^*(c) < +\infty$.

\paragraph{Step 3: Proving $u^*(c) < +\infty$.}
For a given $c$, denote
\begin{align}
    \calA = \left\{(u,\kappa)\in {\RR}_+ \times \RR_+: f_{\delta}(u,\kappa,c) \leq 0\right\}, \qquad\calA_{\kappa} = \left\{u \in {\RR}_+: (u,\kappa) \in \calA \right\}.
\end{align} 
Then, to show that $u^* < +\infty$, it is equivalent to show that $\calA_{\kappa^*} \setminus \{+\infty\} \neq \varnothing$.
\paragraph{(1) $2 < \delta < \delta^*(\|\bmu\|)$.}
We first consider $\delta \in (2,\delta^*(\rho))$.
As $u \rightarrow +\infty$, we have for any $\kappa \geq 0$,
\begin{align}
    \lim_{u \rightarrow +\infty }f_{\delta}(u,\kappa,c) = \lim_{u \rightarrow +\infty} \left\{\left(\frac{\delta}{2}-1\right)u^2\right\} = +\infty,
\end{align}
which implies that $+\infty \notin \calA_{\kappa}$ for any $\kappa \geq 0$. Since $\calA_{\kappa^*} \neq \varnothing$ by the existence of the max-margin solution, then $\calA_{\kappa^*} \setminus \{+\infty\} \neq \varnothing$.

\paragraph{(2) $\delta \leq 2$.}
When $\delta \leq 2$, let $\kappa_0$ be the unique solution to
\begin{align}
    \frac{1}{\delta} = \EE\left[\left(G+\kappa\right)_+^2\right].
\end{align}
Since $f_{\delta}(u,\kappa,c)$ is increasing in $\kappa$, then for $\kappa > \kappa^\prime$, $\calA_{\kappa} \subseteq \calA_{\kappa^\prime}$. Also, for $\kappa < \kappa_0$, we have $\delta \EE\left[\left(G+\kappa\right)_+^2\right] < 1$, 
\begin{align}
    \lim_{u \rightarrow +\infty }f_{\delta}(u,\kappa,c) = \lim_{u \rightarrow +\infty} \left\{\left(\delta\EE\left[\left(G+\kappa\right)_+^2\right]-1\right)u^2\right\} = -\infty,
\end{align}
which implies that $\calA_{\kappa} \neq \varnothing$ and $+\infty \in \calA_{\kappa}$ for $\kappa < \kappa_0$.

Further, if $\kappa > \kappa_0$, we have
\begin{align}
    \delta \EE\left[\left(G+\kappa\right)_+^2\right] > 1,
\end{align}
which implies that
\begin{align}
    \lim_{u \rightarrow +\infty }f_{\delta}(u,\kappa,c) = \lim_{u \rightarrow +\infty} \left\{-u^2 + \left(\delta \EE\left[\left(G+\kappa\right)_+^2\right]\right)u^2\right\} = +\infty.
\end{align}
Thus, $+\infty \notin \calA_{\kappa}$ for $\kappa > \kappa_0$. Then, to complete the proof that $u^* < +\infty$, we need to show that $\calA_{\kappa_0} \neq \varnothing$. Consider
\begin{align}
    \sqrt{1+u^2}G + \kappa_0\sqrt{u^2+c} = Gu+\kappa_0u + \frac{G}{\sqrt{1+u^2}+u} + \frac{c}{\sqrt{u^2+c}+u},
\end{align}
there exists $C > 0$ such that for $u \geq C$,
we have
\begin{align}
    \frac{c}{\sqrt{u^2+c}+u} < \frac{1}{2} \Vert \bmu \Vert.
\end{align}
Then, we have
\begin{align}
    &\EE\left\{\left(\sqrt{1+u^2}G+\kappa_0\sqrt{u^2+c} - \Vert \bmu \Vert\right)_+^2\right\} \leq \EE\left\{\left(Gu+\kappa_0u+ \frac{G}{\sqrt{1+u^2}+u} - \frac{1}{2}\Vert \bmu \Vert\right)_+^2\right\}\nonumber\\
    & \leq \EE\left\{\left(Gu+\kappa_0u-\frac{1}{4}\Vert \bmu \Vert\right)_+^2 \ind\left\{\frac{G}{\sqrt{1+u^2}+u} <\frac{1}{4}\Vert \bmu \Vert \right\}\right\} \nonumber\\
    &\qquad + \EE\left\{\left(Gu+\kappa_0u+\frac{G}{\sqrt{1+u^2}+u}\right)_+^2 \ind\left\{\frac{G}{\sqrt{1+u^2}+u} >\frac{1}{4}\Vert \bmu \Vert \right\}\right\} \nonumber\\
    & \leq \EE\left\{\left(Gu+\kappa_0u-\frac{1}{4}\Vert \bmu \Vert\right)_+^2\right\} + \EE\left\{\left(Gu+\kappa_0u+\frac{G}{\sqrt{1+u^2}+u}\right)_+^2 \ind\left\{\frac{G}{\sqrt{1+u^2}+u} >\frac{1}{4}\Vert \bmu \Vert \right\}\right\}\nonumber\\
    & := I_1(u) + I_2(u),
\end{align}
where $\lim_{u \rightarrow +\infty} I_2(u) = 0$. 
For $I_1(u)$, consider the function 
\begin{align}
    q(t) = \delta\EE\left\{\left(G+\kappa_0-t\right)_+^2\right\} - 1,
\end{align}
then $q(0)=1$ and $q^\prime(t) = -2\delta \EE\left\{\left(G+\kappa_0-t\right)_+\right\} < 0$. At $t\approx 0$, we have the approximation $q(t) = q^\prime(0) t + o(t)$.
If we take $t = \Vert \bmu \Vert/(4u)$, $-u^2 + \delta I_1(u)$ can be approximated by
\begin{align}
    -u^2 + \delta I_1(u) = \left\{\delta\EE\left\{\left(G+\kappa_0-\frac{\Vert \bmu \Vert}{4u}\right)_+^2\right\} - 1\right\}u^2 = \frac{\Vert \bmu \Vert}{4}u q^\prime(0) + o(u),
\end{align}
thus there exists sufficiently large $u \geq C$ such that $-u^2 + \delta(I_1(u)+I_2(u)) < 0$, which implies that $\calA_{\kappa_0} \neq \varnothing$.

We consider the constrained problem with arbitrarily small $\veps>0$,
\begin{align}\label{eq:opt2}
    \begin{cases}
    \max &\kappa\\
    {\rm subject~to} &\inf_{u \geq 0, \kappa \geq \kappa_0+\veps} f_{\delta}(u,\kappa,c) \leq 0,
    \end{cases}
\end{align}
Let $(\tilde{u},\tilde{\kappa})$ be any solution to \eqref{eq:opt2}, where the existence is obtained by the continuity of $g_{\delta}(u,\kappa,c)$ in $\kappa$. On one hand, since \eqref{eq:opt2} is a constrained version of \eqref{eq:opt_new}, then we have $\tilde{\kappa} \leq \kappa^*$; on the other hand, when $\tilde{\kappa} < \kappa^*$, then $(u^*,\kappa^*)$ is a feasible point for the constrained problem, thus $\kappa^* \leq \tilde{\kappa}$. Then, we have $\kappa^* = \tilde{\kappa} > \kappa_0$, indicating that $0 < u^* < +\infty$ since $+\infty \notin \calA_{\kappa}$ for $\kappa > \kappa_0$.

Combining pieces above, we have shown that in either regime $\delta \in (2,\delta^*(\rho))$ or $\delta \leq 2$, the solution to \eqref{eq:opt_new} satisfies that $0 < u^* < +\infty$, which completes the proof.

\subsubsection{Proof of Lemma~\ref{lem:part3}}\label{sec:proof-part3}
Recall the notation
\begin{align}
    \calE_{n,p,\kappa,v} &\coloneqq \left\{\text{there~exists~} \bbeta \in \Theta \cap \bB_v: y_i \langle \hh_{0,i},\bbeta \rangle \geq \kappa \text{~for~all~} i \leq n\right\},\\
    \xi_{n,p,\kappa,v} &\coloneqq \xi_{n,p,\kappa}(\bB_v \cap \Theta) = \min_{\bbeta\in\bB_v \cap \Theta} \max_{\blambda:\Vert \blambda \Vert \leq 1, \by \odot \blambda \geq 0}\frac{1}{\sqrt{p}}\blambda^\top(\kappa \by - \bH_0\bbeta).
\end{align}
where $\bB_v = \{\bbeta \in \RR^p: \Vert \bP^{\perp}_{\bmu} \bbeta \Vert \leq v \Vert \bP_{\bmu} \bbeta \Vert\}$  is a convex cone for each choice of $v \in \bar{\RR}_+$. Accordingly, we denote
\begin{align}
    \xi_{n,p,\kappa,v}^{(2)} & \coloneqq \xi_{n,p,\kappa}^{(2)}(\tilde{\Theta}_{\eta} \cap \{(\gamma,z): z \leq v \gamma\}) \nonumber\\
    & =  \min_{(\gamma,z)\in\tilde{\Theta}_{\eta} \cap \{(\gamma,z): z \leq v \gamma\}} \left\{-z + \sqrt{\delta}\sqrt{\EE\left(\sqrt{\gamma^2+z^2}G+\kappa - \gamma \Vert \bmu \Vert\right)_+^2}\right\}
\end{align}
Then, by definition,
\begin{align}
    \xi_{n,p,\kappa} = \min_{v \in \bar{\RR}_+} \xi_{n,p,\kappa,v},\qquad \xi^{(2)}_{n,p,\kappa} = \min_{v \in \bar{\RR}_+} \xi^{(2)}_{n,p,\kappa,v}.
\end{align}
Recall the definition of $\hat{\kappa}$ and $\hat{u}$ in Theorem~\ref{thm:CGMT}, then
\begin{align}
    \hat{\kappa} = \sup_{\kappa > 0} \left\{\xi_{n,p,\kappa} = 0\right\}, \qquad \hat{u} \in \left\{ v:~ \xi_{n,p,\hat{\kappa},v} = 0\right\}.
\end{align}

\paragraph{Step 1: Proving $\calE_{n,p,\kappa^*,u^*+\veps}$ and $\calE^c_{n,p,\kappa^*,u^*-\veps}$.}
For any $\veps > 0$, we first show that
\begin{align}
    \lim_{n \rightarrow +\infty, n/p\rightarrow \delta}\PP\left(\xi_{n,p,\kappa^*,u^*-\veps} > 0\right) = 1.
\end{align}
By the result in Lemma~\ref{prop:cgmt_general}, we need to show that
\begin{align}
    \inf_{(\gamma, z) \in \tilde{\Theta}_{\eta} \cap \{(\gamma,z):z \leq (u^*-\veps) \gamma\}} F_{\delta}(\gamma,z,\kappa^*) > 0,
\end{align}
up to the change of variables, which is equivalent to
\begin{align}
    \inf_{(\gamma, u) \in \tilde{\Omega} \cap \{(\gamma,u): z\leq u^*-\veps\}} \tilde{F}_{\delta}(\gamma,u,\kappa^*) > 0.
\end{align}
Recall the definition of $u^*$ as the smallest minimizer of $f_{\delta}(u,\kappa^*,c)$ in \eqref{eq:opt_new}. By Section~\ref{sec:proof-g-F-connections} and Lemma~\ref{lem:equi_f_F}, we have $\tilde{F}_{\delta}((u^{*2}+c)^{-1/2},u^*,\kappa^*) = 0$. Assume $\inf_{(\gamma, u) \in \tilde{\Omega} \cap \{(\gamma,u): u\leq u^*-\veps\}} \tilde{F}_{\delta}(\gamma,u,\kappa^*) \leq 0$, which implies that there exists $v^* \leq u^* - \veps < u^*$ and $\gamma_1$ such that 
\begin{align}
    \tilde{F}_{\delta}(\gamma_1,v^*,\kappa^*) \leq \tilde{F}_{\delta}((u^{*2}+c)^{-1/2},u^*,\kappa^*) = \inf_{(\gamma, u) \in \tilde{\Omega}} \tilde{F}_{\delta}(\gamma,u,\kappa^*),
\end{align}
then $(\gamma_1,v^*)$ is also the minimizer to the unconstrained problem. This is a contradiction to the definition that $u^*$ is the smallest minimizer, thus 
\begin{align}
    \inf_{(\gamma, z) \in \tilde{\Theta}_{\eta} \cap \{(\gamma,z):z \leq (u^*-\veps) \gamma\}} F_{\delta}(\gamma,z,\kappa^*) > 0 \qquad \Longrightarrow \qquad \lim_{n \rightarrow +\infty, n/p\rightarrow \delta}\PP\left(\xi_{n,p,\kappa^*,u^*-\veps} > 0\right) = 1.
\end{align}
In addition, since $\tilde{\Theta}_{\eta} \cap \{(\gamma,z):z \leq u^* \gamma\} \subseteq \tilde{\Theta}_{\eta} \cap \{(\gamma,z):z \leq (u^*+\veps) \gamma\}$,
\begin{align*}
\inf_{(\gamma, z) \in \tilde{\Theta}_{\eta} \cap \{(\gamma,z):z \leq (u^*+\veps) \gamma\}} F_{\delta}(\gamma,z,\kappa^*) & \leq \inf_{(\gamma, z) \in \tilde{\Theta}_{\eta} \cap \{(\gamma,z):z \leq u^* \gamma\}} F_{\delta}(\gamma,z,\kappa^*) \leq 0\\
& = \inf_{(\gamma, u) \in \tilde{\Omega} \cap \{(\gamma,u): u \leq u^*\}} \tilde{F}_{\delta}(\gamma,u,\kappa^*)\\
& = \inf_{u \in [0,u^*]}\tilde{F}_{\delta}((u^{*2}+c)^{-1/2},u^*,\kappa^*) \leq 0.
\end{align*}
Then, for any $\veps > 0$, we have 
\begin{align}
    \calE_{n,p,\kappa^*,u^*+\veps}\qquad \text{and} \qquad \calE^c_{n,p,\kappa^*,u^*-\veps} \qquad \text{hold}.
\end{align}

\paragraph{Step 2: Proving $\hat{u} \geq u^*$.}
Suppose the claim is wrong, then there exists $\veps > 0$ such that $\hat{u} \leq u^*-\veps < u^*$, which indicates that
\[
\hat{\bbeta} \in \Theta \cap \bB_{u^* - \veps}.
\]
Then, with $\hat \bbeta$, 
\[
y_i \langle \hh_{0,i}, \hat{\bbeta} \rangle \geq \hat{\kappa},\qquad\text{for ~all~}i \leq n.
\]
By Lemma~\ref{prop:cgmt_general}, we have
\[
\inf_{0 \leq u \leq u^*-\veps} f_{\delta}(u,\hat{\kappa},c) \leq 0.
\]
As $\hat{\kappa} \overset{p}{\rightarrow} \kappa^*$ and $f_{\delta}(u,\kappa,c)$, we have 
\[
\inf_{0 \leq u \leq u^*-\veps} f_{\delta}(u,\kappa^*,c) \leq 0,
\]
which implies that $\calE_{n,p,\kappa^*,u^*-\veps}$ holds. This draws the contradiction and proves that $\hat{u} \geq u^*$.

\paragraph{Step 3: Proving $\hat{u} \leq u^* + 2\veps$.}
For any $\veps > 0$, as $ \bB_{u^*} \subseteq \bB_{u^*+\veps}$, we already have 
\begin{align}\label{eq:extended}
    \lim_{n \rightarrow +\infty, n/p\rightarrow \delta}\PP\left(\xi_{n,p,\kappa^*,u^*+\veps} = 0\right) = 1.
\end{align}
Denote by $\hat \bbeta^\veps$ any maximizer to the optimization problem \eqref{opt:maxmargin2} with an additional constraint $\bbeta \in \bB_{u^* + \veps}$, and let $\hat u^\veps  = \norm{\bP_{\bmu}^\perp \hat \bbeta^{\veps}} / \norm{\bP_{\bmu} \hat \bbeta^{\veps}}$. 
Our next goal is to show that with probability approaching one,
$\hat \bbeta$ satisfies $\norm{\bP_{\bmu}^\perp \hat \bbeta} \le (u^* + 2\veps) \norm{\bP_{\bmu} \hat \bbeta}$. Once it is proved, then $\hat u \xrightarrow{p} u^*$ since $\veps$ is arbitrary.

Consider any $\bbeta$ with $\norm{\bbeta}_{\bOmega} \le 1$ and $\bbeta \notin \bB_{u^* + 2\veps}$ where we recall
\begin{equation*}
\bB_{u^* + 2\veps} = \{ \xx \in \mathbb{R}^p: \norm{\bP_{\bmu}^\perp \xx} \le (u^* + 2\veps) \norm{\bP_{\bmu} \xx} \}.
\end{equation*}
Since the region $\Theta = \{ \xx: \xx^\top \bOmega \xx \le 1\}$ is an ellipsoid, it must be strictly convex. For any $\bbeta \in \Theta$, the interpolant
\begin{equation*}
\bbeta_\lambda := \lambda \hat \bbeta^\veps + (1-\lambda) \bbeta, \qquad \text{where}~\lambda \in (0,1)
\end{equation*}
lies in $\Theta$. Since $\norm{\hat \bbeta^\veps - \bbeta} \ge c_\veps > 0$ for certain constant $c_\veps$, we can find (sufficiently small) constants $\lambda \in (0,1)$ and $\alpha > 0$ such that 
\begin{equation*}
\bbeta_{\lambda} \in \bB_{u^* + \veps}, \quad \text{and} \quad (1+\alpha) \bbeta_{\lambda} \in\Theta.
\end{equation*}
Let us use $\kappa(\bbeta) = \min_{i\le n} y_i \langle \hh_{0,i}, \bbeta \rangle$ to denote the margin associated with $\bbeta$. By the definition of $\hat \bbeta^\veps$ and the fact that $(1+\alpha)\bbeta_{\lambda} \in \bB_{u^*+\veps}$, we have 
\begin{equation}\label{ineq:convex1}
\kappa((1+\alpha) \bbeta_\lambda) \le \kappa(\hat \bbeta^{\veps}).
\end{equation}
By linearity, we also have
\begin{equation}\label{ineq:convex2}
\kappa((1+\alpha) \bbeta_\lambda) \ge (1+\alpha)\lambda \kappa(\hat \bbeta^\veps) + (1+\alpha)(1-\lambda) \kappa(\bbeta).
\end{equation}
Combining \eqref{ineq:convex1} and \eqref{ineq:convex2}, we have
\begin{equation*}
\kappa(\bbeta) \le \frac{1 - (1+\alpha)\lambda}{(1+\alpha)(1-\lambda)} \kappa(\hat \bbeta^\veps) < \kappa(\hat \bbeta^\veps).
\end{equation*}
This proves that if $\bbeta = \hat{\bbeta} \notin \bB_{u^*+2\veps}$, then $\hat{\bbeta} < \kappa(\hat \bbeta^\veps)$, which is a contradiction. Then, we have shown that $\hat \bbeta \in \bB_{u^* + 2\veps}$. Consequently, we have $\hat{u} \in [u^*, u^*+2\veps)$ for any $\veps > 0$, thus
\begin{align}
    \hat u = \frac{\Vert \bP_{\bmu}^{\perp} \hat{\bbeta} \Vert}{\Vert \bP_{\bmu} \hat{\bbeta} \Vert} \xrightarrow{p} u^*.
\end{align}


\subsubsection{Proof of Lemma~\ref{lem:monot}}\label{sec:proof-monot}
By the definition of $\kappa^*(c)$ and $u^*(c)$, it is easy to check that 
\begin{align}\label{eq:optm_cond}
        \begin{cases}
            f_{\delta}(u^*(c),\kappa^*(c),c) &= 0,\\
            \partial_u f_{\delta}(u^*(c),\kappa^*(c),c) &= 0.
        \end{cases}
    \end{align}
In other words, for any c, $(u^*(c), \kappa^*(c), c)$ is the zero point of the systems of equations
\begin{align}\label{eq:optm_cond_general}
        \begin{cases}
            f(u,\kappa,c) &= 0,\\
            \partial_u f(u,\kappa,c) &= 0,
        \end{cases}
    \end{align} 
where we identify $f_\delta$ with $f$ to simply the notation.    

Fix any $c_0$, and its corresponding $u^*(c_0), \kappa^*(c_0)$. 
We can compute that  
\begin{align}\label{eq:invertable-Jacobian}
\partial_c\partial_u f - \frac{\partial_c f}{\partial_{\kappa} f} \partial_{\kappa}\partial_u f \mid_{u^*(c_0), \kappa^*(c_0), c_0} < 0
\end{align}
As a result, the implicit function theorem tells us that in an open neighborhood around $u^*(c_0)$, 
the zero point of~\eqref{eq:optm_cond_general} can be written as a continuously differentiable function 
$u \to \kappa^*(u), c^*(u)$.
Moreover, implicit function theorem also tells us that 
\begin{align}
    \left(\partial_c \partial_u f - \frac{\partial_c f}{\partial_{\kappa} f} \partial_{\kappa}\partial_u f\right) \mid_{u^*(c_0), \kappa^*(c_0), c_0} \; \cdot \; c^\prime(u^*(c_0)) + \partial^2_u f \mid_{u^*(c_0), \kappa^*(c_0), c_0} = 0
\end{align}
As $u^*(c_0)$ is a minimizer for $f(u,\kappa^*(c_0),c_0)$, we have $\partial^2_u f \mid_{u^*(c_0), \kappa^*(c_0), c_0} \geq 0$. 
In all, we conclude that $c^\prime(u^*(c_0)) \geq 0$.

We have the following lemma which shows the continuity of $u^*(c)$.
\begin{lemma}\label{lem:cont}
   Both $\kappa^*(c)$ and $u^*(c)$ are continuous in $c$.
\end{lemma}

\noindent See the end of this section for the proof.
\medskip

According to Lemma~\ref{lem:cont}, in an open neighborhood $\calO$ of $(u_0^*,\kappa_0^*,c_0) = (u^*(c_0), \kappa^*(c_0), c_0)$, for any $c_1 > c_0$ such that $c_1 - c_0$ is sufficiently small, denote $(u^*_1,\kappa^*_1)$ as the solution to \eqref{eq:opt_new} with $c = c_1$. Since both $u^*(c)$ and $\kappa^*(c)$ are continuous in $c$, we have $(u_0^*,\kappa_0^*,c_0), (u^*_1,\kappa^*_1,c_1) \in \calO$. By the implicit function theorem, in the neighborhood $\calO$, $c$ can be written as a differentiable function $c(u^*)$ of $u^*$ with $0 \leq c^\prime(u^*) < +\infty$. 
Suppose we do not have $u_0^* < u_1^*$. If $u_0^* > u_1^*$, by the mean value theorem, then there exists $u_2^* \in (u_1^*, u_0^*)$ such that $c^\prime(u_2^*) < 0$, which draws the contradiction. If $u_0^* = u_1^*$, then $u_0^*$ is mapped to two distinct values and the function $c(u)$ is not well-defined. Therefore, we have $u^*_0 < u^*_1$ and $u^*(c)$ is increasing in $c$.


\paragraph{Proof of~\eqref{eq:invertable-Jacobian}.}
    
Denote
$
    Z(u) = \sqrt{1+u^2}G+\kappa\sqrt{u^2+c} - \Vert \bmu \Vert.
$
By direct calculations, we have
\begin{enumerate}
    \item $\partial_{\kappa}f = 2\delta\sqrt{u^2+c} \EE[Z_+(u)]$,
    \item $\partial_{c} f = \frac{\kappa\delta}{\sqrt{u^2+c}}\EE[Z_+(u)]$,
    \item $\partial_u f = -2u + 2u\delta\PP(Z(u) \geq 0) + 2\delta\frac{\kappa u}{\sqrt{u^2+c}}\EE[Z_+(u)]$,
    \item $\partial_{\kappa}\partial_u f = 2u\delta\frac{\sqrt{u^2+c}}{\sqrt{u^2+1}}\phi\left(\frac{\sqrt{u^2+c}}{\sqrt{u^2+1}} \kappa - \frac{\Vert \bmu \Vert}{\sqrt{u^2+1}}\right) + 2u\frac{\delta}{\sqrt{u^2+c}}\EE[Z_+(u)] + 2u\delta \kappa \PP(Z(u) \geq 0)$,
    \item $\partial_{c} \partial_u f = u\delta\frac{\kappa}{\sqrt{u^2+1}} \frac{1}{\sqrt{u^2+c}} \phi\left(\frac{\sqrt{u^2+c}}{\sqrt{u^2+1}} \kappa - \frac{\Vert \bmu \Vert}{\sqrt{u^2+1}}\right) - u\delta \kappa \frac{1}{(u^2+c)^{3/2}}\EE[Z_+(u)] + \delta \frac{\kappa^2}{u^2+c}\PP(Z(u) \geq 0)$.
\end{enumerate}

Denoting $\nu = \frac{\sqrt{u^2+c}}{\sqrt{u^2+1}} \kappa - \frac{\Vert \bmu \Vert}{\sqrt{u^2+1}}$, one then has
\begin{align}
    \frac{\partial_c f}{\partial_{\kappa} f} \partial_{\kappa}\partial_u f =  \delta u\kappa \sqrt{\frac{1}{(u^2+1)(u^2+c)}}\phi(\nu) + \frac{\delta  u\kappa}{(u^2+c)^{3/2}}\EE[Z_+(u)] + \frac{\delta u\kappa^2}{u^2+c}\PP(Z(u) \geq 0).
\end{align}
As the result, for any $0 < u < +\infty$, we have
\begin{align}
    &\partial_c\partial_u f - \frac{\partial_c f}{\partial_{\kappa} f} \partial_{\kappa}\partial_u f = -\frac{2\delta u \kappa}{(u^2+c)^{3/2}}\EE[Z_+(u)] < 0\nonumber.
\end{align}

\paragraph{Proof of Lemma~\ref{lem:cont}.}
Denote $\ww = (\gamma,u)$ and
\[
\kappa(\ww) = \sup_{\kappa \geq 0} \{\tilde F_{\delta}(\gamma,u,\kappa) \leq 0\},
\]
where $\kappa(\ww) = 0$ if $\{\kappa \geq 0: \tilde F_{\delta}(\gamma,u,\kappa) \leq 0\} = \varnothing$, then $\kappa(\ww)$ is continuous in $\ww$. Rewrite $U(c) = \{(\gamma,u)\in \RR \times \RR_+: \gamma^2(u^2+c) \leq 1\}$, which satisfies that $U(c_1) \subseteq U(c_2)$ for $c_1 \geq c_2$, and denote 
\[
\kappa^*(c) = \sup_{\ww \in U(c)} \kappa(\ww)
\]
Consider any $0 < c < +\infty$ and $\ww^*(c)$ is the associated solution to ~\eqref{eq:opt_Ftilde}. 
\paragraph{Step 1: Continuity of $\kappa^*(c)$.}
We first show that $\kappa^*(c)$ is continuous in $c$. Otherwise, without loss of generality, there exists $\veps > 0$ such that there exists an increasing sequence $\{c_m\}_{m \geq 1}$ with $c/2 < c_m < c$ and $c - c_m \leq \frac{1}{m}$ which satisfies that $| \kappa^*(c_m) - \kappa^*(c) | \geq \veps$, which forms a non-increasing sequence $\{\kappa^*(c_m)\}_{m=1}^{\infty}$ and $\{\ww^*(c_m)\}_{m=1}^{\infty} \subseteq U(c/2)$. Recall that $\kappa^*(c)$ is non-increasing in $c$, then $\kappa^*(c_m) \geq \kappa^*(c) + \veps$ and $\kappa^*(c_m) \in [\kappa^*(c), \kappa^*(c/2)]$. Then, for the sequence in a compact set, there is a subsequence $\{c_{m_k}\}_{k=1}^{\infty} \subseteq \{c_m\}_{m=1}^{\infty}$ such that
\begin{align}
   \lim_{k \rightarrow +\infty} c_{m_k} = c,\qquad \lim_{k \rightarrow +\infty} \kappa^*(c_{m_k}) = \tilde{\kappa}, \qquad \lim_{k \rightarrow +\infty} \ww^*(c_m) = \tilde{\ww} \in \bigcap_{k=1}^{\infty} U(c_{m_k}) = U(c). 
\end{align}
Thus, there exists $K \in \NN$ such that for any $k \geq K$,
\begin{align}
   \kappa(\tilde{\ww}) \geq \kappa(\ww^*(c_{m_k})) - \frac{\veps}{2} = \kappa^*(c_{m_k}) - \frac{\veps}{2} \geq \kappa^*(c) + \frac{\veps}{2} > \kappa^*(c),
\end{align}
which violates the fact that $\ww^*(c) = \argmax_{\ww \in U(c)} \kappa(\ww)$. This proves the left-continuity of $\kappa^*(c)$ and the right-continuity follows the same technique.

\paragraph{Step 2: Continuity of $\ww^*(c)$.} 
Suppose $\ww^*(c) = (\gamma^*(c),u^*(c))$ is not continuous in $c$, particularly for the second coordinate, then there exists $\veps > 0$ such that for any $m \in \NN$, there exists an increasing sequence $\{c_m\}_{m \geq 1}$ with $c/2 < c_m < c$ and $c - c_m \leq \frac{1}{m}$ such that $|u^*(c_m) - u^*(c)| \geq \veps$, which forms a sequence $\{\ww^*(c_m)\}_{m=1}^{\infty}$. Also, we have $\ww^*(c_m) = (\gamma^*(c_m), u^*(c_m)) \in U(c_m) \subseteq U(c/2)$ for each $m$. As $U(c/2)$ is compact, then there is a subsequence $\{c_{m_k}\}_{k=1}^{\infty} \subseteq \{c_m\}_{m=1}^{\infty}$ such that
\begin{align}
   \lim_{k \rightarrow +\infty} c_{m_k} = c,\qquad \lim_{k \rightarrow +\infty} \ww^*(c_{m_k}) = \tilde{\ww} = (\tilde \gamma, \tilde u) \in \bigcap_{k=1}^{\infty} U(c_{m_k}) = U(c). 
\end{align}
Since $\tilde{F}_{\delta}(\ww,\kappa)$ is continuous in $(\ww,\kappa)$ and for any $k$, $\tilde{F}_{\delta}(\ww^*(c_{m_k}),\kappa^*(c_{n_k})) = 0$, then $\tilde{F}_{\delta}(\tilde \ww,\tilde \kappa) = 0$ with $|\tilde u - u^*(c)| \geq \veps$.
Denote
\[
\calB_k = \left\{\ww: \tilde{F}_{\delta}(\ww,\kappa^*(c_{m_k})) \leq 0\right\},\qquad \calB_{\infty} = \left\{\ww: \tilde{F}_{\delta}(\ww,\kappa^*(c)) \leq 0\right\}
\]
Since $\kappa^*(c_{m_k})$ is non-increasing as $k$ increases with $\lim_{k \rightarrow +\infty}\kappa^*(c_{m_k}) = \kappa^*(c)$, and $\tilde{F}_{\delta}(\ww,\kappa)$ is continuous and increasing in $\kappa$, we have
\[
\calB_k \subseteq \calB_{k+1} \subseteq \calB_{\infty},\qquad \bigcup_{k \geq 1} \calB_k = \calB_{\infty}.
\]
By definition, $u^*(c_{m_k}) = \inf \{u: (\gamma,u) \in \calB_k\}$, then as $u^*(c_{m_k}) \rightarrow \tilde u$,
\[
\tilde u \leq \inf \{u: (\gamma,u) \in \calB_k\}, \qquad \text{for~any~} k.
\]
Consequently, we have
\[
\tilde u \leq \inf_{k \geq 1} \inf \left\{u: (\gamma,u) \in \calB_k\right\} = \inf \left\{u: (\gamma,u) \in \bigcup_{k \geq 1} \calB_k \right\} = \inf \left\{u: (\gamma,u) \in \calB_{\infty}\right\}.
\]
In addition, by definition, $u^*(c) = \inf \left\{u: (\gamma,u) \in \calB_{\infty}\right\}$, we have $\tilde u \leq u^*(c)$ and $\tilde u, u^*(c) \in \left\{u: (\gamma,u) \in \calB_{\infty}\right\}$ such that $|\tilde u - u^*(c)| \geq \veps > 0$, which draws the contradiction. Thus, $u^*(c)$ is continuous in $c$.

\section{Proof for Section~\ref{sec:inhomogeneous}}
\subsection{Proof of Proposition~\ref{prop:loss-aug}}
Recall that we assumed features and augmentations with inhomogeneous covariance take the following form.
\begin{align}
    &\hh_{0,i} \stackrel{\mathrm{i.i.d.}}{\sim} \frac{1}{2} \mathcal{N} (-\bmu, \bI_p) + \frac{1}{2} \mathcal{N} (\bmu, \bI_p) \\
&\hh_i,\hh_i^+ |\, \hh_{0,i} \stackrel{\mathrm{i.i.d.}}{\sim} \mathcal{N}(\hh_{0,i}, \sigma_{\mathrm{aug}}^2 \bA).
\end{align}
where $\bA \succeq \bI_p$. The special case where $\bA = \bI_p$ is already analyzed before. Here, we can write
\begin{align*}
    &\hh_{i} \stackrel{d}{=} \xi_i \bmu + ( \bI_p + \sigma_\aug^2 \bA)^{1/2} \gggg_i,  \qquad \text{where} \\
&\xi_i \bot \gggg_i, \qquad \xi \sim \mathrm{Unif}(\{\pm 1\}), \qquad \gggg_i \sim \calN(\mathbf{0}, \bI_p).
\end{align*}
We have
\begin{align*}
& \E \big[ \norm{ \bW \hh_i }^2 \big] = \norm{\bW \bmu}^2 + \tr\big( \bW (\bI_p + \sigma_\aug^2 \bA) \bW^\top \big) \\
& \qquad \qquad \quad \ = \norm{ \bW \bmu}^2 + \norm{ \bW}_{\mathrm{F}}^2 + \sigma^2_\aug\tr\big(\bW \bA \bW^\top\big) =: \tilde \alpha,\\
& \E \big[ \norm{ \bW \hh_i^+ - \bW \hh_i}^2 \big] = 2 \sigma_\aug^2 \tr\big( \bW \bA \bW^\top \big)  \\
&  \norm{ \bW \hh_i - \bW \hh_i^-}^2   =  \norm{ \bW(\xi_i - \xi_i^-) \bmu + \bW ( \bI_p + \sigma_\aug^2 \bA)^{1/2}(\gggg_i - \gggg_i^-)}^2.
\end{align*}
Consider the same infinite-sample loss $\calL(\bW)$ in Section~\ref{sec:gmm_expansion_shrinkage}. The alignment loss can be written as
\[
\calL_{\aalign}(\bW) = \frac{1}{2\tau \tilde \alpha} \EE_{\hh,\hh^+}\left[\|\bW\hh - \bW\hh^+\|^2\right] = \frac{\sigma_\aug^2}{\tau \tilde \alpha}\, \tr \big(  \bW \bA \bW^\top \big).
\]
Further, the difference between the negative pair has the distribution
\[
\hh - \hh^- \overset{d}{=} \frac{1}{2}\delta_0 + \frac{1}{4}\delta_{2\bmu} + \frac{1}{4}\delta_{-2\bmu} + \sqrt{2}\left(\bI + \sigma_{\aug}^2\bA\right)^{1/2} \gggg,
\]
where $\gggg \sim \calN(0,\bI)$. Let $\tilde \bW = \bW \left(\bI + \sigma_{\aug}^2\bA\right)^{1/2}$ and $\tilde \bmu = \left(\bI + \sigma_{\aug}^2\bA\right)^{-1/2} \bmu$. Then, we have
\[
\bW\hh - \bW\hh^- \overset{d}{=} \frac{1}{2}\delta_0 + \frac{1}{4}\delta_{2\tilde \bW \tilde \bmu} + \frac{1}{4}\delta_{-2\tilde \bW \tilde \bmu} + \sqrt{2}\tilde \bW \gggg,
\]
which has the same form as in Section~\ref{sec:loss_calc}.
Therefore, with the SVD of $\tilde \bW = \sum_{j=1}^p \tilde \sigma_j \tilde \uu_j \tilde \vv_j^\top$, we can calculate the following expectation
\begin{align*}
&\mathbb{E}_{\hh^{-},\hh}\left[\exp\left(-\frac{\|\tilde\bW \hh-\tilde\bW \hh^{-}\|_{2}^{2}}{2\tau\tilde\alpha}\right)\right]\nonumber\\
&=\frac{1}{2}\prod_{j=1}^{p}\left(1+\frac{2\tilde \sigma_{j}^{2}}{\tau\tilde \alpha}\right)^{-1/2}\left(1+\exp\left(-\sum_{j=1}^{p}\frac{2\tilde \sigma_{j}^{2}}{2\tilde \sigma_{j}^{2}+\tau\tilde \alpha}(\tilde \vv_{j}^{\top}\tilde \bmu)^{2}\right)\right).
\end{align*}
As a result, we have decomposition $\calL(\bW) = \calL_{\aalign}(\bW) + \calL_{\unif}(\bW)$, where
\begin{align*}
& \calL_{\aalign}(\bW) = \frac{\sigma_\aug^2}{\tau \tilde \alpha}\, \tr \big(  \bW \bA \bW^\top \big) \\
& \calL_{\unif}(\bW) = -\frac{1}{2} \sum_{j=1}^p \log \left(1 + \frac{2 \tilde \sigma_j^2}{\tau \tilde \alpha} \right) + \log \left( 1 + \exp\Big( - \sum_{j=1}^p \frac{2\tilde \sigma_j^2 \langle \tilde \bmu, \tilde \vv_j \rangle^2}{2\tilde \sigma_j^2 + \tau \tilde \alpha } \Big) \right) - \log 2.
\end{align*}
Further, if we assume that $\tilde\alpha \gg \| \bW \|_{\mathrm{F}}^2$, then $\calL_{\unif}(\bW)$ can be approximated by
\begin{align*}
\tilde\calL_{\unif}(\bW) &= -\sum_{j=1}^p \frac{\tilde \sigma_j^2}{\tau \tilde\alpha} + \log\left(1 + \exp\left(-\sum_{j=1}^p \frac{2\tilde \sigma_j^2 \langle \tilde \bmu, \tilde \vv_j \rangle^2}{\tau \tilde\alpha}\right)\right) - \log 2\\
& =-\frac{\| \tilde\bW \|_{\mathrm{F}}^2}{\tau \tilde \alpha} + \log \left( 1 + \exp \Big( - \frac{2 \norm{\tilde \bW  \tilde \bmu}^2}{\tau \tilde \alpha} \Big) \right)-\log 2\\
& = -\frac{1}{\tau \tilde \alpha} \left(\norm{\bW}_{\mathrm{F}}^2 + \sigma_\aug^2  \tr\big(\bW \bA \bW^\top\big)\right) + \log \left( 1 + \exp \Big( - \frac{2 \norm{\bW \bmu}^2}{\tau \tilde \alpha} \Big) \right)-\log 2.
\end{align*}
We can then introduce an approximate loss $\tilde \calL(\bW) = \calL_{\aalign}(\bW) + \tilde \calL_{\unif}(\bW)$, where 
\begin{equation*}
\tilde \calL(\bW) = - \frac{1}{\tau \tilde \alpha} \norm{\bW}_{\mathrm{F}}^2  + \log \left( 1 + \exp \Big( - \frac{2 \norm{\bW \bmu}^2}{\tau \tilde \alpha} \Big) \right).
\end{equation*}

\subsection{Proof of Theorem~\ref{thm:inhomo}}\label{sec:proof-inhomo}
Recall that the minimization problem is $\min_{\bW} \tilde \calL(\bW)$, where
\begin{align*}
\tilde \calL(\bW) &= -\frac{1}{\tau\tilde \alpha}\Vert \bW\Vert_{\mathrm{F}}^{2}+\log\left(1+\exp\left(-\frac{2\|\bW \bmu\|^{2}}{\tau\tilde \alpha}\right)\right), \\
\tilde \alpha & =(1+\sigma_\aug^{2})\|\bW\|_{\mathrm{F}}^{2}+\|\bW\bmu\|^{2}+\sigma_\aug^{2}\rho_\aug\|\bW\vv_\aug\|^{2}.
\end{align*}
First, notice that the loss $\tilde \calL(\bW)$ is scale-invariant. Hence without loss of generality, we can assume $\|\bW\|_{\mathrm{F}}^{2}=1$.
Then $\tilde \calL(\bW)$ only depends on the values of 
\[
T(\bW) = \|\bW \bar \bmu\|^{2},\qquad\text{and}\qquad\|\bW\vv_\aug\|^{2}\,.
\]
Second, we can rewrite the optimization problem into a nested one:
\begin{equation}\label{opt:doublemin}
\min_{T \in [0,1]} \Big[ \min_{\bW: T(\bW) = T} \tilde \calL(\bW) \Big].
\end{equation}
It is easy to see that given any fixed value of $T(\bW)$, the loss $\tilde \calL(\bW)$ is decreasing as $\|\bW\vv_\aug\|^{2}$ decreases. Thus, the inner minimization problem is easy to solve: we only need to determine the smallest possible value of $\|\bW\vv_\aug\|^{2}$, which will be the focus below.

Consider the Gram--Schmidt orthogonalization for $\bar \bmu, \vv_\aug$: Let $\vv_{\perp} \in \mathrm{span}(\bar \bmu, \vv_\aug)$ be a unit  vector orthogonal to $\bar \bmu$; define $a\in [-1,1]$ to be coefficient in the orthogonal decomposition
\begin{align*}
\bar{\bmu} = r \vv_\aug+ a \vv_{\perp}, 
\end{align*}
where we recall that $r = \langle \bar{\bmu} , \vv_\aug \rangle \in [-1,1]$. 
It is also clear that $a^2 + r^2 = 1$.

We split the inner minimization into two cases. 

\paragraph{Case 1: $0 \le T \le a^2$.} We claim that there exists some $\bW$ such that $\|\bW\|_{\mathrm{F}}^2=1$, $\|\bW\bar{\bmu}\|^{2}=T$ and $\|\bW\vv_\aug\|=0$. To prove this claim, we will construct a $\bW$ of the form
\[
\bW=\sum_{j=1}^{p}s_{j} \bar \uu_{j} \bar \vv_{j}^{\top}, 
\]
such that the required equalities hold.  Let us set 
\begin{align*}
& \bar \vv_{p}=\vv_\aug, \qquad s_{p}=0,  \\
& s_{1}^{2}= \frac{T}{a^{2}} \leq 1, \qquad \bar \vv_{1} = \vv_{\perp}.
\end{align*}
Choose $(\bar \vv_j)_{2\le j \le p-1}$ be any orthogonal basis in the orthogonal complement of $\mathrm{span}(\bar \vv_1, \bar \vv_p)$, and $\bar \uu_1,\ldots,\bar \uu_p$ be any orthogonal basis in $\mathbb{R}^p$. Also, choose any $s_2,\ldots,s_{p-1} \ge 0$ such that $s_2^2+\ldots+s_{p-1}^2 = 1 - T/a^2$.

Our construction in fact gives an SVD of $\bW$, with $(s_p, \bar \vv_p)$ being the bottom singular pair. So we have $\bW \vv_\aug=s_p \vv_\aug= 0$. Moreover, 
\begin{align*}
\|\bW\bar{\bmu}\|^{2} & =\sum_{j=1}^{p}s_{j}^{2} \langle \bar \vv_{j}, \bar{\bmu} \rangle^{2}\\
 & =\sum_{j=1}^{p-1}s_{j}^{2}\langle \bar \vv_{j}, \bar{\bmu} \rangle^{2} \qquad \text{since }s_{p}=0\\
 & =s_{1}^{2}\langle \bar \vv_{1}, \bar{\bmu} \rangle^{2}\qquad \text{by orthogonality }\bmu \perp (\bar \vv_j)_{2\le j \le p-1}\\
 & =s_{1}^{2}a^{2}=T.
\end{align*}
Thus the inner minimization of \eqref{opt:doublemin} is solved at $\|\bW \vv_\aug\|= 0$. 

In all, under the constraint $T \le a^2$, the problem of minimizing $\tilde \calL(\bW)$ boils down to 
\[
\min_{T\leq a^{2}}\quad-\frac{1}{\tau\left[(1+\sigma_\aug^{2})+\|\bmu\|^{2}T\right]}+\log\left(1+\exp\left(-\frac{2\|\bmu\|^{2}T}{\tau\left[(1+\sigma_\aug^{2})+\|\bmu\|^{2}T\right]}\right)\right).
\]
This objective is in fact the same as that in Section~\ref{sec:gmm_expansion_shrinkage} via change of variables. Denote $t = \left[(1+\sigma_\aug^{2})+\|\bmu\|^{2}T\right]^{-1} \in [(1+\sigma_\aug^2+\| \bmu \|^2 a^2)^{-1}, (1+\sigma_\aug^2)^{-1}]$. Then the above objective function can be written as
\[
h(t) = -\frac{t}{\tau} + \log\left(1+\exp\left(-\frac{2}{\tau} + 2\frac{1+\sigma_\aug^2}{\tau} t\right)\right),
\]
for which
\[
h^\prime(t) = \frac{1+2\sigma_\aug^2}{\tau} - 2\frac{1+\sigma_\aug^2}{\tau} \frac{1}{1+\exp\left(-\frac{2}{\tau} + 2\frac{1+\sigma_\aug^2}{\tau} t\right)}.
\]
As in the homogeneous case, note that $h^\prime(t)$ is increasing in $t$ and $h^\prime((1+\sigma_\aug^2)^{-1}) = \sigma_\aug^2/\tau > 0$. Further, notice that
\[
h^\prime((1+\sigma_\aug^2+\| \bmu \|^2 a^2)^{-1}) > 0 ~~\Longleftrightarrow~~ \tau > \frac{2a^2\| \bmu \|^2}{\log(1+2\sigma_\aug^2)(1+\sigma_\aug^2+a^2 \| \bmu \|^2)} = \tau^*_1.
\]
Then, if $\tau > \tau^*_1$, then $h^\prime((1+\sigma_\aug^2+\| \bmu \|^2 a^2)^{-1})> 0$ and $h(t)$ is increasing in $t\in [(1+\sigma_\aug^2+\| \bmu \|^2 a^2)^{-1}, (1+\sigma_\aug^2)^{-1}]$, so the minimum of $h(t)$ is achieved when $t$ is the smallest in the interval, or equivalently $T = a^2$. If $\tau \le \tau^*_1$, the minimizer in the interval $[0,a^2]$ is determined by $h'(t)=0$, which yields
\begin{equation}
t^* = \frac{1}{1+\sigma_\aug^2}\left(1-\frac{\tau}{2}\log(1+2\sigma_\aug^2)\right).
\end{equation}
As a result, we have
\begin{align}\label{eq:Tstar}
T^* = \frac{1}{\|\bmu\|^2}\left(\frac{1}{t^*} - (1+\sigma_\aug^2)\right) = \frac{\tau(1+\sigma_\aug^2) }{2\| \bmu \|^{2}} \log(1+2\sigma_\aug^2) \left[1-\frac{\tau}{2}\log(1+2\sigma_\aug^2)\right]^{-1}.
\end{align}

\paragraph{Case 2: $a^2 < T \le 1$.} Note that we implicitly assume $|a|<1$ since $|a|=1$ is always covered in Case~1. In order to solve the inner minimization problem of \eqref{opt:doublemin}, it suffices to consider the semidefinite program
\begin{align}
\begin{split}\label{opt:sdp}
\min &  \qquad \vv_\aug^{\top}\bA\vv_\aug\\
\text{subject to} \qquad & \mathrm{Tr}(\bA)=1,\;\bar{\bmu}^{\top}\bA\bar{\bmu}=T,\;\bA\succeq 0,
\end{split}
\end{align}
where we define $\bA=\bW^{\top}\bW \in \mathbb{R}^{p \times p}$. The square-root matrix of the minimizer gives $\bW$ that obeys $\| \bW \bar \bmu\|^2=T$ and achieves the minimum of the inner minimization problem.

The dual problem of~\eqref{opt:sdp} is given by 
\begin{align}
\begin{split}\label{opt:dual}
\max_{\lambda_{1},\lambda_{2} \in \mathbb{R}} & \qquad-\lambda_{1}-\lambda_{2}T\\
\text{subject to} & \qquad \vv_\aug\vv_\aug^{\top}+\lambda_{1}\bI_p+\lambda_{2}\bar{\bmu}\bar{\bmu}^{\top}\succeq 0.
\end{split}
\end{align}
By duality (in particular, complementary slackness), we know that the optimal $\bA$ and $\lambda_1,\lambda_2$ obey
\begin{equation}\label{eq:complementary}
\left\langle \vv_\aug\vv_\aug^{\top}+\lambda_{1}\bI+\lambda_{2}\bar{\bmu}\bar{\bmu}^{\top},\bA\right\rangle =0 \, .
\end{equation}
For now, we will assume that $a>0$; the case $a=0$ (i.e., $|\langle \bar \bmu, \vv_\aug \rangle|=1$) will be discussed later. 
We consider two separate cases: (1) the optimal $\lambda_1 = 0$, and (2) the optimal $\lambda_1 >0$.

First, when $\lambda_1 = 0$, it is easy to check that the optimal choice of $\lambda_2$ is also 0, which yields the optimal objective 0. 
Second, when $\lambda_1 >0$,  we claim that the optimal $\bA$ is a rank-one matrix, and it is given by 
\begin{equation}\label{eq:rankone}
\bA=  (\sqrt{T}\, \bar \bmu - \sqrt{1-T} \, \bar \bmu_\bot) (\sqrt{T}\, \bar \bmu - \sqrt{1-T} \, \bar \bmu_\bot)^\top,
\end{equation}
where $\bar \bmu_\bot = a \vv_\aug - r \vv_\bot$.
To see this, we note that $\vv_\aug\vv_\aug^{\top}+\lambda_{1}\bI_p+\lambda_{2}\bar{\bmu}\bar{\bmu}^{\top}$ has rank at least $p-2$ since it is two rank-one updates of $\lambda_1 \bI_p$. In fact, we will show that this matrix has rank exactly $p-1$. Otherwise, we must have
\begin{align*}
\vv_\aug^\top \Big( \vv_\aug\vv_\aug^{\top}+\lambda_{1}\bI_p+\lambda_{2}\bar{\bmu}\bar{\bmu}^{\top} \Big)\vv_\aug = 0, \\
\bar \bmu^\top \Big( \vv_\aug\vv_\aug^{\top}+\lambda_{1}\bI_p+\lambda_{2}\bar{\bmu}\bar{\bmu}^{\top} \Big)\bar \bmu = 0.
\end{align*}
These are equivalent to 
\begin{align*}
&1 + \lambda_1 + \lambda_2 \langle \bar \bmu, \vv_\aug \rangle^2 = 0 , \\
&\lambda_1 + \lambda_2 +  \langle \bar \bmu, \vv_\aug \rangle^2 = 0 ,
\end{align*}
which implies $a^2 \cdot (1 - \lambda_2) = 0$. This further implies $\lambda_2 =1$ since $a > 0$. 
However, we know that when $\lambda_1 > 0$, the optimal $\lambda_2$ would not be 1, as $\lambda_2 = 0$ is strictly better. Therefore, we have shown that $\vv_\aug\vv_\aug^{\top}+\lambda_{1}\bI_p+\lambda_{2}\bar{\bmu}\bar{\bmu}^{\top}$ has rank $p-1$. 
By \eqref{eq:complementary} primal optimal $\bA$ is a rank-one matrix, which together with the primal constraint $\bar{\bmu}^{\top}\bA\bar{\bmu}=T$ yields the conclusion that the optimal $\bA$ is uniquely determined by \eqref{eq:rankone}.

As a result, the optimal objective value of \eqref{opt:sdp} is 
\begin{equation*}
\| \bW \vv_\aug\|^2 = \langle \sqrt{T}\, \bar \bmu - \sqrt{1-T} \, \bar \bmu_\bot, \vv_\aug \rangle^2 = \big(r \sqrt{T}\, - a \sqrt{1-T} \,\big)^2\,.
\end{equation*}
As this is always than 0 (the optimal objective when $\lambda_1 =0$), we conclude that the maximal objective of~\eqref{opt:dual} is $\langle \sqrt{T}\, \bar \bmu - \sqrt{1-T} \, \bar \bmu_\bot, \vv_\aug \rangle^2 = \big(r \sqrt{T}\, - a \sqrt{1-T} \,\big)^2$.

In the corner case $a=0$, trivially, the objective is a constant: $\vv_\aug^\top \bA \vv_\aug = \bar \bmu^\top \bA \bar \bmu = T$, which is the same as $\big(r \sqrt{T}\, - a \sqrt{1-T} \,\big)^2$ in this case. 

In all, 
under the constraint $a^2 < T \le 1$, the original minimization problem \eqref{opt:doublemin} becomes 
\begin{align}
\begin{split}\label{eq:obj_fun}
\min_{1\ge T > a^2}\quad &-\frac{1}{\tau\left[(1+\sigma_\aug^{2})+\|\bmu\|^{2}T + \rho_\aug \sigma_\aug^2 (r \sqrt{T} - a \sqrt{1-T})^2 \right]}\\
&+\log\left(1+\exp\left(-\frac{2\|\bmu\|^{2}T}{\tau\left[(1+\sigma_\aug^{2})+\|\bmu\|^{2}T + \rho_\aug \sigma_\aug^2 (r \sqrt{T} - a \sqrt{1-T})^2\right]}\right)\right).
\end{split}
\end{align}

\paragraph{Step 3: combining two cases.} Recall the notation $x_+ = \max\{x,0\}$. Introduce the unified loss $\ell(T)$ that subsumes both cases:
\begin{align}
\begin{split}\label{eq:obj_fun2}
\ell(T) = &-\frac{1}{\tau\left[(1+\sigma_\aug^{2})+\|\bmu\|^{2}T + \rho_\aug \sigma_\aug^2 [ (r \sqrt{T} - a \sqrt{1-T})_+]^2 \right]}\\
&+\log\left(1+\exp\left(-\frac{2\|\bmu\|^{2}T}{\tau\left[(1+\sigma_\aug^{2})+\|\bmu\|^{2}T + \rho_\aug \sigma_\aug^2 [ (r \sqrt{T} - a \sqrt{1-T})_+ ] ^2\right]}\right)\right).
\end{split}
\end{align}
Notice that $\ell(T) \ge \ell^*(T)$ for each $T \in [0,1]$ where 
\[
\ell^*(T) := -\frac{1}{\tau\left[(1+\sigma_\aug^{2})+\|\bmu\|^{2}T\right]}+\log\left(1+\exp\left(-\frac{2\|\bmu\|^{2}T}{\tau\left[(1+\sigma_\aug^{2})+\|\bmu\|^{2}T\right]}\right)\right), 
\]
and $\ell(T) = \ell^*(T)$ for $T \in [0,a^2]$. In addition, the two functions satisfy that 
\[
\ell^*(a^2) = \ell(a^2), \qquad (\ell^*)^\prime(a^2) = \ell^\prime(a^2).
\]
If $\tau \le \tau_1^*$, we have seen in Step 1 that the minimizer $T^*$ of $\ell^*(T)$ is given by \eqref{eq:Tstar} and $T^* \in [0, a^2]$. Since $\ell(T) \ge \ell^*(T)$ and equality holds if $T \le a^2$, it verifies that $T^*$ is also the minimizer of $\ell(T)$.

If $\tau > \tau_1^*$, the analysis in Step 1 shows that for $T \in [0,a^2]$, $(\ell^*)(T)$ (and thus $\ell(T)$) is monotonically decreasing. Further, denote
\[
h(T) = (1+\sigma_\aug^{2})+\|\bmu\|^{2}T + \rho_\aug \sigma_\aug^2 (r \sqrt{T} - a \sqrt{1-T})^2,
\]
which is an increasing function of $T$ when $T \geq a^2$. The derivative of $\ell(T)$ takes the following form:
\[
\ell^\prime(T) = \frac{1}{\tau} \frac{h^\prime(T)}{h^2(T)} - \frac{2 \Vert \bmu \Vert^2}{\tau} \frac{\exp\left(-\frac{2\|\bmu\|^{2}T}{\tau h(T)}\right)}{1+\exp\left(-\frac{2\|\bmu\|^{2}T}{\tau h(T)}\right)} \frac{h(T) - T h^\prime(T)}{h^2(T)}.
\]
If $a \neq 0$ and $\rho_\aug>0$, we have $\lim_{T \rightarrow 1-} h^\prime(T) = +\infty$, then $\ell^\prime(1-) > 0$, which indicates that the minimizer is attained in the interval $(a^2,1)$. (Note that $T^* = 1$ is impossible!)

\paragraph{Step 4: degenerate cases.}
If $\rho_\aug=0$, our analysis is reduced to the homogeneous case, where $T^*=1$ if $\tau > \tau^*$.

If $a = 0$ and  $\rho_\aug>0$, we have $\bar \bmu = \pm \vv_\aug$ and 
\[
\ell(T) = -\frac{1}{\tau\left[1+\sigma_\aug^{2}+LT\right]}+\log\left(1+\exp\left(-\frac{2\|\bmu\|^{2}T}{\tau\left[1+\sigma_\aug^{2}+LT\right]}\right)\right)
\]
where $L = \| \bmu\|^2 + \rho_\aug \sigma_\aug^2$. Th analysis of the $\ell(T)$ is similar to the homogeneous case after a change of variables.

\subsection{Proof of Corollary~\ref{cor:change}}\label{sec:proof-change}

As a consequence of Theorem~\ref{thm:inhomo}, most of the statements in Corollary~\ref{cor:change} are straightforward. Below we will focus on Item 2 and 3. Let $S_k = \sum_{j \le k} \langle \vv_j, \bar \bmu \rangle^2$ for the $k$-th cumulative score for $\bmu$, where $1 \le k \le p-1$.

First, by the ordering of singular values, we have 
\[
(k+1) \sigma_{k+1}^2 \le \| \bW \|_F^2 = 1 \qquad \Rightarrow \qquad \sigma_{k+1} \le \frac{1}{\sqrt{k+1}} \,.
\] 
Therefore, we have
\[
\sum_{j > k} \sigma_j^2 \langle \vv_j, \bar \bmu \rangle^2 \le \sigma_{k+1}^2 \cdot \sum_{j > k} \langle \vv_j, \bar \bmu \rangle^2 \le \frac{1}{k+1} (1-S_k).
\]
We thus obtain an upper bound on $\| \bW \bar \bmu \|^2$:
\begin{align}
\| \bW \bar \bmu \|^2 &= \sum_{j\le k} \sigma_j^2 \langle \vv_j , \bar \bmu \rangle^2 + \sum_{j > k} \sigma_j^2 \langle \vv_j , \bar \bmu \rangle^2 \\
&\le \sum_{j\le k} \langle \vv_j , \bar \bmu \rangle^2 + \frac{1}{k+1} (1-S_k) \\
&\le \frac{k}{k+1} S_k + \frac{1}{k+1}.
\end{align}
If $\| \bW \bar \bmu \|^2 \ge c$ for certain constant $c>0$, we can choose $k\ge 1$ to the smallest integer with $k \ge (2c)^{-1} - 1$. Combining the upper bound and lower bound, we obtain $S_k \ge c/4$. This proves the first jump of the cumulative score. The second jump is because $S_{p} - S_{p-1} = \langle \vv_p, \bar \bmu \rangle^2 = r^2 > 0$ due to nondegeneracy.

Moreover, if $\tau > \tau_1^*$, the top right singular vector of $\bW$ is $\vv_1 = \sqrt{T^*} \bar{\bmu} - \sqrt{1-T^*}\bmu_{\perp}$ where $T^* > \sqrt{1-r^2}$. Since
\[
\langle \vv_1, \vv_\aug \rangle = r \sqrt{T^*} - \sqrt{1-r^2}\, \sqrt{1-T^*}
\]
is monotone increasing in $T^*$. The inner product must be positive if $\tau > \tau_1^*$. Also, $\langle \vv_1, \bar \bmu \rangle = \sqrt{T^*} > 0$, so this inner product is also positive.


\end{document}